\documentclass{article}

\PassOptionsToPackage{numbers, compress}{natbib}

\usepackage[dvipsnames]{xcolor}
\usepackage[nonatbib,final]{neurips_2020}
\usepackage[numbers]{natbib}
\usepackage[utf8]{inputenc} 
\usepackage[T1]{fontenc}    
\usepackage{url}            
\usepackage{booktabs}       
\usepackage{amsfonts}       
\usepackage{nicefrac}       
\usepackage{microtype}
\usepackage{graphicx}
\usepackage{booktabs} 
\usepackage{makecell}
\usepackage{graphicx}
\usepackage{caption}
\usepackage{courier}
\usepackage{multirow}
\usepackage{color}
 \usepackage{subfigure}
\usepackage{tikz}
\usepackage{amssymb}
\usepackage{amsmath}
\usepackage{gensymb}
\usepackage{amsfonts}       
\usepackage{enumitem}
\usepackage{multirow}
\usepackage[skins]{tcolorbox}
\usepackage{amsmath,nccmath}

\usepackage[utf8]{inputenc} 
\usepackage[T1]{fontenc}    

\usepackage[backref=page,colorlinks=true,linkcolor=RedOrange, citecolor=green]{hyperref}    
\usepackage{url}            
\usepackage{booktabs}       
\usepackage{amsfonts}       
\usepackage{nicefrac}       
\usepackage{microtype}      
\usepackage{lipsum}
\usepackage[ruled]{algorithm2e}
\usepackage{algorithmic}
\newcommand\mycommfont[1]{\footnotesize\ttfamily\textcolor{black}{#1}}
\newcommand{\algcolor}[3]{\hspace*{-\fboxsep}\colorbox{#1}{\parbox{#2\linewidth}{#3}}}
\newcommand{\algemph}[3]{\algcolor{#1}{#2}{#3}}
\usepackage{amssymb}
\usepackage{hyperref}
\usepackage{cleveref}
\usepackage{mathtools} 
\usepackage{mathtools}
\usepackage[thinlines]{easytable}
\usepackage{array}
\usepackage{amsmath}
\usepackage{wrapfig}
\usepackage{epsf}
\usepackage{graphics}
\usepackage{psfrag}

\usepackage{color}

\usepackage{mathtools}
\usepackage{amsfonts}
\usepackage{amsthm}
\usepackage{amsmath}
\usepackage{amssymb}

\newtheorem{theorem}{Theorem}

\newtheorem{lemma}{Lemma}
\newtheorem{corollary}{Corollary}
\newtheorem{definition}{Definition}
\newtheorem{remark}{Remark}

\newtheorem{assumption}{Assumption}

\long\def\comment#1{}

\newcommand{\bm}[1]{\boldsymbol{#1}}

\DeclarePairedDelimiter\floor{\lfloor}{\rfloor}

\usepackage{scrwfile}

\usepackage{enumitem}
\setitemize{noitemsep,topsep=0pt,parsep=0pt,partopsep=0pt}

\usetikzlibrary{fit,calc}

\definecolor{antiquewhite}{rgb}{0.98, 0.92, 0.84} 
\definecolor{blizzardblue}{rgb}{0.67, 0.9, 0.93}
\newcommand*{\rom}[1]{\expandafter\@slowromancap\romannumeral #1@}
 
\TOCclone[Table of Contents]{toc}{atoc}
\newcommand\StartAppendixEntries{}
\AfterTOCHead[toc]{%
  \renewcommand\StartAppendixEntries{\value{tocdepth}=-10000\relax}%
}
\AfterTOCHead[atoc]{%
  \edef\maintocdepth{\the\value{tocdepth}}%
  \value{tocdepth}=-10000\relax%
  \renewcommand\StartAppendixEntries{\value{tocdepth}=\maintocdepth\relax}%
}
\newcommand*\appendixwithtoc{%
\hypersetup{linkcolor=black}
  \addtocontents{toc}{\protect\StartAppendixEntries}
  \listofatoc
 \hypersetup{linkcolor=RedOrange}
}

\title{\bf{\LARGE{Distributionally Robust Federated Averaging }}}

\author{Yuyang Deng \qquad Mohammad Mahdi Kamani\qquad Mehrdad Mahdavi \vspace*{.2em} \\ 
The Pennsylvania State University \vspace*{.2em} \\ 
\texttt{ \{yzd82,mqk5591,mzm616\}@psu.edu}
}

\begin{document}
\maketitle

\begin{abstract}
In this paper, we study communication efficient distributed algorithms for distributionally robust federated learning via periodic averaging with adaptive sampling. In contrast to standard empirical risk minimization,  due to the minimax structure of the underlying optimization problem, a key difficulty arises from the fact that the global parameter that controls the mixture of local losses can only be updated infrequently on the global stage. To compensate for this, we propose a {Distributionally Robust Federated Averaging} ({\sffamily{DRFA}}) algorithm that employs a novel snapshotting scheme to approximate the accumulation of history gradients of the mixing parameter. We analyze the convergence rate of {\sffamily{DRFA}} in both convex-linear and nonconvex-linear settings. We also generalize the proposed idea to objectives with regularization on the mixture parameter and propose a  proximal variant, dubbed as {\sffamily{DRFA-Prox}}, with provable convergence rates. We also analyze an alternative optimization method for regularized case in strongly-convex-strongly-concave and non-convex (under PL condition)-strongly-concave settings. To the best of our knowledge, this paper is the first to solve distributionally robust federated learning with reduced communication, and to analyze the efficiency of local descent methods on distributed minimax problems. We give corroborating experimental evidence for our theoretical results in federated learning settings.
\end{abstract}
\section{Introduction}

Federated learning (FL) has been a key learning paradigm to train a centralized model from an extremely large number of devices/users without accessing their local data~\cite{konevcny2016federated}. A commonly used approach is to aggregate the individual loss functions usually weighted  proportionally to their sample sizes and solve the following optimization problem in a distributed manner: \vspace{-0.0cm}
\begin{equation} 
\min_{\boldsymbol{w} \in \mathcal{W}}  F (\boldsymbol{w}) := 
  \sum_{i=1}^N \frac{n_i}{n}\left\{f_i(\boldsymbol{w}):=\mathbb{E}_{\xi\sim \mathcal{P}_i}[\ell(\bm{w};\xi)]\right\}, \label{weighted loss}
\end{equation}
where  $N$ is number of clients each with $n_i$ training samples drawn from some unknown distribution $\mathcal{P}_i$ (possibly different from other clients), $f_i(\bm{w})$ is the local objective at device $i$ for a given  loss function $\ell$, $\mathcal{W}$ is a closed convex set, and $n$ is total number of samples.

In a federated setting,  in contrast to classical distributed optimization,  in solving the optimization problem in Eq.~\ref{weighted loss}, three key challenges need to be tackled including i) communication efficiency, ii) the low participation of devices, and iii) heterogeneity of local data shards. To circumvent the communication bottleneck, an elegant idea is to periodically average locally evolving models as employed in FedAvg algorithm~\cite{mcmahan2017communication}. Specifically,  each local device optimizes its own model for $\tau$ local iterations using SGD, and then a subset of devices is selected by the server to communicate their models for averaging. This approach, which can be considered as a variant of local SGD~\cite{stich2018local,haddadpour2019local,haddadpour2019trading} but with partial participation of devices, can significantly reduce the number of communication rounds, as demonstrated both empirically and theoretically in various studies~\cite{li2019convergence,khaled2019better,haddadpour2019convergence,haddadpour2020federated,woodworth2020minibatch}. 

While being compelling from the communication standpoint, FedAvg does not necessarily tackle the \textit{data heterogeneity} concern in FL.  In fact, it has been shown that the generalization capability of the central model learned by FedAvg, or any model obtained by solving Eq.~\ref{weighted loss} in general, is inevitably plagued by increasing the diversity among local data distributions~\cite{li2019feddane,karimireddy2019scaffold,haddadpour2019convergence}.  This is mainly due to the fact the objective in Eq.~\ref{weighted loss} assumes that all local data are sampled from the same distribution, but in a federated setting, local data distributions can significantly vary from the average distribution. Hence, while the global model enjoys a good \textit{average performance}, its performance often degrades significantly on local data when the distributions drift dramatically. 

To mitigate the data heterogeneity issue, one solution is to personalize the global model to local distributions. A few notable studies~\cite{deng2020adaptive,mansour2020three} pursued this idea and proposed to learn a mixture of the global and local models. While it is empirically observed that the per-device mixture model can reduce the generalization error on local distributions compared to the global model, however, the learned global model still suffers from the same issues as in FedAvg, which limits its adaptation to newly joined devices. An alternative solution is to learn a model that has uniformly good performance over almost all devices by minimizing the  agnostic (distributionally robust) empirical loss:\vspace{-0.1cm}
\begin{equation}
   \min_{\boldsymbol{w} \in \mathcal{W}} \max_{\bm{\lambda}\in \Lambda } F (\boldsymbol{w},\boldsymbol{\lambda}) := \sum_{i=1}^N \lambda_i  f_i(\boldsymbol{w}), \label{agnostic loss}\vspace{-0.1cm}
\end{equation}
where $\boldsymbol{\lambda} \in \Lambda \doteq	 \{\boldsymbol{\lambda} \in \mathbb{R}_{+}^N: \sum_{i=1}^N \lambda_i = 1\}$ is the global weight for each local loss function. 

The main premise is that by minimizing the robust empirical loss, the learned model is guaranteed to perform well over the worst-case combination of empirical local distributions, i.e., limiting the reliance to only a fixed combination of local objectives\footnote{Beyond robustness, agnostic loss  yields a notion of fairness~\cite{mohri2019agnostic}, which is not the focus of present work.}.~\citet{mohri2019agnostic} was among the first to introduce the agnostic loss into federated learning, and provided convergence rates for convex-linear and strongly-convex-strongly-concave functions. However, in their setting, the server has to communicate with local user(s) at each iteration to update the global mixing parameter $\boldsymbol{\lambda}$, which hinders its scalability due to communication cost.  

The aforementioned issues, naturally leads to the following question:  \emph{\textit{Can we propose a provably communication efficient algorithm that is also distributionally robust?}} The purpose of this paper is to give an affirmative answer to this question by proposing a \textbf{Distributionally Robust Federated Averaging} ({\sffamily{DRFA}}) algorithm that is distributionally robust, while being communication-efficient via periodic averaging,
and partial node participation, as we show both theoretically and empirically. From a high-level algorithmic perspective, we develop an approach to analyze minimax optimization methods where model parameter $\bm{w}$  is trained distributedly at local devices, and mixing parameter $\bm{\lambda}$ is only updated at server periodically. Specifically,  each device optimizes its model locally, and  a subset of them are adaptively sampled based on $\boldsymbol{\lambda}$ to  perform  model averaging. We note that since $\bm{\lambda}$ is updated only at synchronization rounds,  it will inevitably hurt the convergence rate. Our key technical contribution is the introduction and analysis of a  \textit{randomized snapshotting schema} to approximate the accumulation of history of local gradients to update $\bm{\lambda}$ as to entail good convergence.\vspace{-1mm}

\noindent\textbf{Contributions.} We summarize the main contributions of our work as follows:\vspace{-0.0cm}
\begin{itemize}[leftmargin=*]
    \item To the best of our knowledge, the proposed {\sffamily{DRFA}} algorithm is the first to solve distributionally robust optimization in a communicationally  efficient manner  for federated learning, and to give theoretical analysis on heterogeneous (non-IID) data distributions. The proposed idea of decoupling the updating of $\bm{w}$ from $\boldsymbol{\lambda}$ can be integrated as a building block into other federated optimization methods, e.g.~\cite{karimireddy2019scaffold,li2018federated} to yield a distributionally robust solution.\vspace{1mm}
    \item We derive the convergence rate of our algorithm when loss function is convex in $\bm{w}$ and linear in $\boldsymbol{\lambda}$, and establish an $O({1}/{T^{3/8}})$ convergence rate with only $O\left(T^{3/4}\right)$ communication rounds. For  nonconvex loss, we establish convergence rate of $O({1}/{T^{1/8}})$ with only $O\left(T^{3/4}\right)$ communication rounds. Compared to~\cite{mohri2019agnostic}, we significantly reduce the communication rounds.\vspace{1mm}
    \item For the regularized objectives, we propose a variant algorithm, dubbed as {\sffamily{DRFA-Prox}}, and prove that it enjoys the same convergence rate as {\sffamily{DRFA}}.  We also analyze an alternative method for optimizing regularized objective and derive the convergence rate in strongly-convex-strongly-concave and non-convex (under PL condition)-strongly-concave settings.\vspace{1mm}
    \item We demonstrate the practical efficacy of the proposed algorithm over competitive baselines through experiments on federated datasets.
\end{itemize}

\section{Related Work} \vspace{-1mm}
\textbf{Federated Averaging.} Recently, many federated methods have been considered in the literature. FedAvg, as a variant of local GD/SGD, is firstly proposed in~\cite{mcmahan2017communication} to alleviate the communication bottleneck in FL. The first convergence analysis of local SGD on strongly-convex smooth loss functions has established in~\cite{stich2018local} by showing an $O\left({1}/{T}\right)$ rate with only $O(\sqrt{T})$ communication rounds. The analysis of the convergence of local SGD for nonconvex functions and its adaptive variant is proposed in~\cite{haddadpour2019local}. The extension  to heterogeneous data allocation and general convex functions, with a tighter bound, is carried out in~\cite{khaled2020tighter}.~\cite{haddadpour2019convergence} analyzed local GD and SGD on nonconvex loss functions as well as networked setting in a fully decentralized setting. The recent work~\cite{li2019convergence} analyzes the convergence of FedAvg under non-iid data for strongly convex functions. In~\cite{woodworth2020local,woodworth2020minibatch}, Woodworth et al compare the convergence rate of local SGD and mini-batch SGD, under homogeneous and heterogeneous settings respectively.

\textbf{Distributionally Robust Optimization.} There is a rich body of literature on Distributionally Robust Optimization (DRO), and here, we try to list the most closely related work. DRO is an effective approach to deal with the imbalanced or non-iid data~\cite{namkoong2016stochastic,namkoong2017variance,fan2017learning,pmlr-v89-zhu19a,fan2017learning,mohri2019agnostic}, which is usually formulated as a minimax problem. A bandit mirror descent algorithm to solve the DRO minimax problem is proposed in~\cite{namkoong2016stochastic} . Another approach is to minimize top-k losses in the finite sum to achieves the distributional robustness~\cite{fan2017learning}. The first proposal of the DRO in federated learning is~\cite{mohri2019agnostic}, where they advocate minimizing the maximum combination of empirical losses to mitigate data heterogeneity. 

\textbf{Smooth Minimax Optimization.} Another related line of work  to this paper is the minimax optimization. One popular primal-dual optimization method is \emph{(stochastic) gradient descent ascent} or (S)GDA for short. The first work to prove that (S)GDA can converge efficiently on nonconvex-concave objectives is~\cite{lin2019gradient}. Other classic algorithms for the minimax problem are extra gradient descent (EGD)~\cite{korpelevich1976extragradient} and optimistic gradient descent (OGD), which are widely studied and applied in machine learning (e.g., GAN training~\cite{goodfellow2014generative,daskalakis2017training,liu2019decentralized,liang2018interaction}). The algorithm proposed in~\cite{thekumparampil2019efficient} combines the ideas of mirror descent and Nesterov's accelerated gradient descent (AGD)~\cite{nesterov1983method}, to achieve $\tilde{O}\left({1}/{T^2}\right)$ rate on strongly-convex-concave functions, and $\tilde{O}\left({1}/{T^{1/3}}\right)$ rate on nonconvex-concave functions. A proximally guided stochastic mirror descent and variance reduction gradient method (PGSMD/PGSVRG) for nonconvex-concave optimization is proposed in~\cite{rafique2018non}. Recently, an algorithm using AGD as a building block is designed in~\cite{lin2020near}, showing a linear convergence rate on strongly-convex-strongly-concave objective, which matches with the theoretical lower bound~\cite{zhang2019lower}. The decentralized minimax problem is studied in~\cite{srivastava2011distributed,mateos2015distributed,liu2019decentralized},  however, none of these works study the case where one variable is distributed and trained locally, and the other variable is updated periodically, similar to our proposal.

\section{Distributionally Robust Federated Averaging}\label{sec: DRFA}\vspace{-1mm}
We consider a federated setting where  $N$ users aim to learn a global model in a collaborative manner without exchanging their data with each other. However, users can exchange information via a server that is connected to all users. Recall that the distributionally robust optimization problem  can be formulated as 
$
   \min_{\boldsymbol{w} \in \mathcal{W}} \max_{\bm{\lambda}\in \Lambda } F (\boldsymbol{w},\boldsymbol{\lambda}) := \sum_{i=1}^N \lambda_i  f_i(\boldsymbol{w})
$, where  $f_i(\bm{w})$ is the local objective  function corresponding to user $i$,  which is often defined as the empirical or true risk over its local data. As mentioned earlier,  we address this problem in a federated setting where we assume that $i$th local data shard is sampled from a local distribution $\mathcal{P}_i$-- possibly different from the distribution of other data shards. Our goal is to train a central model $\bm{w}$ with limited communication rounds. We will start with this simple setting where the global objective is linear in the mixing parameter $\bm{\lambda}$, and will show in  Section~\ref{sec:regularized} that our algorithm can also provably optimize regularized objectives where a functional constraint is imposed on the mixing parameter, with a slight difference in the scheme to update $\boldsymbol{\lambda}$.

\vspace{-1mm}
\subsection{The proposed algorithm}\vspace{-2mm}
To solve the aforementioned problem, we propose   {\sffamily{DRFA}} algorithm as summarized in Algorithm~\ref{alg:1}, which consists of two main modules: local model updating and periodic mixture parameter synchronization. The local model updating  is similar to the common local SGD~\cite{stich2018local} or FedAvg~\cite{mcmahan2017communication}, however, there is a subtle difference in selecting the clients as we employ an adaptive sampling schema. To formally present the steps of {\sffamily{DRFA}}, let us define $S$ as the rounds of communication between server and users and $\tau$ as the number of local updates that each user runs between two consecutive rounds of communication. We use $T = S\tau$ to denote the total number of iterations the optimization proceeds.
\begin{algorithm}[t]
	\renewcommand{\algorithmicrequire}{\textbf{Input:}}
	\renewcommand{\algorithmicensure}{\textbf{Output:}}
	\caption{Distributionally Robust Federated Averaging ({\sffamily{DRFA}}) }
	\label{alg:1}
	\begin{algorithmic}[1]
		\REQUIRE $N$ clients , synchronization gap $\tau$, total number of iterations $T$, $S=T/\tau$, learning rates $\eta$, $\gamma$, sampling size $m$, initial model $\bar{\boldsymbol{w}}^{(0)}$ and initial $\boldsymbol{\lambda}^{(0)}$.\\
		\ENSURE Final solutions $\hat{\boldsymbol{w}} = \frac{1}{mT} \sum_{t=1}^{T}\sum_{i\in{\mathcal{D}^{(\floor{\frac{t}{\tau}})}}} \boldsymbol{w}^{(t)}_i$,  $\hat{\boldsymbol{\lambda}} = \frac{1}{S}\sum_{s=0}^{S-1} \boldsymbol{\lambda}^{(s)}$, or (2) $\bm{w}^T$, $\boldsymbol{\lambda}^S$. 
         \FOR{  $s = 0$ to $S-1$ } 
        
        {\STATE {Server \textbf{samples}  $\mathcal{D}^{(s)} \subset [N]$ according to  $\boldsymbol{\lambda}^{(s)}$ with size of $m$}\\
        \STATE {Server \textbf{samples} $t'$ from $ s\tau+1,\ldots,(s+1)\tau$} uniformly at random\\
        \STATE Server \textbf{broadcasts} $\bar{\boldsymbol{w}}^{(s)}$ and $t'$ to all clients  $i \in  
        \mathcal{D}^{(s)}$}\\[5pt]
        \algemph{antiquewhite}{1}{
        \FOR{clients $i \in \mathcal{D}^{(s)}$ \textbf{parallel}}
        \STATE Client \textbf{sets} $\boldsymbol{w}_i^{(s\tau)} = \bar{\bm{w}}^{(s)}$
        \FOR{$t = s\tau,\ldots,(s+1)\tau-1 $}
        \STATE {$\bm{w}^{(t+1)}_i =  \prod_{\mathcal{W}}\left(\bm{w}^{(t)}_i - \eta \nabla f_i(\bm{w}^{(t)}_i;\xi^{(t)}_i)\right)  $}\\[-8pt]
        \ENDFOR
        \ENDFOR
        \STATE {Client $i \in \mathcal{D}^{(s)}$ \textbf{sends} $\bm{w}^{((s+1)\tau)}_i$ and $\bm{w}^{(t')}_i$ back to the server}}
        {\STATE {Server \textbf{computes} $\bar{\bm{w}}^{(s+1)} = \frac{1}{m} \sum_{i\in \mathcal{D}^{(s)}}  \bm{w}^{((s+1)\tau)}_i$ }\\[-5pt]
        \STATE Server \textbf{computes} $\bm{w}^{(t')} = \frac{1}{m} \sum_{i\in \mathcal{D}^{(s)}}  \bm{w}^{(t')}_i$}\\
        \algemph{blizzardblue}{1}{
        \STATE {Server uniformly samples a subset $\mathcal{U}  \subset  [N]$ of clients with size $m$ \hfill\mycommfont{{// Update $\boldsymbol{\lambda}$}}}
        \STATE {Server \textbf{broadcasts} $\bm{w}^{(t')}$ to each client $i \in \mathcal{U}$, compute $f_i(\bm{w}^{(t')};\xi_i )$ over a local minibatch}
        \STATE {Make $N$-dimensional vector $\bm{v}$: $v_i = \frac{N}{m}f_i(\bm{w}^{(t')};\xi_i $) if $i \in \mathcal{U}$, otherwise ${v}_i = 0$}
        \STATE { Server updates $\boldsymbol{\lambda}^{(s+1)} = \prod_{\Lambda}\left( \boldsymbol{\lambda}^{(s)}  + \tau \gamma  \bm{v} \right) $}
         }
        \ENDFOR 
	\end{algorithmic}  
\end{algorithm}

\noindent\textbf{Periodic model averaging via adaptive sampling.}~Let $\bar{\bm{w}}^{(s)}$ and $\boldsymbol{\lambda}^{(s)}$ denote the global primal and dual parameters at server after synchronization stage $s-1$, respectively.  At the beginning of the $s$th communication stage, server  selects $m$ clients $\mathcal{D}^{(s)} \subset [N]$ randomly based on the probability vector $\boldsymbol{\lambda}^{(s)}$ and broadcasts its current model $\bar{\bm{w}}^{(s)}$ to all the clients  $i \in \mathcal{D}^{(s)}$. Each   client $i$, after receiving the global model, updates it using local SGD on its own data for $\tau$ iterations. To be more specific, let $\bm{w}^{(t+1)}_i$ denote the model at client $i$ at iteration $t$ within stage $s$. At each local iteration $t = s\tau, \ldots, (s+1)\tau$, client $i$ updates its local model according to the following rule
$$\bm{w}^{(t+1)}_i =  \prod_{\mathcal{W}}\left(\bm{w}^{(t)}_i - \eta \nabla f_i(\bm{w}^{(t)}_i;\xi^{(t)}_i)\right),  $$
where $\prod_{\mathcal{W}}(\cdot)$ is the projection onto $\mathcal{W}$ and the stochastic gradient is computed on a random sample $\xi^{(t)}_i$ picked from the $i$th local dataset. After $\tau$ local steps, each client  sends  its current model $\bm{w}^{((s+1)\tau)}_i$ to the server to compute the next global average primal model  $\bar{\bm{w}}^{(s+1)} = (1/m)\sum_{i\in \mathcal{D}^{(s)}}  \bm{w}^{((s+1)\tau)}_i$. This  procedure is repeated for $S$ stages. We note that adaptive sampling  not only addresses the scalability issue, but also leads to smaller communication load compared to full participation case.

\noindent\textbf{Periodic mixture parameter updating.}~The global mixture parameter $\boldsymbol{\lambda}$ controls the mixture of different local losses, and can only be updated by server at synchronization stages. The updating scheme for $\boldsymbol{\lambda}$ will be  different when the objective function is equipped with or without the regularization on $\boldsymbol{\lambda}$. In the absence of regularization on $\boldsymbol{\lambda}$, the problem is simply linear in $\boldsymbol{\lambda}$. A key observation is that  in linear case, the gradient of $\boldsymbol{\lambda}$ only depends on $\bm{w}$, so we can approximate the sum of history gradients over the previous local period (which does not show up in the real dynamic).  Indeed, between two synchronization stages, from iterations $s\tau+1$ to $(s+1)\tau$, in the \textit{fully synchronized} setting~\cite{mohri2019agnostic}, we can  update $\boldsymbol{\lambda}$ according to \vspace{-2mm}
\begin{equation*}
        \boldsymbol{\lambda}^{(s+1)} =\prod_{\Lambda}\left(\boldsymbol{\lambda}^{(s)} +\gamma \sum_{t = s\tau+1}^{(s+1)\tau} \nabla_{\boldsymbol{\lambda}}F( \bm{w}^{(t)}, \boldsymbol{\lambda}^{(s)})\right)
        \end{equation*}
 where $\bm{w}^{(t)} = \frac{1}{m}\sum_{i\in \mathcal{D}^{(s)}} \bm{w}_i^{(t)}$ is the  average model  at iteration $t$. 
    
    To approximate this update,  we propose a \textit{random snapshotting} schema as follows. At the beginning of the $s$th communication stage, server samples a random iteration $t'$ (snapshot index)  from the range of $s\tau+1$ to $(s+1)\tau$ and sends it to sampled devices $\mathcal{D}^{(s)}$ along with the global model. After the local updating stage is over, every  selected device sends its local model at index $t'$, i.e., $\bm{w}^{(t')}_i$, back to the server. Then, server computes the average model $\bm{w}^{(t')} = \frac{1}{|\mathcal{D}^{(s)}|} \sum_{i\in\mathcal{D}^{(s)}}\bm{w}^{(t')}_i$, that will be used for updating the mixture parameter $\boldsymbol{\lambda}^{(s)}$ to $\boldsymbol{\lambda}^{(s+1)}$ ($\boldsymbol{\lambda}^{(s+1)}$ will be used at stage $s+1$ for sampling another subset of users $\mathcal{D}^{(s+1)}$). To simulate the update we were supposed to do in the fully synchronized setting, server broadcasts  $\bm{w}^{(t')}$ to a set $\mathcal{U}$ of $m$ clients, selected uniformly at random,  to stochastically evaluate their local losses $f_i(\cdot), i \in \mathcal{U}$ at $ \bm{w}^{(t')}$ using a random minibatch $\xi_i$ of their local data. After receiving evaluated losses, server will construct the vector $\bm{v}$ as in Algorithm~\ref{alg:1}, where $v_i = \frac{N}{m} f_i(\bm{w}^{(t')};\xi_i), i \in \mathcal{U}$ to compute a stochastic gradient at dual parameter. We claim that this is an \textit{unbiased estimation} by noting the following identity:
\begin{equation}
   \mathbb{E}_{t',\mathcal{U},\xi_i}\left[\tau \bm{v}\right] = \mathbb{E}_{t'}\left[\tau \nabla_{\boldsymbol{\lambda}}F\left( \bm{w}^{(t')}, \boldsymbol{\lambda}^{(s)}\right)\right] = \sum_{t = s\tau+1}^{(s+1)\tau} \nabla_{\boldsymbol{\lambda}}F\left(\bm{w}^{(t)}, \boldsymbol{\lambda}^{(s)}\right) \label{eq: approximation gradient}.
\end{equation}
However, the above estimation has a high variance in the order of $O(\tau^2)$, so a crucial question that we need to address is finding the proper choice of $\tau$ to guarantee convergence, while minimizing the overall communication cost. We also highlight that unlike local SGD, the proposed algorithm requires two rounds of communication at each synchronization step for decoupled updating of parameters. 

\section{Convergence Analysis}
In this section, we present our theoretical results on the guarantees of the {\sffamily{DRFA}} algorithm  for two  general class of convex and  nonconvex smooth loss functions. All the proofs are deferred to appendix. 

\noindent\textbf{Technical challenge.}~Before stating the main results we would like to highlight one of the main theoretical challenges in proving the convergence rate. In particular, a key step in analyzing the local descent methods with periodic averaging is to bound the deviation between local and  (virtual) global at each iteration. In minimizing empirical risk (finite sum),~\cite{khaled2019better} gives a tight bound on the deviation of a local model from  averaged model which depends on the quantity $\frac{1}{N}\sum_{i=1}^N \|\nabla f_i(\bm{x}^*)\|^2$, where $\bm{x}^*$ is the minimizer of $\frac{1}{N}\sum_{i=1}^N  f_i(\bm{x})$. However,  their analysis is not generalizable to minimax setting, as the dynamic of primal-dual method will change the minimizer of $F(\cdot, \boldsymbol{\lambda}^{(s)})$ every time $\boldsymbol{\lambda}^{(s)}$ is updated, which makes the analysis more challenging compared to the average loss case.   In light of this and in order to subject heterogeneity of local distributions to a more formal treatment in minimax setting, we introduce a  quantity to measure dissimilarity among local gradients.
\begin{definition}[Weighted Gradient Dissimilarity] 
A set of local objectives $f_i(\cdot), i=1, 2, \ldots, N$  exhibit $
\Gamma$ gradient dissimilarity defined as
$    \Gamma := \sup_{\bm{w} \in \mathcal{W}, \boldsymbol{p}\in \Lambda,i\in [n],}  \sum_{j\in [n]}p_j\|\nabla f_i(\bm{w})- \nabla f_{j}(\bm{w})\|^2$.
\end{definition}
The above notion is a generalization of gradient dissimilarity, which is employed in the analysis of local SGD in federated setting~\cite{li2019communication,deng2020adaptive,li2019convergence,woodworth2020minibatch}. This quantity will be zero if and only if all local functions are identical. The obtained bounds will depend on the gradient dissimilarity as local updates only employ samples from  local data with possibly different statistical realization.

 We now turn to analyzing the convergence of the proposed algorithm. Before, we make the following customary assumptions:
\begin{assumption} [Smoothness/Gradient Lipschitz]\label{assumption: smoothness}
 Each component function $f_i(\cdot), i=1,2, \ldots, N$ and global function $F(\cdot, \cdot)$  are $L$-smooth, which implies: $
    \| \nabla f_i(\bm{x}_1) - \nabla f_i(\bm{x}_2)\| \leq  L \|\bm{x}_1 - \bm{x}_2\|, \forall i \in [N], \forall \bm{x}_1, \bm{x}_2$ and $ 
    \| \nabla F(\bm{x}_1,\bm{y}_1) - \nabla F(\bm{x}_2,\bm{y}_2)\| \leq  L \|(\bm{x}_1,\bm{y}_1) - (\bm{x}_2,\bm{y}_2)\|, \forall (\bm{x}_1,\bm{y}_1), (\bm{x}_2,\bm{y}_2)$.
\end{assumption}
 
\begin{assumption} [Gradient Boundedness] \label{assumption: bounded gradient} The gradient w.r.t $\boldsymbol{w}$ and $\boldsymbol{\lambda}$ are  bounded, i.e., $\|\nabla f_i(\boldsymbol{w})\|\leq G_w $ and  $\|\nabla_{\boldsymbol{\lambda}}F(\boldsymbol{w},\boldsymbol{\lambda})\|\leq G_{\lambda}$.
 \end{assumption}

\begin{assumption} [Bounded Domain] \label{assumption: bounded domain} The diameters of $\mathcal{W}$ and $ \Lambda$ are bounded by $D_{\mathcal{W}}$ and $D_{\Lambda}$.
\end{assumption}

\begin{assumption}[Bounded Variance] \label{assumption: bounded variance}  Let $\Tilde{\nabla} F(\bm{w};\boldsymbol{\lambda}) $ be stochastic gradient for $\boldsymbol{\lambda}$, which is the $N$-dimensional vector such that the $i$th entry is $f_i(\bm{w};\xi)$, and the rest are zero. Then we assume $
    \|\nabla f_i(\bm{w};\xi) - \nabla f_i(\bm{w})\| \leq  \sigma^2_{w}, \forall i \in [N]$ and $
     \|\Tilde{\nabla} F(\bm{w};\boldsymbol{\lambda}) - \nabla F(\bm{w};\boldsymbol{\lambda})\| \leq  \sigma^2_{\lambda}$.
\end{assumption}
\subsection{Convex losses}
The following theorem  establishes the convergence rate of primal-dual gap for convex objectives.
\begin{theorem} \label{theorem1}
Let each local function $f_i$ be convex, and global function $F$ be linear in $\boldsymbol{\lambda}$. Assume the conditions in  Assumptions~\ref{assumption: smoothness}-\ref{assumption: bounded variance} hold. If we optimize (\ref{agnostic loss}) using Algorithm~\ref{alg:1} with synchronization gap $\tau =  \frac{T^{1/4}}{\sqrt{m}}$, learning rates $\eta = \frac{1}{4L \sqrt{T}}$ and $\gamma = \frac{1}{T^{5/8}}$, for the returned solutions $\hat{\boldsymbol{w}}$ and $\hat{\boldsymbol{\lambda}}$ it holds that
\begin{equation*}
{\small 
\begin{aligned}
   \max_{\boldsymbol{\lambda}\in \Lambda}\mathbb{E}[F(\hat{\boldsymbol{w}},\boldsymbol{\lambda} )] -\min_{\bm{w}\in\mathcal{W}} \mathbb{E}[F(\boldsymbol{w} ,\hat{\boldsymbol{\lambda}} )] \leq O\Big{(}\frac{D_{\mathcal{W}}^2+G_{w}^2}{\sqrt{T}} +\frac{D_{\Lambda}^2}{T^{3/8}}  +\frac{G_{\lambda}^2}{m^{1/2}T^{3/8}} +\frac{\sigma_{\lambda}^2}{m^{3/2}T^{3/8}}+ \frac{\sigma_w^2+\Gamma}{m\sqrt{T} }\Big{)}.
\end{aligned}}
\end{equation*}
 
\end{theorem} 
 
The proof of Theorem~\ref{theorem1} is deferred to Appendix~\ref{sec: proof DRFA convex}. Since we update $\boldsymbol{\lambda}$ only at the synchronization stages, it will almost inevitably hurt the convergence. The original agnostic federated learning~\cite{mohri2019agnostic} using SGD can achieve an  $O(1/\sqrt{T})$ convergence rate, but we achieve a slightly slower rate $O\left(1/ T^{3/8}\right)$ to  reduce the communication complexity from $O(T)$ to $O(T^{3/4})$. Indeed, we trade $O(T^{1/8})$ convergence rate for  $O(T^{1/4})$ communication rounds. As we will show in the proof, if we choose $\tau$ to be a constant, then we recover the same $O(1/{\sqrt{T}})$ rate as \cite{mohri2019agnostic}. Also, the dependency of the obtained rate does not demonstrate a  linear speedup in the number of sampled workers $m$. However, increasing $m$ will also accelerate the rate, but does not affect the dominating  term. We leave tightening the obtained rate to achieve a linear speedup in terms of $m$ as an interesting future work.

\subsection{Nonconvex losses}
We now proceed to state the convergence in the case where local objectives $f_i, i \in [N]$ are nonconvex, e.g., neural networks. Since $f_i$ is no longer convex, the primal-dual gap is not a meaningful quantity to measure the convergence. Alternatively, following the standard analysis of nonconvex minimax optimization, one might consider the following functions to facilitate the analysis.
\begin{definition} \label{df1}
We define function $\Phi (\cdot)$ at any primal parameter $\bm{w}$ as:
\begin{eqnarray}
      \Phi (\boldsymbol{w}) :=   F(\boldsymbol{w},\boldsymbol{\lambda}^*(\boldsymbol{w})), \quad \text{where} \; \boldsymbol{\lambda}^*(\boldsymbol{w}) := \arg \max_{\boldsymbol{\lambda}\in \Lambda} F(\boldsymbol{w},\boldsymbol{\lambda}) \label{Phi}.
\end{eqnarray}
\end{definition}
However, as argued in~\cite{lin2019gradient}, on nonconvex-concave(linear) but not strongly-concave objective, directly using $\|\nabla \Phi(\bm{w})\|$ as convergence measure is still difficult for analysis. Hence, Moreau envelope of $\Phi$ can be utilized to analyze the convergence as used in several recent studies~\cite{davis2019stochastic,lin2019gradient,rafique2018non}.
\begin{definition}[Moreau Envelope] A function $\Phi_{p} (\bm{x})$ is the $p$-Moreau envelope of a function $\Phi$ if $    \Phi_{p} (\bm{x}) := \min_{\bm{w}\in \mathcal{W}} \left\{ \Phi  (\bm{w}) + \frac{1}{2p}\|\bm{w}-\bm{x}\|^2\right\}$.
\end{definition}
We will use  $1/2L$-Moreau envelope of $\Phi$, following the setting in~\cite{lin2019gradient,rafique2018non}, and  state the convergence rates in terms of $\|\nabla \Phi_{1/2L} (\bm{w})\|$.

\begin{theorem} \label{thm:nonconvex-linear}
Assume each local function $f_i$ is nonconvex, and global function $F$ is linear in $\boldsymbol{\lambda}$. Also, assume the conditions in  Assumptions~\ref{assumption: smoothness}-\ref{assumption: bounded variance} hold. If we optimize (\ref{agnostic loss}) using Algorithm~\ref{alg:1} with synchronization gap $\tau =  T^{1/4}$, letting $\bm{w}^t = \frac{1}{m}\sum_{\mathcal{D}^{(\floor{ \frac{t}{\tau}})}} \bm{w}_i^{(t)}$ to denote the  virtual average model at $t$th iterate, by choosing $\eta =   \frac{1}{4LT^{3/4}} $ and $\gamma = \frac{1}{\sqrt{T}} $, we have:
 
 {\small\begin{align*}{ \frac{1}{T}\sum_{t=1}^T \mathbb{E}\left[\left\| \nabla  \Phi_{1/2L}(\bm{w}^{t}) \right\|^2 \right] \leq   O\left(\frac{D_{\Lambda}^2}{T^{1/8}}+ \frac{ \sigma_{\lambda}^2}{mT^{1/4}}+ \frac{G_{\lambda}^2}{ T^{1/4}}+ \frac{G_w \sqrt{G_w^2+\sigma_w^2}}{T^{1/8}} +  \frac{D_{\mathcal{W}}(\sigma_w+\sqrt{\Gamma})}{T^{1/2} }\right)}.  
\end{align*}}

\end{theorem} 
The proof of Theorem~\ref{thm:nonconvex-linear} is deferred to Appendix~\ref{sec: proof DRFA nonconvex}. We obtain an $O\left({1}/{T^{1/8}}\right)$ rate here, with $O\left({1}/{T^{3/4}}\right)$ communication rounds. Compared to  SOTA algorithms proposed in~\cite{lin2019gradient,rafique2018non} in nonconvex-concave setting which achieves an $O\left({1}/{T^{1/4}}\right)$ rate in a single machine setting,  our algorithm is  distributed and communication efficient. Indeed, we trade $O\left({1}/{T^{1/8}}\right)$ rate for saving $O\left({1}/{T^{1/4}}\right)$ communications. One thing worth noticing is that in~\cite{lin2019gradient},  it is proposed to use a smaller step size for the primal variable than dual variable, while here we choose a small step size for dual variable too. That is mainly because the approximation of dual gradients in our setting introduces a large variance which necessities to employ smaller rate to compensate for high variance.  Also, the number of participated clients will not accelerate the leading term, unlike vanilla local SGD or its variants~\cite{mcmahan2017communication,stich2018local,karimireddy2019scaffold}.

\begin{algorithm}[t]
	\renewcommand{\algorithmicrequire}{\textbf{Input:}}
	\renewcommand{\algorithmicensure}{\textbf{Output:}}
	\caption{Distributionally Robust Federated Averaging: Proximal Method   (\sffamily{DRFA-Prox})}
	\label{alg:3}
	\begin{algorithmic}[1]
	 \REQUIRE The algorithm is identical to Algorithm~\ref{alg:1} except the updating rule for $\boldsymbol{\lambda}$.
        \algemph{blizzardblue}{1} {
         \STATE {\quad Server uniformly samples a subset $\mathcal{U}  \subset  [N]$ of clients with size $m$ \hfill\mycommfont{{// Update $\boldsymbol{\lambda}$}}}
        \STATE {\quad Server \textbf{broadcasts} $\bm{w}^{(t')}$ to each client $i \in \mathcal{U}$}
        \STATE {\quad Each client $i \in \mathcal{U}$  computes $f_i(\bm{w}^{(t')};\xi_i)$ over a local minibatch $\xi_i$ and sends to server}
        \STATE {\quad Server computes $N$-dimensional vector $\bm{v} $: ${v}_i = \frac{N}{m}f_i(\bm{w}^{(t')};\xi_i $) if $i \in \mathcal{U}$, otherwise ${v}_i = 0$  }
        \STATE { \quad Server updates $\boldsymbol{\lambda}^{(s+1)} = \arg\max_{\bm{u}\in \Lambda} \left \{\tau g(\bm{u}) - \frac{1}{2\gamma}\|\boldsymbol{\lambda}^{(s)}+\gamma \tau \bm{v} - \bm{u} \|^2 \right \}. $}
         }
	\end{algorithmic}  
\end{algorithm}

\section{DRFA-Prox: Optimizing Regularized Objective}\label{sec:regularized} 
As mentioned before, our algorithm can be generalized to impose a regularizer on $\bm{\lambda}$ captured by a regularization function $g(\boldsymbol{\lambda})$ and to solve  the following minimax optimization problem:
\begin{equation}
   \min_{\boldsymbol{w} \in \mathcal{W}} \max_{\boldsymbol{\lambda}\in \Lambda } F (\boldsymbol{w},\boldsymbol{\lambda}) := \left\{f(\boldsymbol{w},\boldsymbol{\lambda}):=\sum_{i=1}^N \lambda_i  f_i(\boldsymbol{w}) \right\} +g (\boldsymbol{\lambda}) . \label{regularized loss}
\end{equation}
The regularizer $g(\boldsymbol{\lambda})$ can be introduced to leverage the domain prior, or to make the $\boldsymbol{\lambda}$ update robust to adversary (e.g., the malicious node may send a very large fake gradient of $\boldsymbol{\lambda}$). The choices of $g$ include KL-divergence, optimal transport~\cite{kantorovich2006translocation,monge1781memoire}, or $\ell_p$ distance.

In regularized setting, by examining the structure of the gradient w.r.t. $\boldsymbol{\lambda}$, i.e.,
$    \nabla_{\boldsymbol{\lambda}} F(\bm{w},\boldsymbol{\lambda}) = \nabla_{\boldsymbol{\lambda}} f(\bm{w},\boldsymbol{\lambda}) + \nabla_{\boldsymbol{\lambda}} g(\boldsymbol{\lambda}) \nonumber.
$, while  $\nabla_{\bm{\lambda}} f(\bm{w},\boldsymbol{\lambda})$ is independent of $\boldsymbol{\lambda}$, but $\nabla_{\boldsymbol{\lambda}} g(\boldsymbol{\lambda})$ has dependency on $\boldsymbol{\lambda}$, and consequently our approximation method in Section~\ref{sec: DRFA} is not fully applicable  here. Inspired by the proximal gradient methods~\cite{beck2017first,nesterov2013gradient,beck2009fast}, which is widely employed in the problems where the gradient of the regularized term is hard to obtain, we adapt a similar idea, and propose a proximal variant of {\sffamily{DRFA}}, called {\sffamily{DRFA-Prox}}, to tackle regularized objectives. In {\sffamily{DRFA-Prox}}, the only difference is the updating rule of $\boldsymbol{\lambda}$ as detailed in Algorithm~\ref{alg:3}. We still employ the gradient approximation in {\sffamily{DRFA}} to estimate history gradients of $\nabla_{\bm{\lambda}} f$, however we utilize  proximity operation to update $\boldsymbol{\lambda}$:
\begin{align}
    \boldsymbol{\lambda}^{(s+1)} = \arg\max_{\bm{u}\in \Lambda} \left \{\tau g(\bm{u}) - \frac{1}{2\gamma}\|\boldsymbol{\lambda}^{(s)}+\gamma \tau \bm{v} - \bm{u} \|^2 \right \}.\nonumber
\end{align}
As we will show in the next subsection, {\sffamily{DRFA-Prox}} enjoys the same convergence rate as {\sffamily{DRFA}}, both on convex and nonconvex losses.

\subsection{Convergence of DRFA-Prox} 
The following  theorems  establish the convergence rate of {\sffamily{DRFA-Prox}} for convex and nonconvex objectives in federated setting.
\begin{theorem}[Convex loss] \label{thm: regularized convex-linear}
Let each local function $f_i$ be convex. Assume the conditions in  Assumptions~\ref{assumption: smoothness}-\ref{assumption: bounded variance} hold. If we optimize (\ref{regularized loss}) using Algorithm~\ref{alg:3} with synchronization gap $\tau =  \frac{T^{1/4}}{\sqrt{m}}$, $\eta = \frac{1}{4L \sqrt{T}}$, and $\gamma = \frac{1}{T^{5/8}}$, for the returned solutions $\hat{\boldsymbol{w}}$ and $\hat{\boldsymbol{\lambda}}$ it holds that: 
\begin{equation}
{\small
\begin{aligned}
    \min_{\bm{w}\in\mathcal{W}} \max_{\boldsymbol{\lambda}\in \Lambda}\mathbb{E}[F(\hat{\boldsymbol{w}},\boldsymbol{\lambda} )  & - F(\boldsymbol{w} ,\hat{\boldsymbol{\lambda}} )] \leq O\Big{(}\frac{D_{\mathcal{W}}^2+G_{w}^2}{\sqrt{T}} +\frac{D_{\Lambda}^2}{T^{3/8}}+\frac{G_\lambda^2  }{m^{1/2}T^{3/8}} +\frac{\sigma_{\lambda}^2}{m^{3/2}T^{3/8}}+ \frac{\sigma_w^2+\Gamma}{m \sqrt{T} }\Big{)}.\nonumber
\end{aligned}}
\end{equation}
\end{theorem}
The proof of Theorem~\ref{thm: regularized convex-linear} is deferred to Appendix~\ref{sec: proof DRFA-Prox convex}. Clearly, we obtain a convergence rate of  $O\left({1}/{T^{3/8}}\right)$, which is same as rate obtained in Theorem~\ref{theorem1} for {\sffamily{DRFA}} in non-regularized case.

\begin{theorem}[Nonconvex loss] \label{thm:regularized nonconvex-linear}
Assume each local function $f_i$ is nonconvex. Also, assume the conditions in  Assumptions~\ref{assumption: smoothness}-\ref{assumption: bounded variance} hold. If we optimize (\ref{regularized loss}) using Algorithm~\ref{alg:3} with synchronization gap $\tau =  T^{1/4}$, letting $\bm{w}^t = \frac{1}{m}\sum_{\mathcal{D}^{(\floor{ \frac{t}{\tau}})}} \bm{w}_i^{(t)}$ to denote the virtual average model at $t$th iterate, by choosing $\eta =   \frac{1}{4LT^{3/4}} $ and $\gamma = \frac{1}{\sqrt{T}} $, we have:

 {\small\begin{align*}{   \frac{1}{T}\sum_{t=1}^T \mathbb{E}\left[\left\| \nabla  \Phi_{1/2L}(\bm{w}^{t}) \right\|^2 \right] \leq   O\left(\frac{D_{\Lambda}^2 }{T^{1/8}}+ \frac{ \sigma_{\lambda}^2}{mT^{1/4}}+ \frac{ G_{\lambda}^2}{T^{1/4}}+ \frac{G_w  \sqrt{G_w^2+\sigma_w^2}}{T^{1/8}} +  \frac{D_{\mathcal{W}} (\sigma_w+\sqrt{\Gamma})}{T^{1/2} }\right)}.  
\end{align*}}

\end{theorem} 
The proof of Theorem~\ref{thm:regularized nonconvex-linear} is deferred to Appendix~\ref{sec: proof DRFA-Prox nonconvex}. Note that, we recover the same convergence rate as {\sffamily{DRFA}} on nonconvex losses (Theorem~\ref{theorem2}). However, we should remark that solving the proximal problem will take extra computation time, which is not reflected in the convergence rate.

\subsection{An alternative algorithm for regularized objective}\label{sec:alternative algorithm}
Here we present an alternative method similar to vanilla AFL~\cite{mohri2019agnostic} to optimize regularized objective~(\ref{regularized loss}), where we choose to do the full batch gradient ascent for $\boldsymbol{\lambda}$ every $\tau$ iterations according to $ \boldsymbol{\lambda}^{(s+1)} = \prod_{\Lambda}\left( \boldsymbol{\lambda}^{(s)} +  \gamma \nabla_{\lambda} F\left(\bar{\bm{w}}^{(s)}, \boldsymbol{\lambda}^{(s)}\right)\right)
$.   We establish convergence rates in terms of $\Phi(\bm{w})$ as in Definition~\ref{df1}, under assumption that $F(\cdot, \bm{\lambda})$ is strongly-convex or satisfies PL-condition~\cite{karimi2016linear} in $\bm{w}$, and strongly-concave in $\boldsymbol{\lambda}$. Due to lack of space, we present a summary of the rates and defer the exact statements to Appendix~\ref{app: alternative algorithm} and the proofs to Appendices~\ref{sec:scsc} and~\ref{sec:ncsc}.  

\noindent \textbf{Strongly-convex-strongly-concave case.} In this setting,  we obtain an $\tilde{O}\left({\tau}/{T}\right)$ rate. If we choose $\tau = 1$, which is fully synchronized SGDA, then we recover the same rate $\tilde{O}\left({1}/{T}\right)$ as in~\cite{mohri2019agnostic}. If we choose $\tau$ to be $O(\sqrt{T/m})$, we recover the rate $  \Tilde{O}\left( {1}/{\sqrt{mT}}  \right)$, which  achieves a linear speedup in the number of sampled workers (see Theorem~\ref{theorem2} in Appendix~\ref{app: alternative algorithm}).

\noindent\textbf{Nonconvex (PL condition)-strongly-concave case.} We also provide the convergence analysis when $F$ is nonconvex but satisfying the PL condition~\cite{karimi2016linear} in $\bm{w}$, and strongly concave in $\boldsymbol{\lambda}$.  In this setting, we also obtain an $\tilde{O}\left({ \tau}/{T}\right)$ convergence rate which is slightly worse than that of strongly-convex-strongly-concave case. The best known result of non-distributionally robust version of FedAvg on PL condition is $O({1}/{T})$ \cite{haddadpour2019convergence}, with $O(T^{1/3})$ communication rounds. It turns out that we trade some convergence rates to guarantee worst-case performance (see Theorem~\ref{thm3} in  Appendix~\ref{app: alternative algorithm}).

\section{Experiments}
In this section, we empirically verify {\sffamily{DRFA}} and compare its performance to other baselines. More experimental results are discussed in the Appendix~\ref{app:add_exp}. We implement our algorithm based on \texttt{Distributed} API of PyTorch~\cite{paszke2019pytorch} using \texttt{MPI} as our main communication interface, and on an \texttt{Intel Xeon E5-2695} \texttt{CPU} with $28$ cores. We use three datasets, namely, Fashion MNIST~\cite{xiao2017/online}, Adult~\cite{adult}, and Shakespeare~\cite{caldas2018leaf} datasets. The code repository used for these experiments can be found at: \url{https://github.com/MLOPTPSU/FedTorch/}

\noindent\textbf{Synchronization gap.}~To show the effects of synchronization gap on {\sffamily{DRFA}} algorithm, we run the first experiment on the Fashion MNIST dataset with logistic regression as the model. We run the experiment with $10$ devices and a server, where each device has access to only one class of data, making it distributionally heterogeneous. We use different synchronization gaps of $\tau\in\{5,10,15\}$, and set $\eta=0.1$ and $\gamma=8\times10^{-3}$. The results are depicted in Figure~\ref{comp_tau_fashion_mnist}, where out of all the test accuracies on each single local distribution, we report the worst one as the worst distribution accuracy. 
Based on our optimization scheme, we aim at optimizing the worst distribution accuracy (or loss), thus the measure depicted in Figure~\ref{comp_tau_fashion_mnist} is in accordance with our goal in the optimization.
It can be inferred that the smaller the synchronization gap is, the fewer number of iterations required to achieve $50\%$ accuracy in the worst distribution (Figure~\ref{fig:comp_fashion_mnnist_iteration}). However, the larger synchronization gap needs fewer number of communication and shorter amount of time to achieve $50\%$ accuracy in the worst distribution (Figure~\ref{fig:comp_fashion_mnist_comm_round} and~\ref{fig:comp_fashion_mnist_wall_clock}).

\begin{figure*}[t!]
		\centering
		\subfigure[]{
		\centering
		\includegraphics[width=0.31\textwidth]{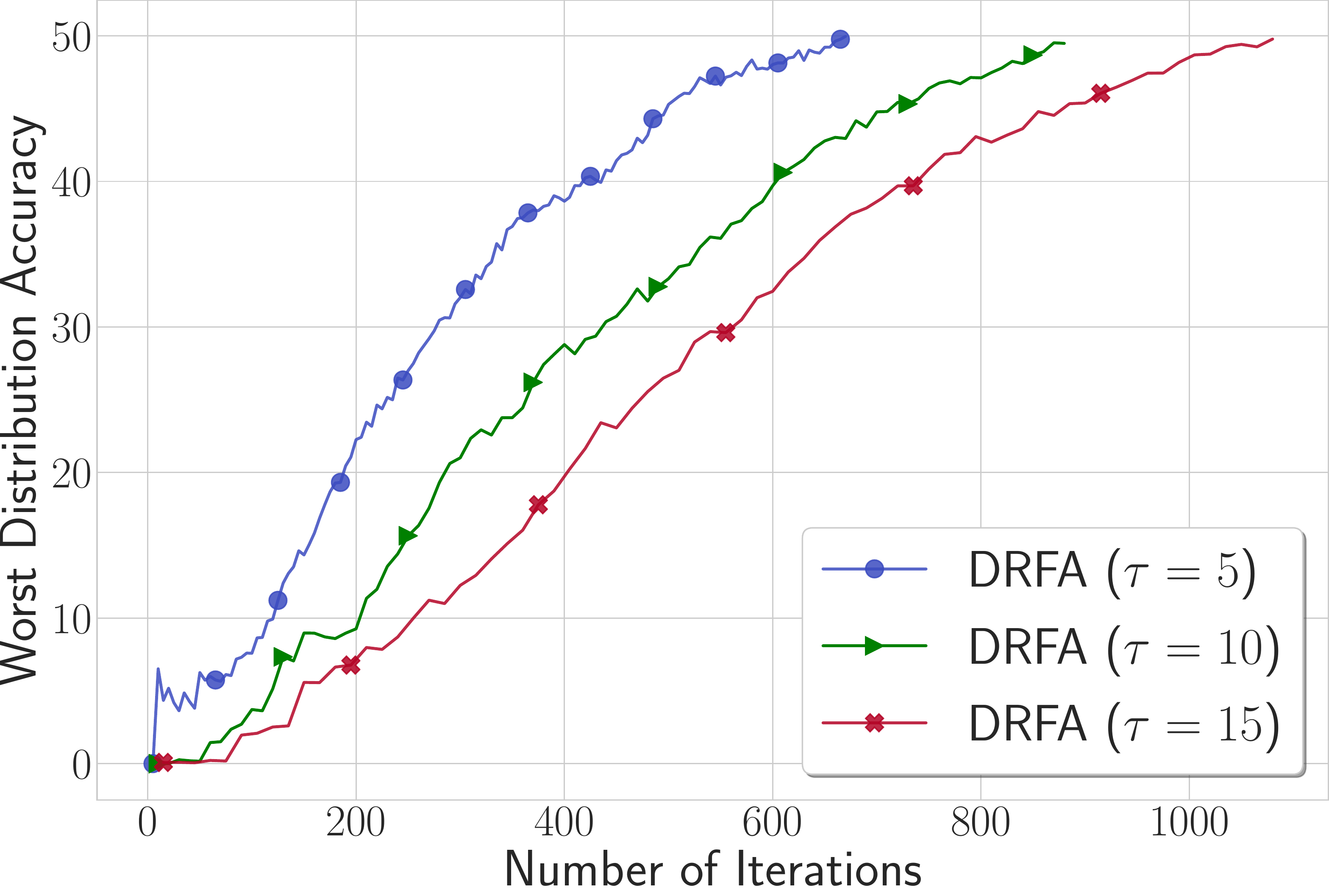}
		\label{fig:comp_fashion_mnnist_iteration}
		}
		\hfill
		\subfigure[]{
			\centering 
			\includegraphics[width=0.31\textwidth]{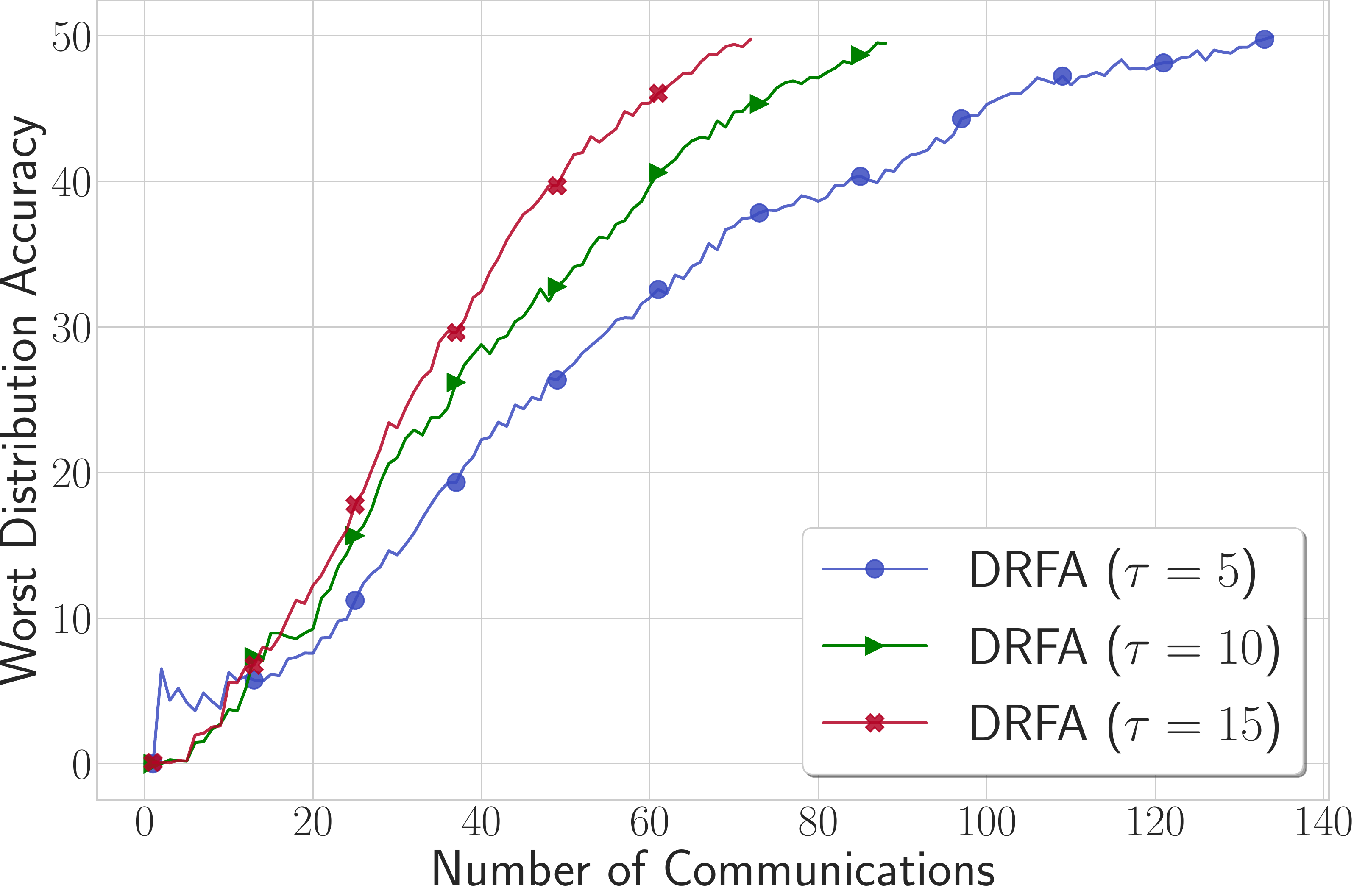}
			\label{fig:comp_fashion_mnist_comm_round}
			}
			\hfill
		\subfigure[]{
			\centering 
			\includegraphics[width=0.31\textwidth]{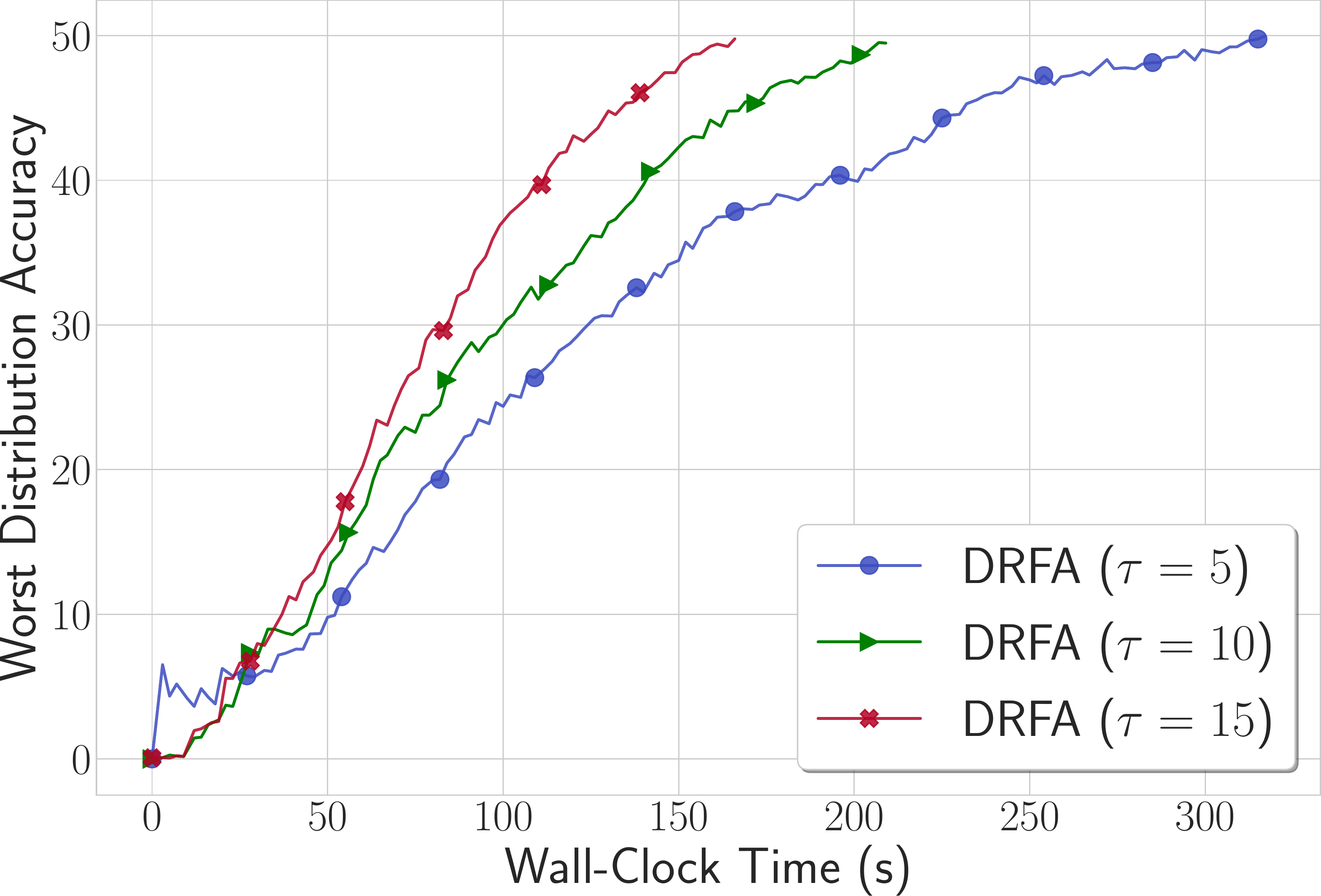}
			\label{fig:comp_fashion_mnist_wall_clock}
			}
	\caption[]{Comparing the effects of synchronization gap on the {\sffamily{DRFA}} algorithm on the Fashion MNIST dataset with a logistic regression model. The figures are showing the worst distribution accuracy during the training. }
	\label{comp_tau_fashion_mnist}
	\vspace{-0.4cm}
\end{figure*}

\begin{figure*}[t!]
		\centering
		\subfigure[]{
			\centering 
			\includegraphics[width=0.31\textwidth]{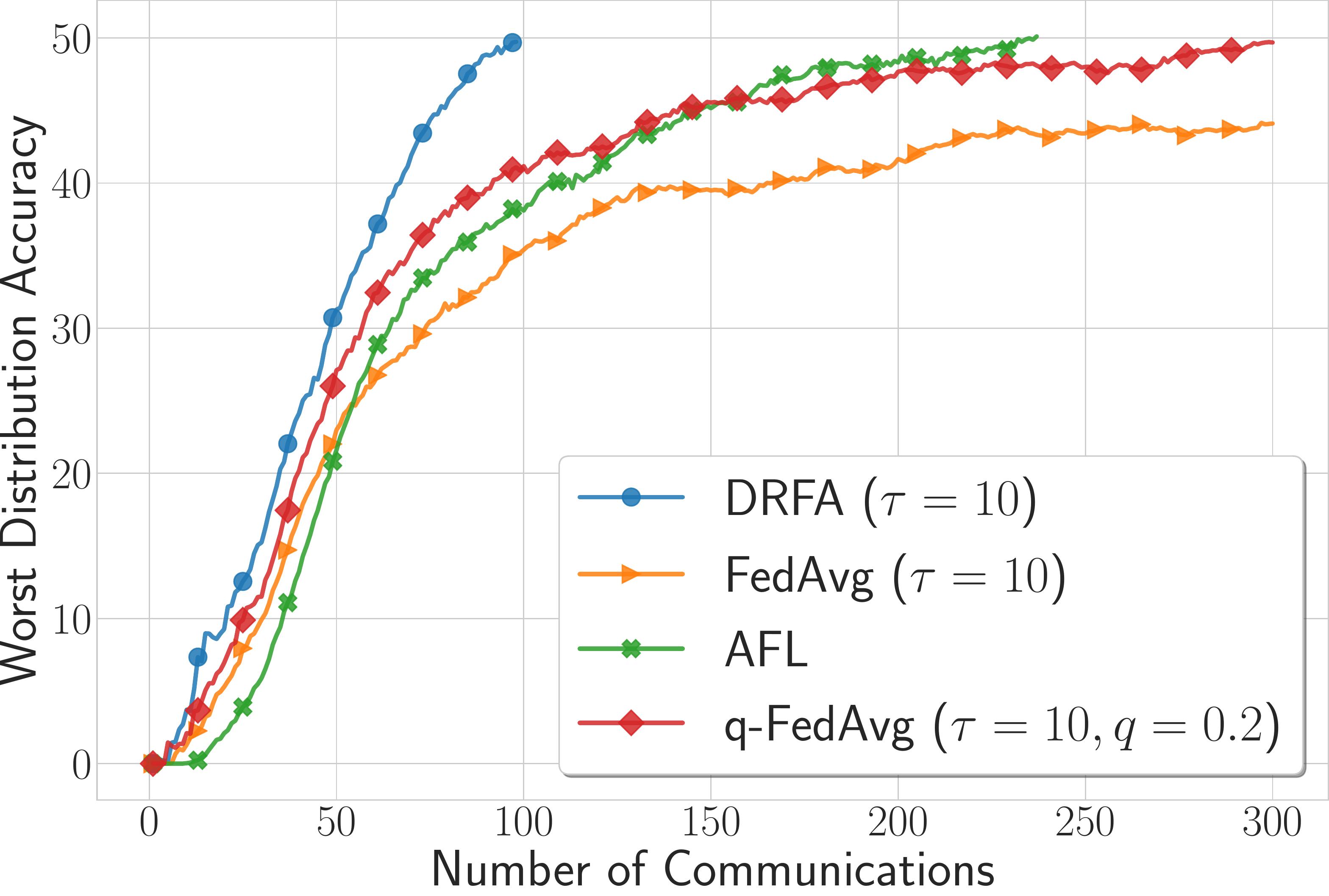}
			\label{fig:comp_base_fashion_mnist_comm_round}
			}
			\hfill
		\subfigure[]{
		\centering
		\includegraphics[width=0.31\textwidth]{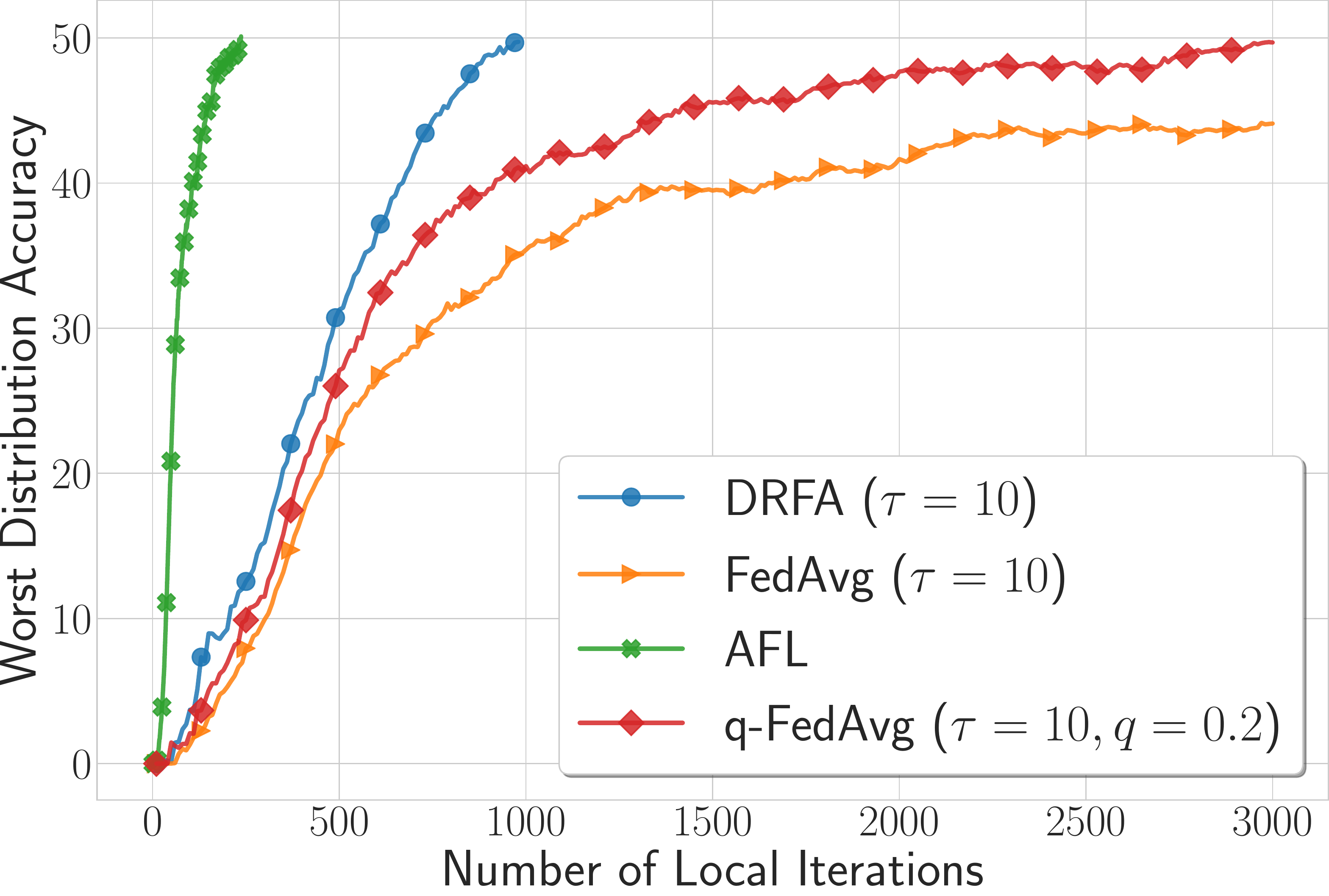}
		\label{fig:comp_base_fashion_mnnist_iteration}
		}
		\hfill
		\subfigure[]{
			\centering 
			\includegraphics[width=0.31\textwidth]{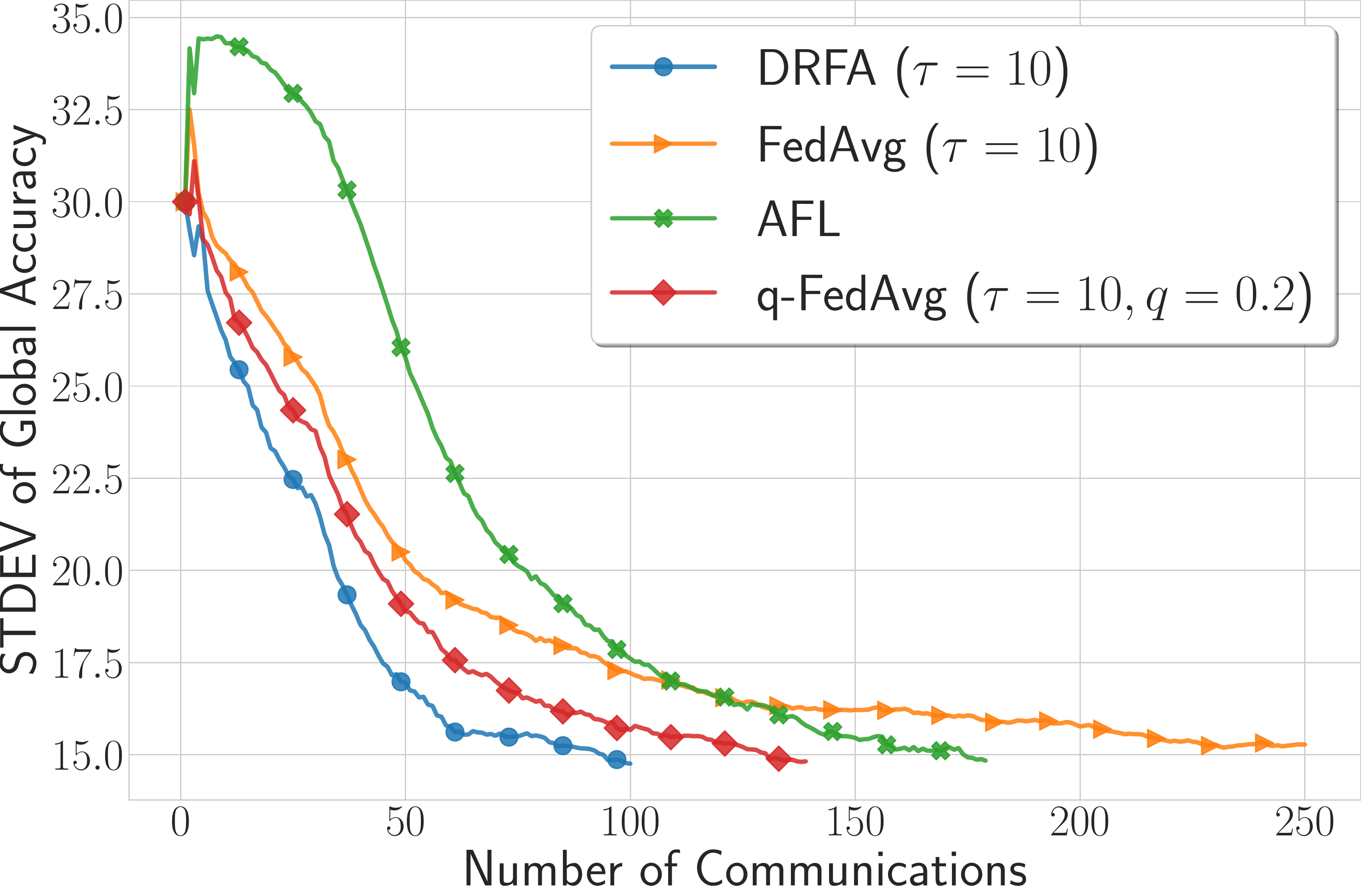}
			\label{fig:stdev}
			}
	\caption[]{Comparing {\sffamily{DRFA}} algorithm with AFL~\cite{mohri2019agnostic}, q-FedAvg~\cite{li2019fair}, and FedAvg on Fashion MNIST dataset with logistic regression. {\sffamily{DRFA}} can achieve the same level of worst distribution accuracy, with fewer number of communication rounds, and hence, lower runtime. It also efficiently decreases the variance among the performance of different nodes with fewer communication rounds.}
	\label{comp_base_fashion_mnist1}\vspace{-0.4cm}
\end{figure*}

\noindent\textbf{Comparison with baselines.} From the algorithmic point of view, the AFL algorithm~\cite{mohri2019agnostic} is a special case of our {\sffamily{DRFA}} algorithm, by setting the synchronization gap $\tau=1$. Hence, the first experiment suggests that we can increase the synchronization gap and achieve the same level of worst accuracy among distributions with fewer number of communications. In addition to AFL, q-FedAvg proposed by~\citet{li2019fair} aims at balancing the performance among different clients, and hence, improving the worst distribution accuracy. In this part, we compare {\sffamily{DRFA}} with AFL, q-FedAVG, and FedAvg.

To compare them, we run our algorithm, as well as AFL, q-FedAvg and FedAvg on Fashion MNIST dataset with logistic regression model on $10$ devices, each of which has access to one class of data. We set $\eta=0.1$ for all algorithms, $\gamma=8\times 10^{-3}$ for {\sffamily{DRFA}} and AFL, and $q=0.2$ for q-FedAvg. The batch size is $50$ and synchronization gap is $\tau=10$. 
Figure~\ref{fig:comp_base_fashion_mnnist_iteration} shows that AFL can reach to the $50\%$ worst distribution accuracy with fewer number of local iterations, because it updates the primal and dual variables at every iteration. 
However, Figure~\ref{fig:comp_base_fashion_mnist_comm_round} shows that {\sffamily{DRFA}} outperforms AFL, q-FedAvg and FedAvg in terms of number of communications, and subsequently, wall-clock time required to achieve the same level of worst distribution accuracy (due to much lower number of communication needed). 
\begin{wrapfigure}{r}{0.33\linewidth}
	\centering
	\vspace{-0.3cm}
		\includegraphics[width=0.33\textwidth]{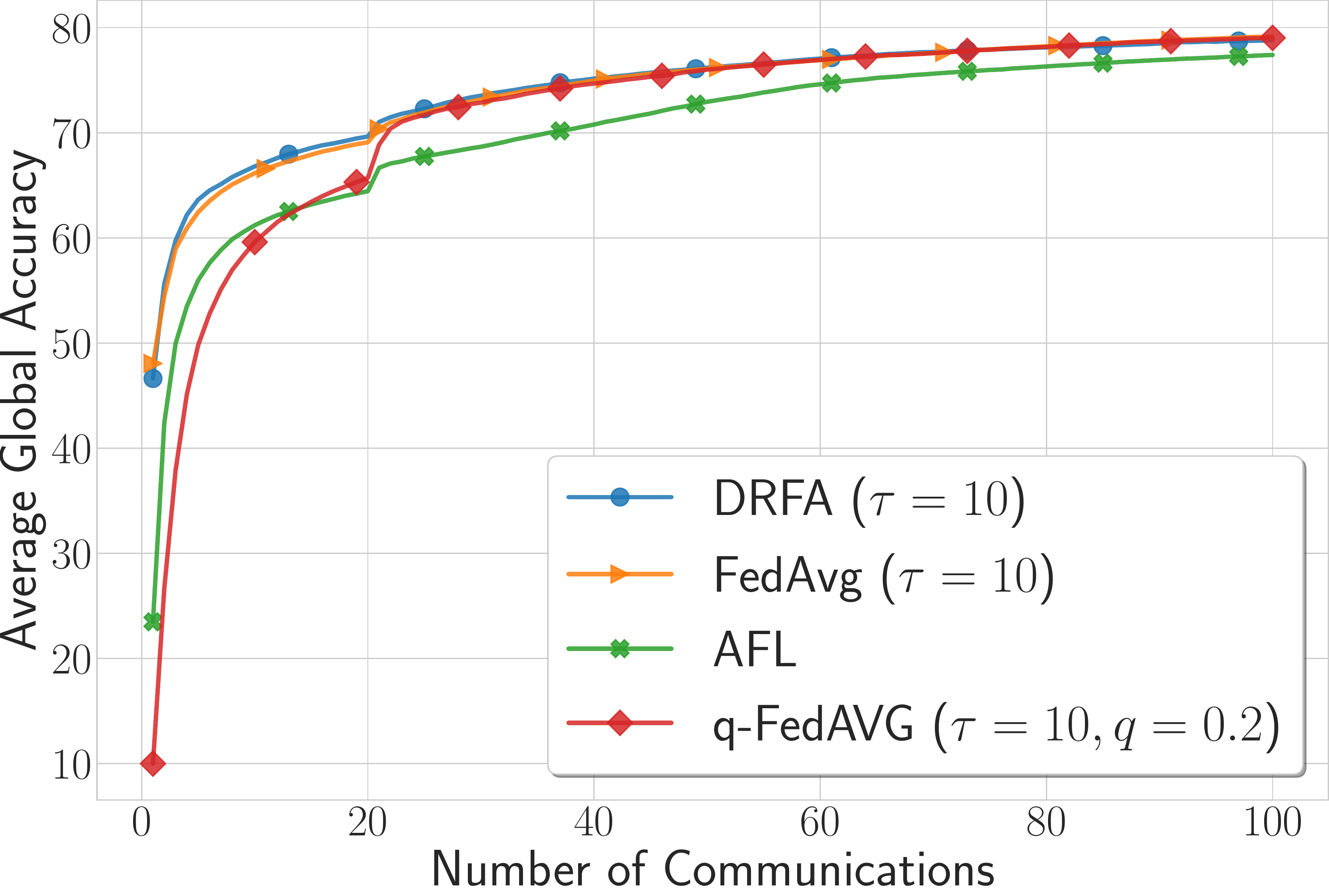}\vspace{-1mm}
	\caption[]{Averag global accuracy for each algorithm for $100$ rounds of communication. It shows that {\sffamily{DRFA}} keeps the same level of global accuracy as FedAvg, while it boosts its worst performing distribution accuracy.}\label{comp_base_fashion_mnist}
	\vspace{-0.3cm}
\end{wrapfigure}
Note that, q-FedAvg has is very close to AFL in terms of communication rounds, but it is far behind it in terms of local computations.
Also, note that FedAvg has the same computation complexity as {\sffamily{DRFA}} and q-FedAvg at each round but cannot reach the $50\%$ accuracy even after $300$ rounds of communication.
Similar to q-FedAvg, to show how different devices are performing, Figure~\ref{fig:stdev} depicts the standard deviation among the accuracy of different clients, which shows the level of fairness of the learned model among different clients. It can be inferred that  {\sffamily{DRFA}} can achieve the same level as AFL and q-FedAvg with fewer number of communication rounds, making it more efficient.
To compare the average performance of these algorithms, Figure~\ref{comp_base_fashion_mnist} shows the global training accuracy of them over $100$ rounds of communication on Fashion MNIST with logistic regression, where {\sffamily{DRFA}} performs as good as FedAvg in this regard. AFL needs more communication rounds to reach to the same level.

\section{Conclusion}
In this paper we propose a communication efficient scheme for distributionally robust federated model training. In addition, we give the first analysis of local SGD in distributed minimax optimization, under general smooth convex-linear, and nonconvex linear, strongly-convex-strongly-concave and nonconvex (PL-condition)-strongly concave settings. The experiments demonstrate the convergence of our method, and the distributional robustness of the learned model. The future work would be  improving obtained convergence rates due to gap we observed compared to centralized case. Another interesting question worth exploring will be investigating variance reduction schemes to achieve faster rates, in particular for updating mixing parameter. 

\clearpage
\section*{Broader Impact}
This work advocates a distributionally robust algorithm for federated learning. The algorithmic solution is designed to preserve the privacy of users, while training a high quality model. The proposed algorithm tries to minimize the maximum loss among worst case distribution over clients' data. Hence, we can ensure that even if the data distribution among users is highly heterogeneous, the trained model is reasonably good for everyone, and not benefiting only a group of clients. This will ensure the fairness in training a global model with respect to every user, and it is vitally important for critical decision making systems such as healthcare. In such a scenario, the model learned by simple algorithms such as FedAvg would have an inconsistent performance over different distributions, which is not acceptable. However, the resulting model from our algorithm will have robust performance over different distributions it has been trained on.
\section*{Acknowledgements}
This work has been done using the Extreme Science and Engineering Discovery Environment (XSEDE) resources, which is supported by National Science Foundation under grant number ASC200045. We are also grateful for the GPU donated by NVIDIA that was used in this research.

\bibliography{references}  
\bibliographystyle{plainnat}

\appendix
\clearpage
\appendix
\makesuptitle

\appendixwithtoc
\newpage

\section{Additional Experiments}\label{app:add_exp}
In this section, we further investigate the effectiveness of the proposed {\sffamily{DRFA}} algorithm. To do so, we use the Adult and Shakespeare datasets.

\noindent \textbf{Experiments on Adult dataset.}~The Adult dataset contains census data, with the target of predicting whether the income is greater or less than $\$50K$. The data has $14$ features from age, race, gender, among others. It has $32561$ samples for training distributed across different groups of sensitive features. One of these sensitive features is gender, which has two groups of ``male'' and ``female''. The other sensitive feature we will use is the race, where it has $5$ groups of ``black'', ``white'', ``Asian-Pac-Islander'', ``Amer-Indian-Eskimo'', and ``other''. We can distribute data among nodes based on the value of these features, hence make it heterogeneously distributed. 

For the first experiment, we distribute the training data across $10$ nodes, 5 of which contain only data from the female group and the other $5$ have the male group's data. Since the size of different groups' data is not equal, the data distribution is unbalanced among nodes. Figure~\ref{comp_base_adult9} compares {\sffamily{DRFA}} with AFL~\cite{mohri2019agnostic}, q-FedAvg~\cite{li2019fair}, and FedAvg~\cite{mcmahan2017communication} on the Adult dataset, where the data is distributed among the nodes based on the gender feature. We use logistic regression as the loss function, the learning rate is set to $0.1$ and batch size is $50$ for all algorithms, $\gamma$ is set to $0.2$ for both {\sffamily{DRFA}} and AFL, and $q=0.5$ is tuned for the best results for q-FedAvg. The worst distribution or node accuracy during the communication rounds shows that {\sffamily{DRFA}} can achieve the same level of worst accuracy with a far fewer number of communication rounds, and hence, less overall wall-clock time. However, AFL computational cost is less than that of {\sffamily{DRFA}}. Between each communication rounds {\sffamily{DRFA}}, q-FedAvg and FedAvg have $10$ update steps. FedAvg after the same number of communications as AFL still cannot reach the same level of worst accuracy. Figure~\ref{fig:comp_base_adult9_stdev} shows the standard deviation of accuracy among different nodes as a measure for the fairness of algorithms. It can be inferred that {\sffamily{DRFA}} efficiently decreases the variance with a much fewer number of communication rounds with respect to other algorithms.
\begin{figure*}[h!]
		\centering
		\subfigure[]{
			\centering 
			\includegraphics[width=0.31\textwidth]{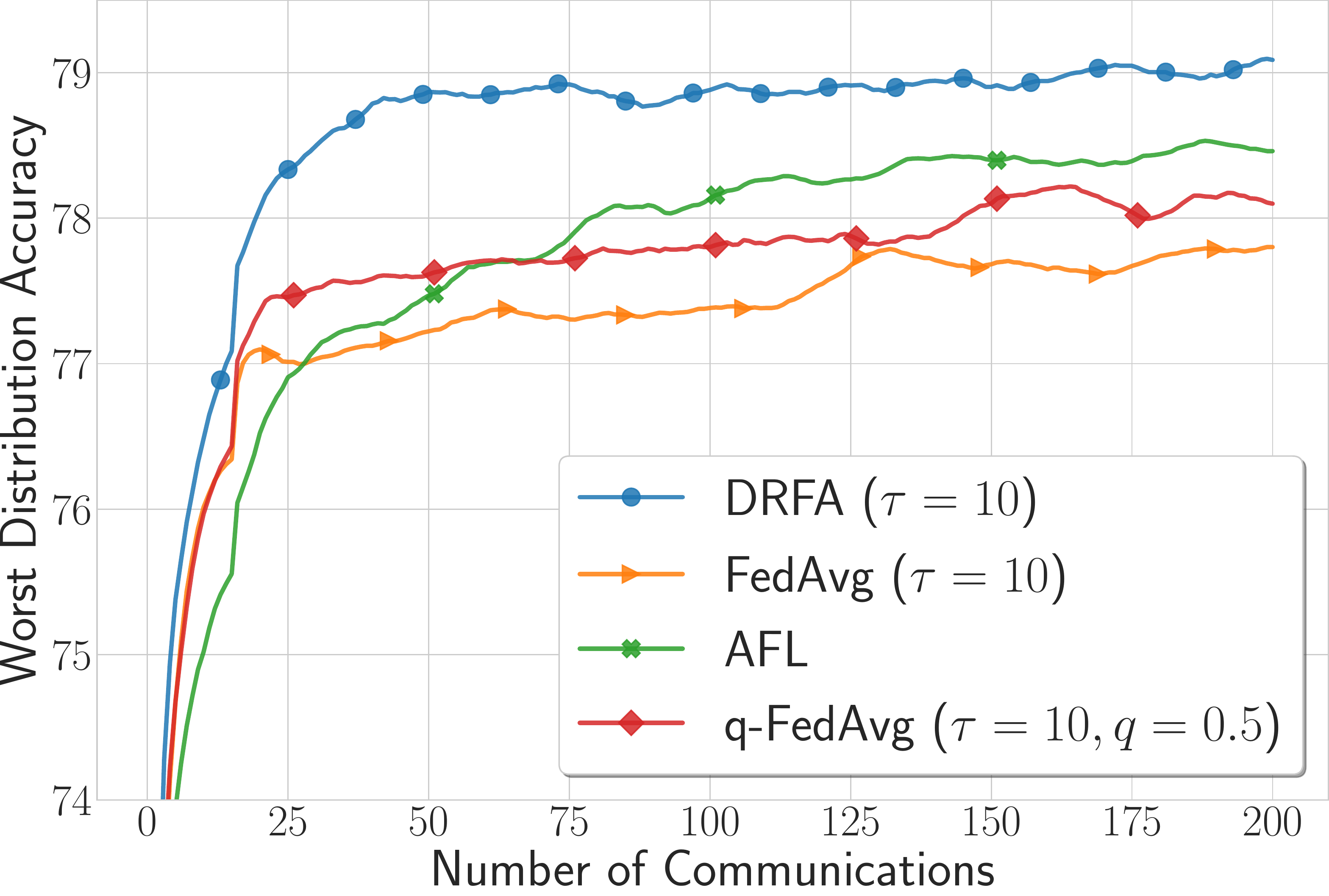}
			\label{fig:comp_base_adult9_comm_round}
			}
			\hfill
		\subfigure[]{
		\centering
		\includegraphics[width=0.31\textwidth]{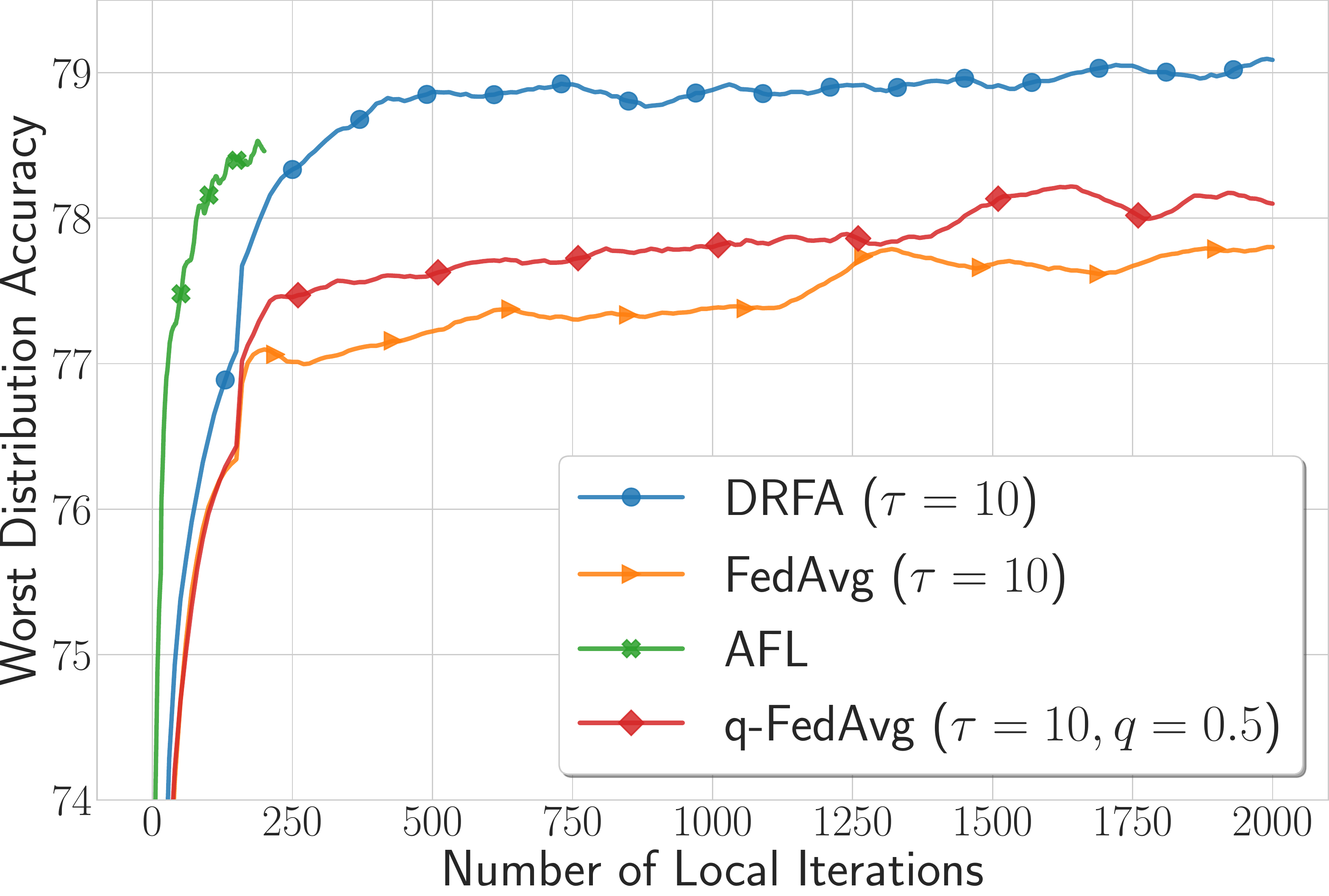}
		\label{fig:comp_base_adult9_iteration}
		}
		\hfill
		\subfigure[]{
			\centering 
			\includegraphics[width=0.31\textwidth]{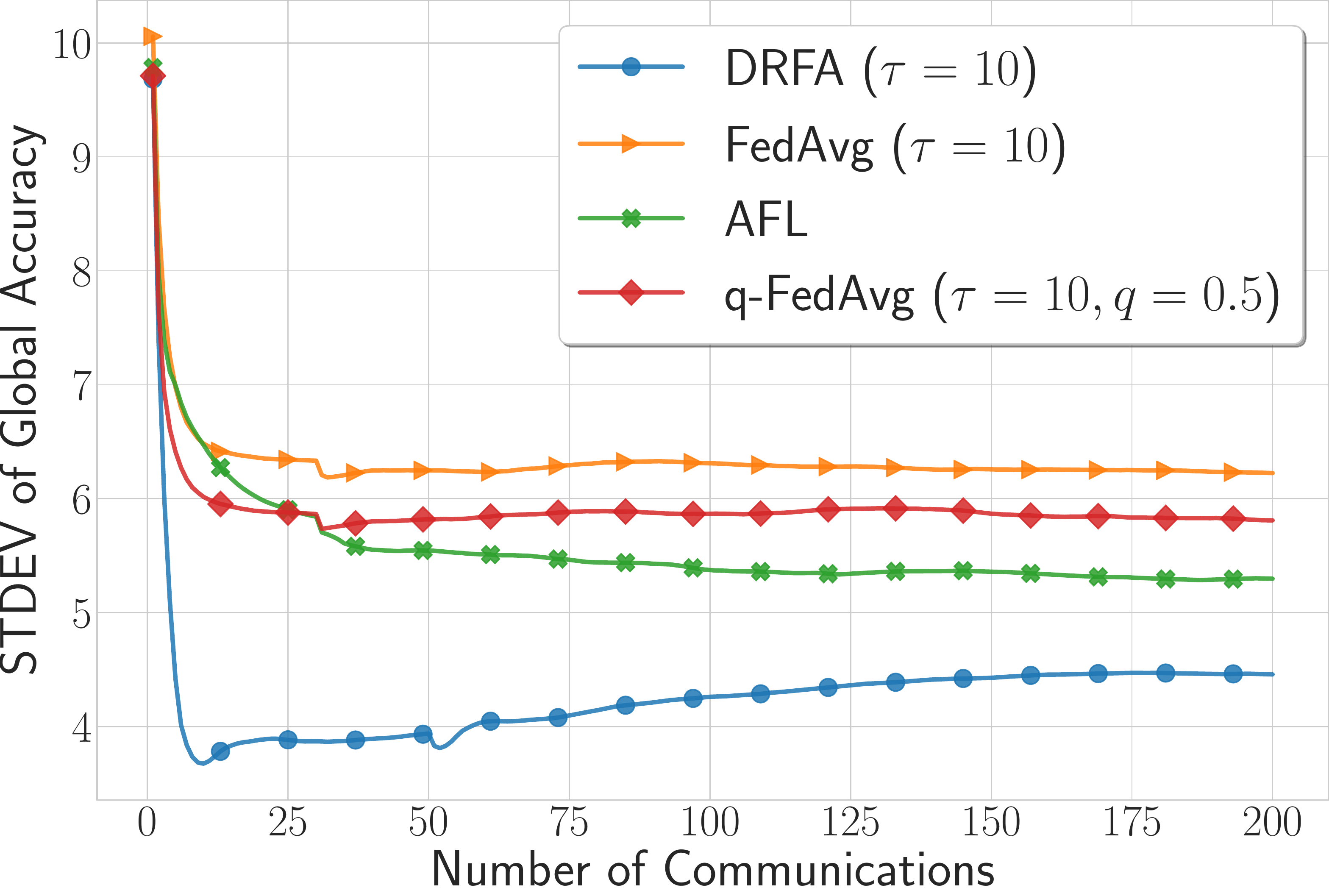}
			\label{fig:comp_base_adult9_stdev}
			}
	\caption[]{Comparing the worst distribution accuracy on {\sffamily{DRFA}}, AFL, q-FedAvg, and FedAVG on the Adult dataset. We have $10$ nodes, and data is distributed among them based on the gender feature. The loss function is logistic regression. {\sffamily{DRFA}} needs a fewer number of communications to reach the same worst distribution accuracy than the AFL and q-FedAvg algorithms. Also, {\sffamily{DRFA}} efficiently decreases the variance of the performance of different clients.}
	\label{comp_base_adult9}
\end{figure*}

Next, we  distribute the Adult data among clients based on the ``race'' feature, which has $5$ different groups. Again the size of data among these groups is not equal and makes the distribution unbalanced. We distribute the data among $10$ nodes, where every node has only data from one group of the race feature. For this experiment, we use a nonconvex loss function, where the model is a multilayer perceptron (MLP) with $2$ hidden layers, each with $50$ neurons. The first layer has $14$ and the last layer has $2$ neurons. The learning rate is set to $0.1$ and batch size is $50$ for all algorithms, the $\gamma$ is set to $0.2$ for {\sffamily{DRFA}} and AFL, and the $q$ parameter in q-FedAvg is tuned for $0.5$. Figure~\ref{comp_base_adult8} shows the results of this experiment, where again, {\sffamily{DRFA}} can achieve the same worst-case accuracy with a much fewer number of communications than AFL and q-FedAvg. In this experiment, with the same number of local iterations, AFL still cannot reach to the {\sffamily{DRFA}} performance. In addition, the variance on the performance of different clients in Figure~\ref{fig:comp_base_adult8_stdev} suggests that {\sffamily{DRFA}} is more successful than q-FedAvg to balance the performance of clients.

\begin{figure*}[t!]
		\centering
				\subfigure[]{
			\centering 
			\includegraphics[width=0.31\textwidth]{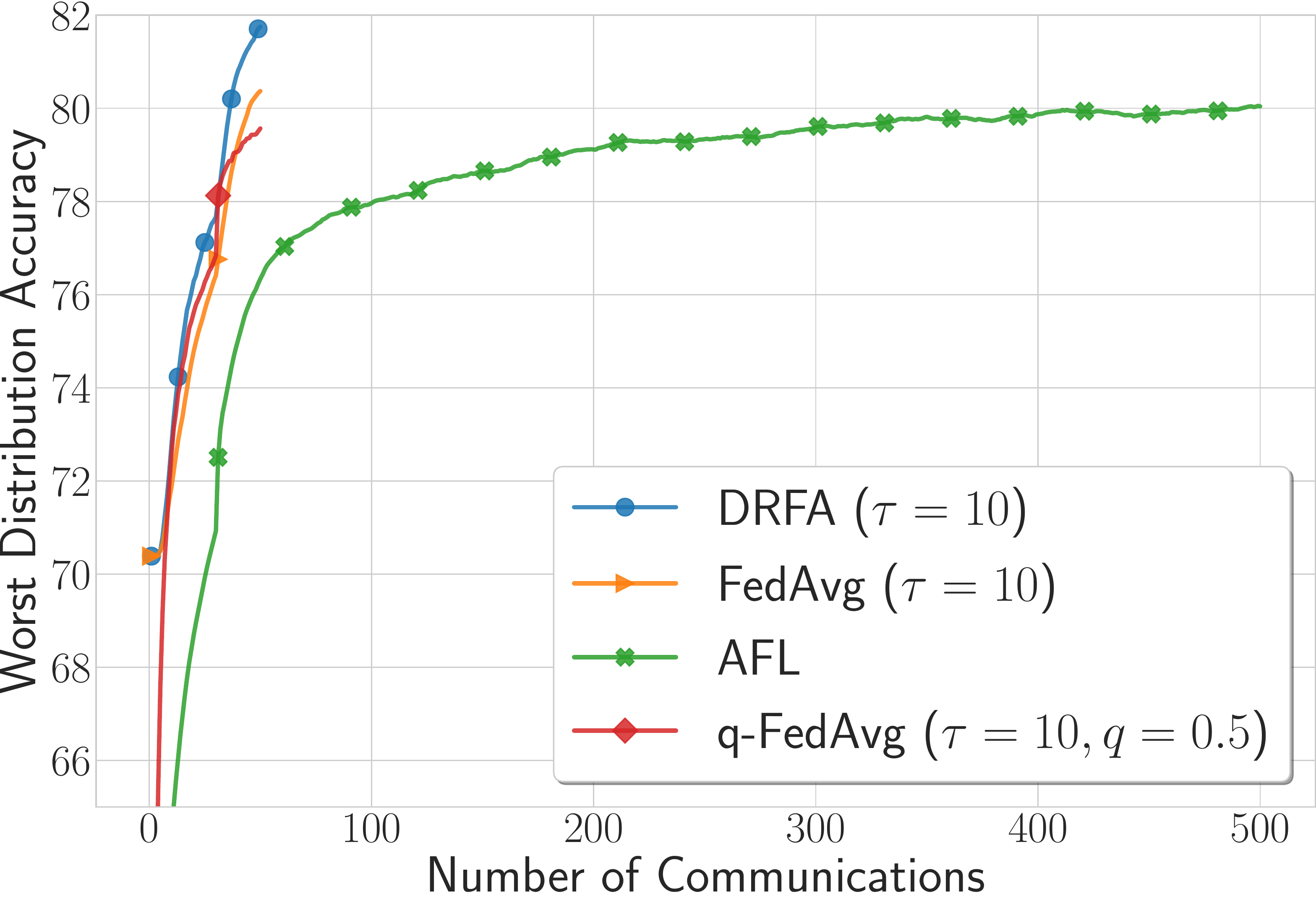}
			\label{fig:comp_base_adult8_comm_round}
			}
			\hfill
		\subfigure[]{
		\centering
		\includegraphics[width=0.31\textwidth]{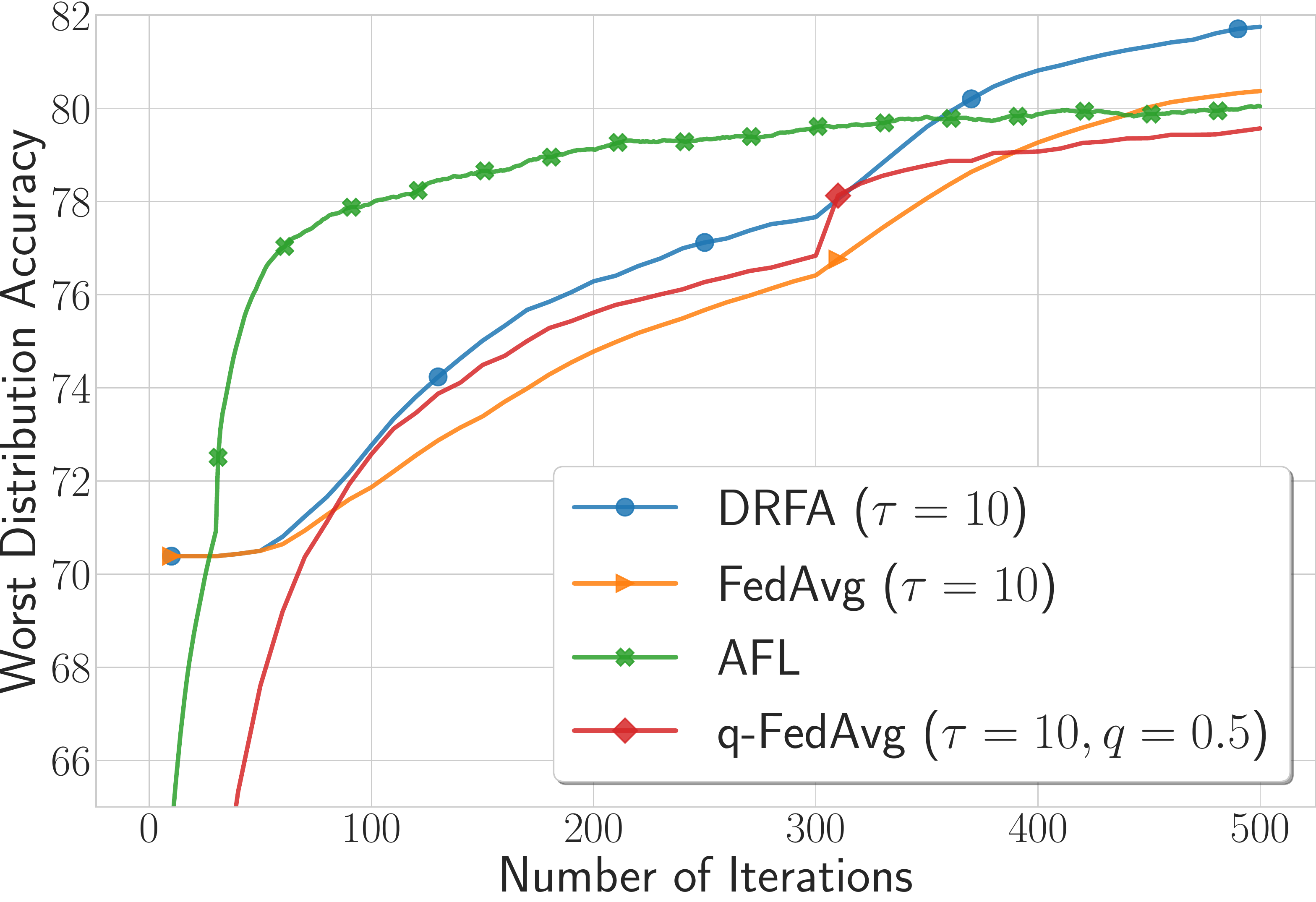}
		\label{fig:comp_base_adult8_iteration}
		}
		\hfill
		\subfigure[]{
			\centering 
			\includegraphics[width=0.31\textwidth]{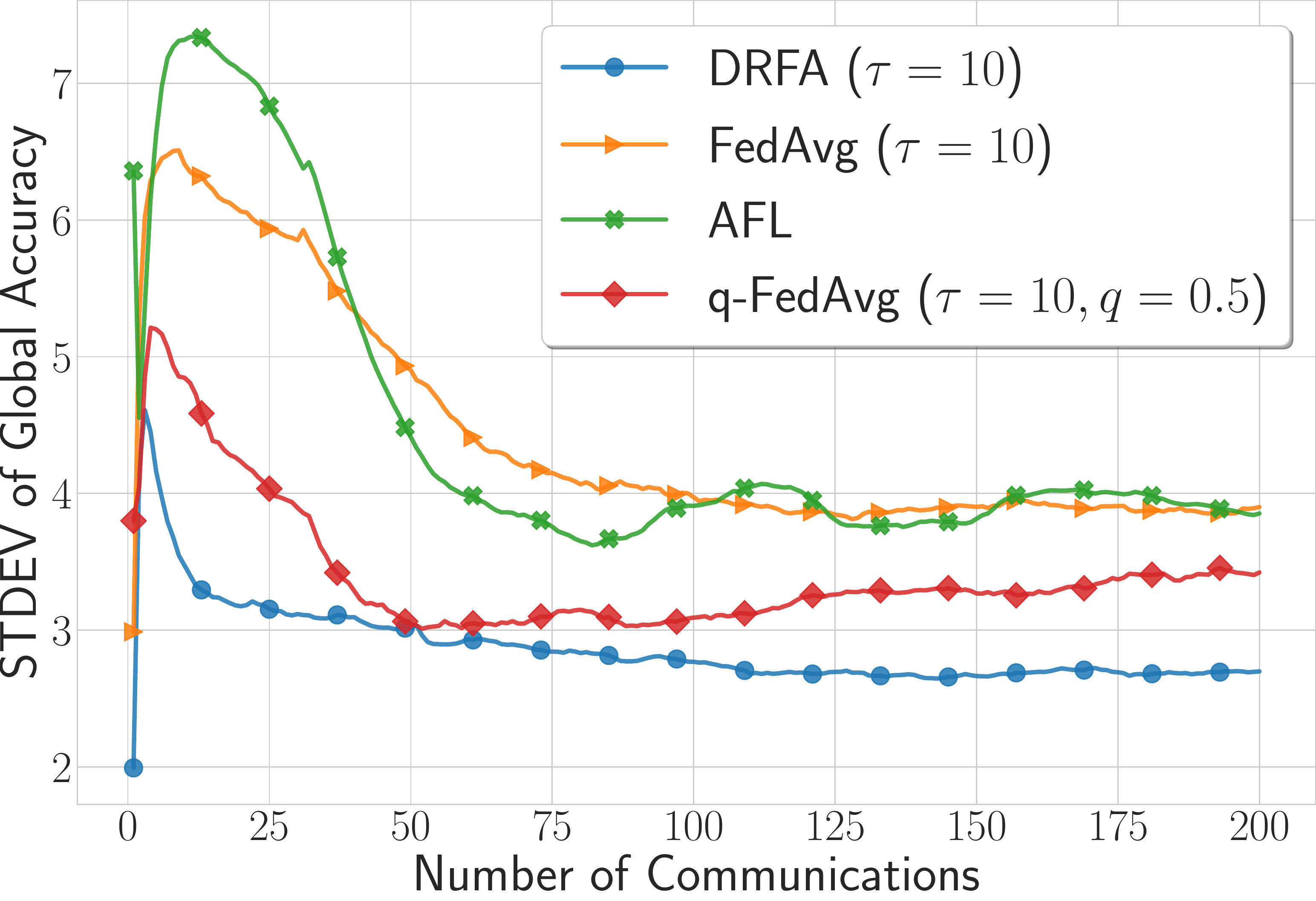}
			\label{fig:comp_base_adult8_stdev}
			}
	\caption[]{Comparing the worst distribution accuracy on {\sffamily{DRFA}}, AFL, q-FedAvg, and FedAvg with the Adult dataset. We have $10$ nodes, and data is distributed among them based on the race feature. The model is an MLP with $2$ hidden layers, each with $50$ neurons and a cross-entropy loss function. {\sffamily{DRFA}} needs a fewer number of communications to reach the same worst distribution accuracy than the AFL and q-FedAvg algorithms. Moreover, {\sffamily{DRFA}} is more efficient in reducing the performance variance among different clients than q-FedAvg. }
	\label{comp_base_adult8}
\end{figure*}

\noindent \textbf{Experiments on Shakespeare dataset.}~Now, we run the same experiments on the Shakespeare dataset. This dataset contains the scripts from different Shakespeare's plays divided based on the character in each play. The task is to predict the next character in the text, providing the preceding characters. For this experiment, we use $100$ clients' data to train our RNN model. The RNN model comprises an embedding layer from $86$ characters to $50$, followed by a layer of GRU~\cite{cho2014learning} with $50$ units. The output is going through a fully connected layer with an output size of $86$ and a cross-entropy loss function. We use the batch size of $2$ with $50$ characters in each batch. The learning rate is optimized to $0.8$ for the FedAvg and used for all algorithms. The $\gamma$ is tuned to the $0.01$ for AFL and {\sffamily{DRFA}}, and $q=0.1$ is the best for the q-FedAvg. Figure~\ref{comp_base_shaks} shows the results of this experiment on the Shakespeare dataset. It can be seen that {\sffamily{DRFA}} and FedAvg can reach to the same worst distribution accuracy compared to AFL and q-FedAvg. The reason that FedAvg is working very well in this particular dataset is that the distribution of data based on the characters in the plays does not make it heterogeneous. In settings close to homogeneous distribution, FedAvg can achieve the best results, with {\sffamily{DRFA}} having a slight advantage over that.

\begin{figure*}[t!]
		\centering
				\subfigure[]{
			\centering 
			\includegraphics[width=0.31\textwidth]{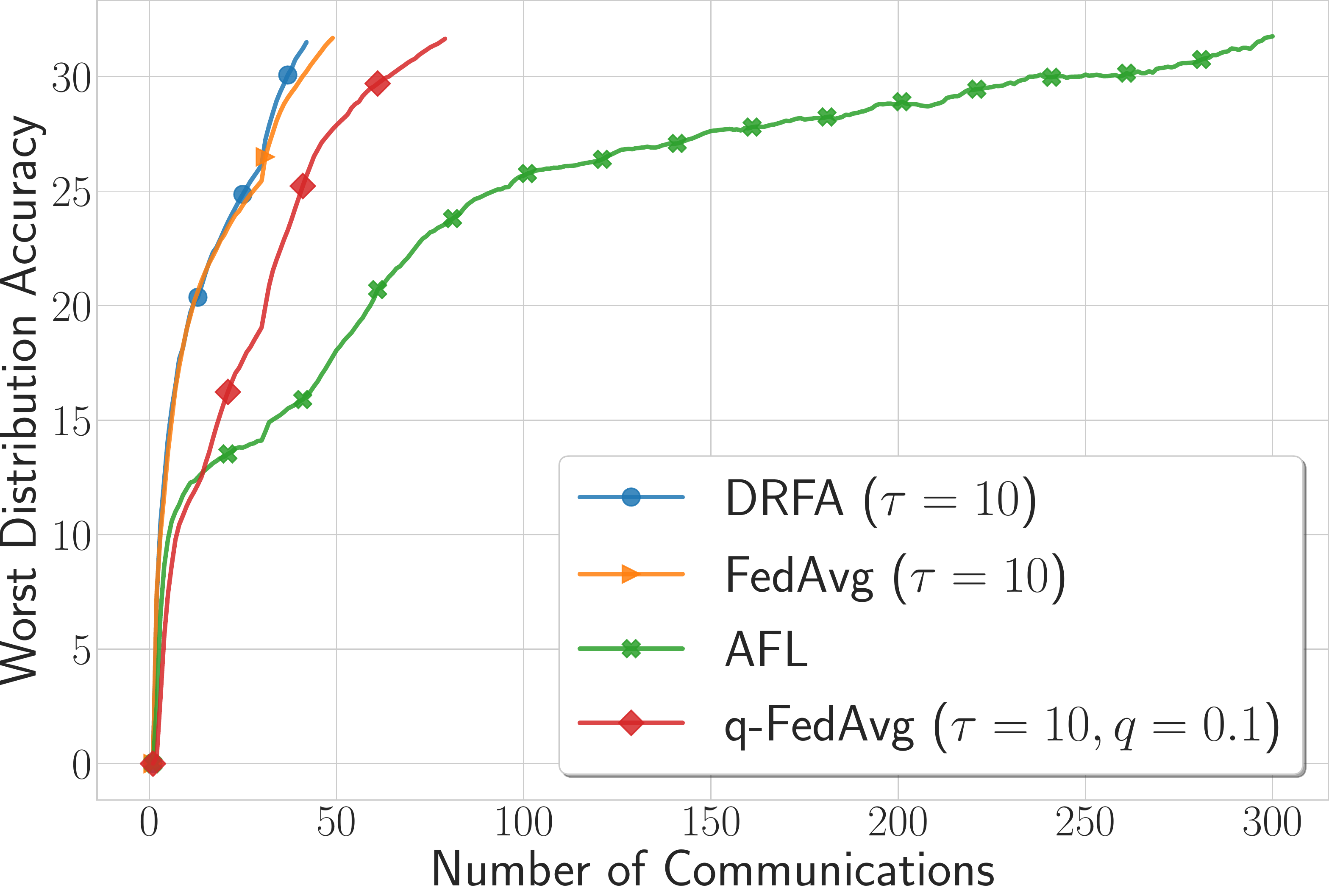}
			\label{fig:comp_base_shaks_comm_round}
			}
			\hfill
		\subfigure[]{
		\centering
		\includegraphics[width=0.31\textwidth]{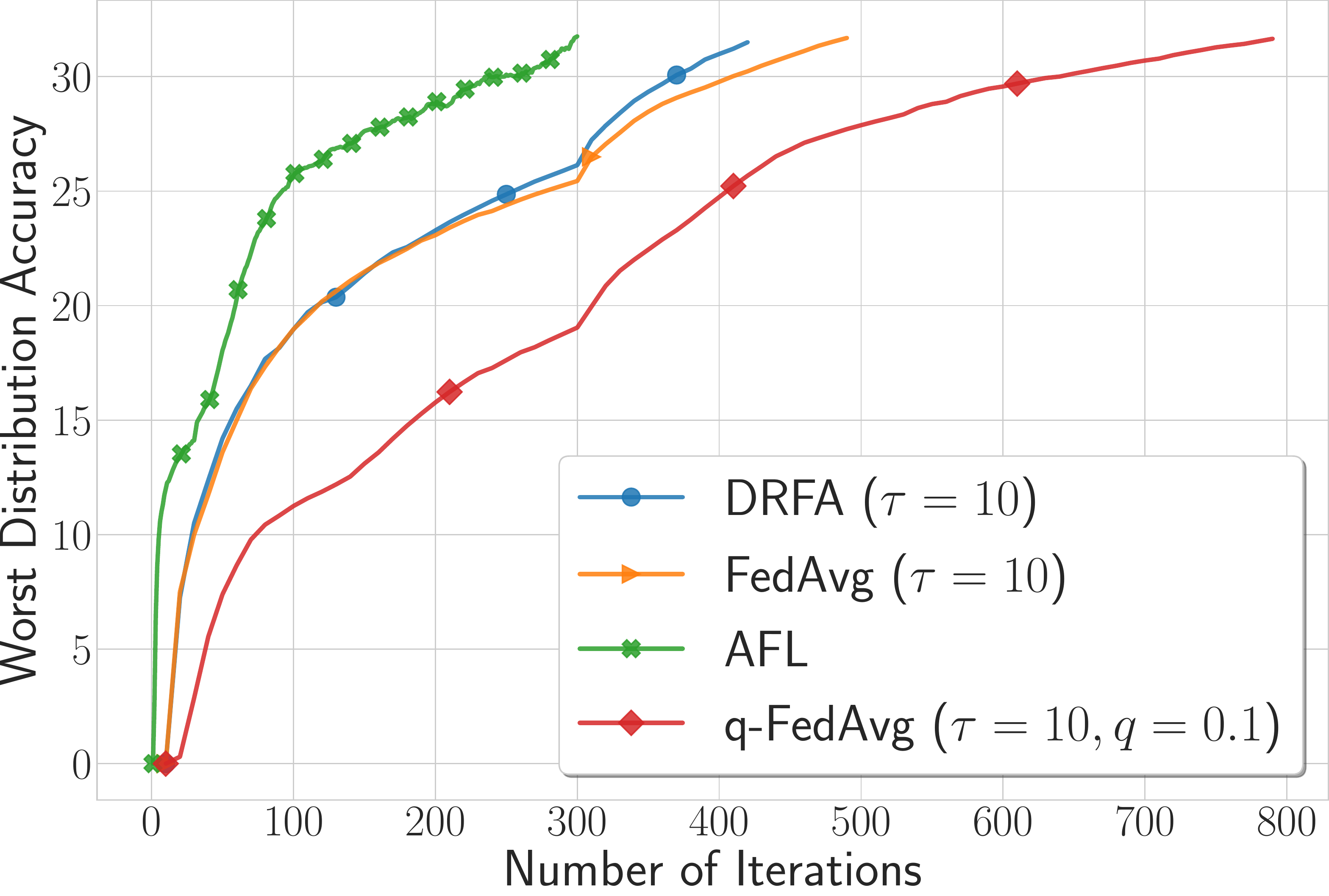}
		\label{fig:comp_base_shaks_iteration}
		}
		\hfill
		\subfigure[]{
			\centering 
			\includegraphics[width=0.31\textwidth]{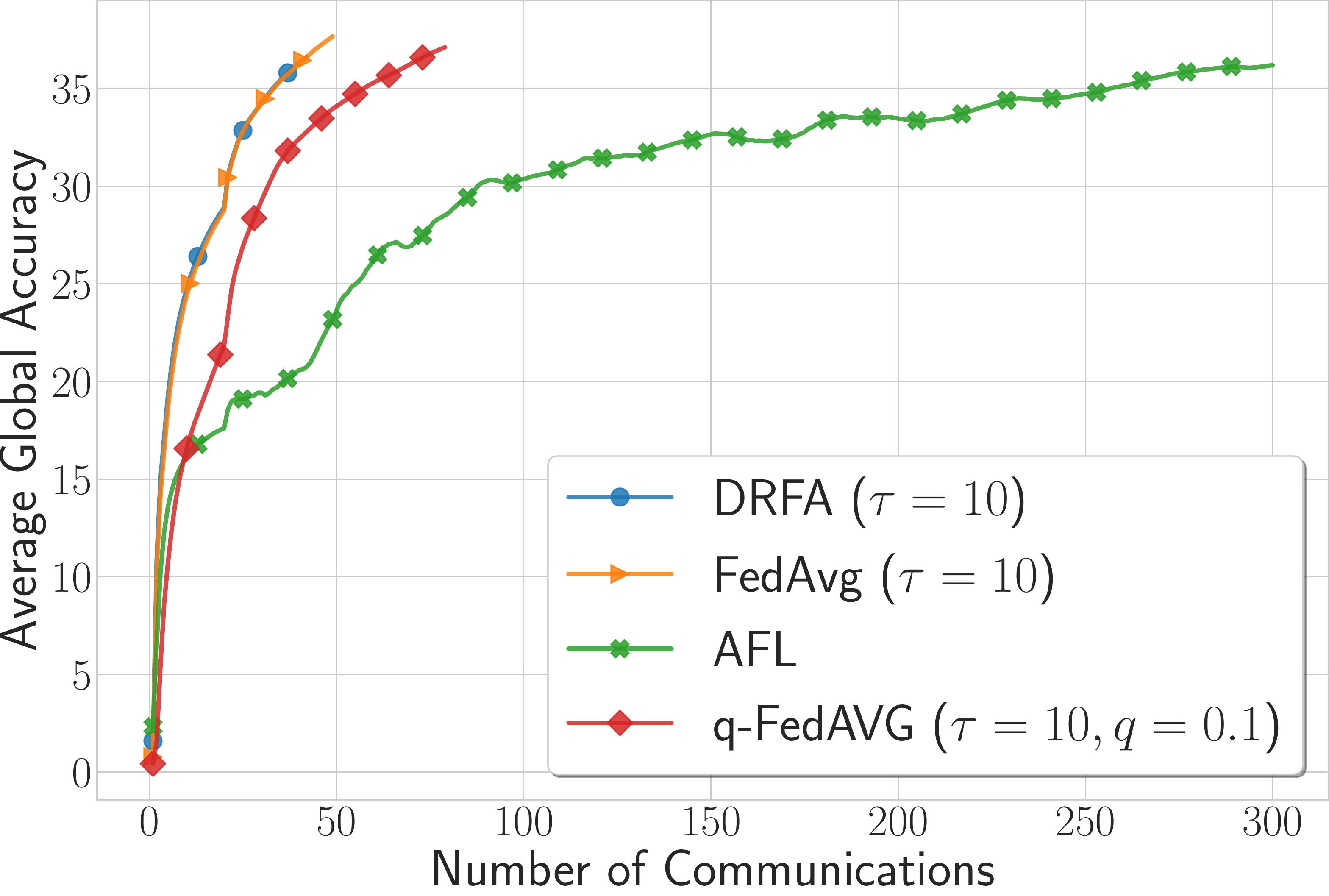}
			\label{fig:comp_base_shaks_avg}
			}
	\caption[]{ Comparing different algorithms on training an RNN on Shakespeare dataset using $100$ clients. {\sffamily{DRFA}} and FedAvg outperform the other two algorithms in terms of communication efficiency, however, AFL can achieve the same level with lower computation cost. In the average performance, AFL requires much more communication to reach to the same level as FedAvg and {\sffamily{DRFA}}. }
	\label{comp_base_shaks}
\end{figure*}

\section{Formal Convergence Theory for Alternative Algorithm in Regularized Case}\label{app: alternative algorithm}
Here, we will present the formal convergence theory of the algorithm we described in Section~\ref{sec:alternative algorithm}, where we use full batch gradient ascent to update $\boldsymbol{\lambda}^{(s)}$. To do so, the server sends the current global model $\bar{\bm{w}}^{(s)}$ to all clients and each client evaluates the global model on its local data shards and send $f_i(\bar{\bm{w}}^{(s)})$ back to the server. Then the server can compute the full gradient over dual parameter $\bm{\lambda}$ and take a gradient ascent (GA) step to update it. The algorithm is named {\sffamily{DRFA-GA}} and described in Algorithm~\ref{alg:5}. We note that {\sffamily{DRFA-GA}} can be considered as communication-efficient variant of AFL, but without sampling clients to evaluate the gradient at dual parameter. We conduct the convergence analysis on the setting where the regularized term is strongly-concave in $\boldsymbol{\lambda}$, and loss function is strongly-convex and nonconvex but satisfying Polyak-Łojasiewicz (PL) condition in $\bm{w}$. So, our theory includes strongly-convex-strongly-concave and nonconvex (PL condition)-strongly-concave cases.

\noindent \textbf{Strongly-Convex-Strongly-Concave case.~} We start by stating the convergence rate  when the individual local objectives are strongly convex and the regularizer $g(\bm{\lambda})$ is strongly concave in $\bm{\lambda}$, making the global objective $F (\boldsymbol{w},\boldsymbol{\lambda}) := \sum_{i=1}^N \lambda_i  f_i(\boldsymbol{w}) +g(\boldsymbol{\lambda})$ also strongly concave in $\bm{\lambda}$.

\begin{theorem}\label{theorem2}
Let each local function $f_i$ be $\mu$-strongly convex, and global function $F$ is $\mu$-strongly concave in $\boldsymbol{\lambda}$. Under Assumptions~\ref{assumption: smoothness}, \ref{assumption: bounded gradient},\ref{assumption: bounded domain},\ref{assumption: bounded variance}, if we optimize (\ref{regularized loss}) using the {\sffamily{DRFA-GA}} (Algorithm~\ref{alg:5}) with synchronization gap $\tau$, choosing  learning rates as $\eta = \frac{4\log T}{\mu T}$ and $\gamma = \frac{1}{ L}$ and $T \geq \frac{16\alpha \log T}{\mu}$, where $\alpha=\kappa L+L$,  using the  averaging scheme $\hat{\boldsymbol{w}} =  \frac{2}{mT} \sum_{t=T/2}^T  \sum_{i\in{\mathcal{D}^{(\floor*{\frac{t}{\tau}})}}} \boldsymbol{w}^{(t)}_i$ we have:
\begin{equation*}
 \begin{aligned}   \mathbb{E}[\Phi(\hat{\boldsymbol{w}})-\Phi(\boldsymbol{w}^*)]  = \Tilde{O}\left( \frac{\mu D_{\mathcal{W}}^2}{T} + \frac{\kappa^2L\tau D_{\Lambda}^2 }{T}  
      + \frac{\sigma_w^2+ G_w^2}{\mu mT}  +     \frac{\kappa^2\tau^2(\sigma_w^2 + \Gamma)}{\mu T^2} +  \frac{\kappa^6\tau^2G_w^2}{\mu T^2}\right),
 \end{aligned}
\end{equation*}
where $\kappa = L/\mu$, and $\bm{w}^*$ is the minimizer of $\Phi$.
\begin{proof}
The proof is given in Section~\ref{sec:scsc}.
\end{proof}
\end{theorem}

\begin{corollary}\label{Coro: 2}
Continuing with Theorem~\ref{theorem2}, if we choose $\tau =  \sqrt{T/m}$,  we recover the rate:
\begin{align*}
  \mathbb{E}[\Phi(\hat{\boldsymbol{w}})-\Phi(\boldsymbol{w}^*)]  = \Tilde{O}\left( \frac{\kappa^2L D_{\Lambda}^2 }{ \sqrt{mT}} +\frac{\mu D_{\mathcal{W}}^2}{T} +\frac{\kappa^2 (\sigma_w^2 + \Gamma)+\kappa^6G_w^2}{\mu mT } \right).
\end{align*}
\end{corollary}
Here we obtain $\tilde{O}\left(\frac{\tau}{T}\right)$ rate in Theorem~\ref{theorem2}. If we choose $\tau = 1$, which is fully synchronized  SGD, then we recover the same rate $\tilde{O}\left(\frac{1}{T}\right)$ as in vanilla agnostic federated learning~\cite{mohri2019agnostic}. If we choose $\tau$ to be $O(\sqrt{T/m})$, we recover the rate $  \Tilde{O}\left( \frac{1}{\sqrt{mT}} + \frac{1}{mT} \right)$, which can achieve linear speedup with respect to number of sampled workers. The dependency on gradient dissimilarity $\Gamma$ shows that the data heterogeneity will slow down the rate,  but will not impact the dominating term.

\begin{algorithm}[t]
	\renewcommand{\algorithmicrequire}{\textbf{Input:}}
	\renewcommand{\algorithmicensure}{\textbf{Output:}}
	\caption{Distributionally Robust Federated Averaging: Gradient Ascent   ({\sffamily{DRFA-GA}})}
	\label{alg:5}
	\begin{algorithmic}[1]
\REQUIRE $N$ clients , synchronization gap $\tau$, total number of iterations $T$, $S=T/\tau$, learning rates $\eta$, $\gamma$, sampling size $m$, initial model $\bar{\boldsymbol{w}}^{(0)}$ and initial $\boldsymbol{\lambda}^{(0)}$.
		\ENSURE Final solutions $\hat{\boldsymbol{w}} = \frac{2}{mT} \sum_{t=T/2}^{T}\sum_{i\in{\mathcal{D}^{(\floor{\frac{t}{\tau}})}}} \boldsymbol{w}^{(t)}_i$,  $\hat{\boldsymbol{\lambda}} = \frac{1}{S}\sum_{s=0}^{S-1} \boldsymbol{\lambda}^{(s)}$, or (2) $\bm{w}^T$, $\boldsymbol{\lambda}^S$. 
		 \FOR{  $s = 0$ to $S-1$ } 
        {\STATE {Server \textbf{samples}  $\mathcal{D}^{(s)} \subset [N]$ according to  $\boldsymbol{\lambda}^{(s)}$ with size of $m$}\\
        \STATE Server \textbf{broadcasts} $\bar{\boldsymbol{w}}^{(s)}$ to all clients  $i \in 
        \mathcal{D}^{(s)}$}\\[5pt]
        \algemph{antiquewhite}{1}{
        \FOR{clients $i \in \mathcal{D}^{(s)}$ \textbf{parallel}}
        \STATE Client \textbf{sets} $\boldsymbol{w}_i^{(s\tau)} = \bar{\bm{w}}^{(s)}$
        \FOR{$t = s\tau,\ldots,(s+1)\tau-1 $}
        \STATE {$\bm{w}^{(t+1)}_i =  \prod_{\mathcal{W}}\left(\bm{w}^{(t)}_i - \eta \nabla f_i(\bm{w}^{(t)}_i;\xi^{(t)}_i)\right)  $}\\[-8pt]
        \ENDFOR
        \ENDFOR
        \STATE {Client $i \in \mathcal{D}^{(s)}$ \textbf{sends} $\bm{w}^{((s+1)\tau)}_i$ back to the server}}
        \algemph{blizzardblue}{1} {
        \STATE Server sends $\bar{\bm{w}}^{(s)}$ to all clients \hfill{\mycommfont{{// Update $\boldsymbol{\lambda}$}}}
        \STATE Each client $i \in [N]$ evaluates $\bar{\bm{w}}^{(s)}$ on its local data and sends  $f_i(\bar{\bm{w}}^{(s)})$ back to server
         \STATE {Server updates  $\boldsymbol{\lambda}^{(s+1)}= \prod_{\Lambda}\left(\boldsymbol{\lambda}^{(s)}+\gamma \nabla_{\boldsymbol{\lambda}} F\left(\bar{\bm{w}}^{(s)},\boldsymbol{\lambda}^{(s)}\right) \right)$}}
        \STATE {Server \textbf{computes} $\bar{\bm{w}}^{(s+1)} = \frac{1}{m} \sum_{i\in \mathcal{D}^{(s)}}  \bm{w}^{((s+1)\tau)}_i$ }\\
    \ENDFOR
	\end{algorithmic}  
\end{algorithm}

\noindent\textbf{Nonconvex (PL condition)-Strongly-Concave Setting.} We provide the convergence analysis under the condition where $F$ is nonconvex but satisfies PL condition in $\bm{w}$, and strongly concave in $\boldsymbol{\lambda}$. In the constraint problem, to prove the convergence, we have to consider a generalization of PL condition~\cite{karimi2016linear} as  formally stated below.

\begin{definition}
[($\mu$,$\eta$)-generalized Polyak-Łojasiewicz (PL)] \label{assumption: PL} The global objective function $F(\cdot, \boldsymbol{\lambda})$ is differentiable and satisfies the ($\mu$,$\eta$)-generalized Polyak-Łojasiewicz condition with constant $\mu$ if the following holds: $$\frac{1}{2\eta^2}\left\|\bm{w}-\prod_{\mathcal{W}}\left(\bm{w}-\eta \nabla_{\boldsymbol{w}} F(\boldsymbol{w},\boldsymbol{\lambda})\right) \right\|_2^2 \geq \mu(F(\boldsymbol{w},\boldsymbol{\lambda}) -  \min_{\bm{w}'\in\mathcal{W}}F(\boldsymbol{w}',\boldsymbol{\lambda})), \forall \boldsymbol{\lambda} \in \Lambda$$. 
\end{definition}
\begin{remark}
When the constraint is absent, it reduces to vanilla PL condition~\cite{karimi2016linear}. The similar generalization of PL condition is also mentioned in~\cite{karimi2016linear}, where they introduce a variant of PL condition to prove the convergence of proximal gradient method.  Also we will show that, if $F$ satisfies $\mu$-PL condition in $\bm{w}$, $\Phi(\bm{w})$ also satisfies $\mu$-PL condition.
\end{remark}
 We now proceed to provide the global convergence of $\Phi$ in this setting.
\begin{theorem}
\label{thm3}
Let global function $F$ satisfy ($\mu$,$\eta$)-generalized PL condition in $\bm{w}$ and $\mu$-strongly-concave in $\boldsymbol{\lambda}$. Under Assumptions~\ref{assumption: smoothness},\ref{assumption: bounded gradient},\ref{assumption: bounded domain},\ref{assumption: bounded variance},  if we optimize (\ref{regularized loss}) using the {\sffamily{DRFA-GA}} (Algorithm~\ref{alg:5}) with synchronization gap $\tau$, choosing  learning rates $\eta = \frac{ 4\log T}{\mu T}$, $\gamma = \frac{1}{ L}$ and $m \geq T$, with the total iterations satisfying $T\geq \frac{8\alpha \log T}{\mu}$ where $\alpha = L+\kappa L$, $\kappa = \frac{L}{\mu}$, we have:
\begin{align*}
     \mathbb{E}\left[\Phi(\bm{w}^{(T)}) - \Phi(\bm{w}^{*})\right]  \nonumber
     & \leq O\left(\frac{\Phi(\bm{w} ^{(0)}) - \Phi( \bm{w} ^*)}{T}\right)+ 
    \Tilde{O}\left(\frac{\sigma_w^2+G_w^2}{\mu T}\right)  +    \Tilde{O}\left(\frac{\kappa^2 L \tau D_\Lambda^2}{T}\right)   \nonumber\\
    & \quad +  \Tilde{O}\left(\frac{\kappa^6\tau^2  G_w^2}{\mu T^2}\right)+  \Tilde{O}\left(\frac{\kappa^2\tau^2 (\sigma_w^2+\Gamma)}{\mu T^2}\right)\nonumber.
\end{align*}
where $\bm{w}^* \in \arg\min_{\bm{w} \in \mathcal{W}}\Phi(\bm{w})$.
\end{theorem}
\begin{proof}
The proof is given in Section~\ref{sec:ncsc}.
\end{proof}
\begin{corollary}
Continuing with Theorem~\ref{thm3}, if we choose $\tau =  \sqrt{T/m}$,  we recover the rate:
\begin{align*}
  \mathbb{E}[\Phi(\hat{\boldsymbol{w}})-\Phi(\boldsymbol{w}^*)]  = \Tilde{O}\left( \frac{\kappa^2L  D_{\Lambda}^2 }{ \sqrt{T}} +\frac{\Phi(\bm{w}^{(0)}) - \Phi(\bm{w}^{*})}{T} +\frac{\kappa^2 (\sigma_w^2 + \Gamma)+\kappa^6G_w^2}{\mu T } \right).
\end{align*}
\end{corollary}
We obtain $\tilde{O}\left(\frac{ \tau}{T}\right)$ convergence rate here, slightly worse than that of strongly-convex-strongly-concave case. We also get linear speedup in the number of sampled workers if properly choose $\tau$. The best known result of non-distributionally robust version of FedAvg on PL condition is $O(\frac{1}{T})$ \cite{haddadpour2019convergence}, with $O(T^{1/3})$ communication rounds. It turns out that we trade some convergence rate to guarantee a worst case performance. We would like to mention that, here we require $m$, the number of sampled clients to be a large number, which is the imperfection of our analysis. However, we would note that, this  is similar  to the  analysis in~\cite{ghadimi2016mini} for projected SGD on constrained nonconvex minimization problems, where it is required to employ growing  mini-batch sizes  with iterations to guarantee convergence to a first-order stationary point (i.e., imposing a constraint on minibatch size based on target accuracy $\epsilon$ which plays a similar rule to $m$ in our case).

\section{Proof of Convergence  of {\sffamily{DRFA}} for Convex Losses (Theorem~\ref{theorem1})} \label{sec: proof DRFA convex}
In this section we will present the proof of Theorem~\ref{theorem1}, which states the convergence of {\sffamily{DRFA}} in convex-linear setting. 
\subsection{Preliminary}
Before delving into the proof, let us introduce some useful variables and lemmas for ease of analysis. We define a virtual sequence $\{ \bm{w}^{(t)}\}_{t=1}^T$ that will be used in our proof, and we also define some intermediate variables:
\begin{equation}
\begin{aligned}
       \bm{w}^{(t)} &= \frac{1}{m}\sum_{i\in{\mathcal{D}^{(\floor{\frac{t}{\tau}})}}} \bm{w}^{(t)}_i, && \text{(average model of selected devices)} 
       \\ \Bar{\bm{u}}^{(t)} &= \frac{1}{m}\sum_{i\in \mathcal{D}^{(\floor{\frac{t}{\tau}})}}  \nabla f_i(\bm{w} ^{(t)}_i), && \text{(average full gradient of selected devices)} \\
        \bm{u}^{(t)} &= \frac{1}{m}\sum_{i\in{\mathcal{D}^{(\floor{\frac{t}{\tau}})}}}\nabla f_i( \bm{w} ^{(t)}_i;\xi^{(t)}_i) && \text{(average stochastic gradient of selected devices)} \label{vs} \\
        \Bar{\bm{v}}^{(t)} &= \nabla_{\boldsymbol{\lambda}} F(\bm{w} ^{(t)},\boldsymbol{\lambda}) = \left[f_1(\bm{w}^{(t)}), \ldots,f_N(\bm{w} ^{(t)})\right] && \text{(full gradient w.r.t. dual)}\\
         \bar{\Delta}_{s} &= \sum_{t=s\tau+1}^{(s+1)\tau} \gamma \Bar{\bm{v}}^{(t)}, \\
         \Delta_{s} &= \tau \gamma \bm{v}, && \text{(see below)}\\
         \delta^{(t)} &= \frac{1}{m} \sum_{i\in\mathcal{D}^{(\floor{\frac{t}{\tau}})}} \left\|\bm{w}^{(t)}_i - \bm{w}^{(t)}\right\|^2, \nonumber
\end{aligned}
\end{equation}
where $\bm{v} \in \mathbb{R}^N$ is the stochastic gradient for dual variable generated by Algorithm~\ref{alg:1} for updating $\boldsymbol{\lambda}$, such that ${v}_i = f_i(\bm{w}^{(t')};\xi_i)$ for $i \in \mathcal{U} \subset [N]$ where $\xi_i$ is stochastic minibatch sampled from $i$th local data shard,  and $t'$ is the snapshot index  sampled from $s\tau+1$ to $(s+1)\tau$.

\subsection{Overview of the Proof}
The proof techniques consist of  analyzing the one-step progress for the virtual iterates $\bm{w} ^{(t+1)}$ and  $\boldsymbol{\lambda}^{(s+1)}$, however periodic decoupled updating along with sampling makes the analysis more involved compared to fully synchronous primal-dual schemes for minimax optimization. Let us start from analyzing one iteration on $\bm{w}$. From the updating rule we can show that
\begin{align}
     \mathbb{E}\|\bm{w} ^{(t+1)} - \bm{w} \|^2 &\leq \mathbb{E}\|\bm{w} ^{(t)}- \bm{w} \|^2 -2\eta\mathbb{E}\left[F(\bm{w} ^{(t)},\boldsymbol{\lambda}^{(\floor{\frac{t}{\tau}})})  - F(\bm{w} ,\boldsymbol{\lambda}^{(\floor{\frac{t}{\tau}})})\right] \nonumber\\
     & \quad +  L\eta \mathbb{E}\left[\delta^{(t)}\right] + \eta^2\mathbb{E}\|\Bar{\bm{u}}^{(t)} - \bm{u}^{(t)}\|^2  + \eta^2G_{w}^2.\nonumber
\end{align}
Note that, similar to analysis of local SGD, e.g., ~\cite{stich2018local}, the key question is how to bound the deviation $\delta^{(t)}$ between local and (virtual) averaged model. By the definition of gradient dissimilarity, we establish that: 
\begin{align}
      \frac{1}{T}\sum_{t=0}^{T}\mathbb{E}\left[\delta^{(t)}\right]  =  10\eta^2\tau^2  \left(\sigma_w^2+\frac{\sigma_w^2}{m} + \Gamma \right).\nonumber
\end{align}
It turns out the deviation can be upper bounded by variance of stochastic graident, and the gradient dissimilarity. The latter term controls how heterogenous the local component functions are, and it becomes zero when all local functions are identical, which means we are doing minibatch SGD on the same objective function in parallel. 

Now we switch to the one iteration analysis on $\boldsymbol{\lambda}$:
\begin{align}
 \mathbb{E} \|\boldsymbol{\lambda}^{(s+1)} - \boldsymbol{\lambda} \|^2 &\leq \mathbb{E}\|\boldsymbol{\lambda}^{(s)}- \boldsymbol{\lambda} \|^2 \nonumber\\ &\quad -\sum_{t=s\tau+1}^{(s+1)\tau } \mathbb{E}[2\gamma(F(\bm{w} ^{(t)},\boldsymbol{\lambda}^{(s)})-F(\bm{w} ^{(t)},\boldsymbol{\lambda}))] + \mathbb{E}\|\bar{\Delta}_{s}\|^2 +  \mathbb{E}\|\Delta_{s} - \bar{\Delta}_{s}\|^2.\nonumber
\end{align}
It suffices to bound the variance of $\Delta_s$. Using the identity of independent variables we can prove:
\begin{equation}
         \mathbb{E}[\|\Delta_{s} - \bar{\Delta}_{s}\|^2] \leq \gamma^2\tau^2\frac{\sigma_{\lambda}^2}{m}. \nonumber
\end{equation} 
It shows that the variance depends quadratically on $\tau$\footnote{This dependency is very heavy, and one open question is to see if we employ a variance reduction scheme to loosen  this dependency.}, and can achieve linear speed up with respect to the number of sampled workers. Putting all pieces together, and doing the telescoping sum will yield the result in Theorem~\ref{theorem1}.

\subsection{Proof of Technical Lemmas}
In this section we are going to present some technical lemmas that will be used in the proof of Theorem~\ref{theorem1}.
\begin{lemma}
\label{lemma: bounded variance of w}
The stochastic gradient $\bm{u}^{(t)}$ is unbiased, and its variance is bounded, which implies:
\begin{equation}
\begin{aligned}
    \mathbb{E}_{\xi_i^{(t)},\mathcal{D}^{(\floor{\frac{t}{\tau}})}}\left[\bm{u}^{(t)}\right] &=  \mathbb{E}_{\mathcal{D}^{(\floor{\frac{t}{\tau}})}}\left[\Bar{\bm{u}}^{(t)}\right] = \mathbb{E} \left[   \sum_{i=1}^N \lambda^{(\floor{\frac{t}{\tau}})}_i \nabla f_i(\bm{w}^{(t)}_i)\right] , \nonumber \\ \mathbb{E}\left[\|\bm{u}^{(t)} - \Bar{\bm{u}}^{(t)}\|^2\right] &=  \frac{\sigma^2_{w}}{m}. \nonumber
\end{aligned}
\end{equation}
\begin{proof}
The unbiasedness is due to the fact that we sample the clients according to $\boldsymbol{\lambda}^{(\floor{\frac{t}{\tau}})}$. The variance term is due to  the identity $\mathrm{Var}(\sum_{i=1}^m \bm{X}_i) = \sum_{i=1}^m \mathrm{Var}(\bm{X}_i)$.
\end{proof}
\end{lemma}

\begin{lemma}
\label{lemma: bounded variance of lambda}
The stochastic gradient at $\boldsymbol{\lambda}$ generated by Algorithm~\ref{alg:1} is unbiased, and its variance is bounded, which implies:
\begin{equation}
    \mathbb{E}[\Delta_{s}] =  \bar{\Delta}_{s}, \quad \quad \mathbb{E}[\|\Delta_{s} - \bar{\Delta}_{s}\|^2] \leq \gamma^2\tau^2\frac{\sigma_{\lambda}^2}{m}.\label{l10}
\end{equation} 
\begin{proof}
The unbiasedness is due to we sample the workers uniformly. The variance term is due to  the identity $\mathrm{Var}(\sum_{i=1}^m \bm{X}_i) = \sum_{i=1}^m \mathrm{Var}(\bm{X}_i)$.
\end{proof}
\end{lemma}

\begin{lemma}[One Iteration Primal Analysis]
\label{lemma: one iteration w}
For {\sffamily{DRFA}}, under the same conditions as in Theorem~\ref{theorem1}, for all $\bm{w} \in \mathcal{W}$, the following holds: 
\begin{align}
     \mathbb{E}\|\bm{w} ^{(t+1)} - \bm{w} \|^2 &\leq \mathbb{E}\|\bm{w} ^{(t)}- \bm{w} \|^2 -2\eta\mathbb{E}\left[F(\bm{w} ^{(t)},\boldsymbol{\lambda}^{(\floor{\frac{t}{\tau}})})  - F(\bm{w} ,\boldsymbol{\lambda}^{(\floor{\frac{t}{\tau}})})\right] \nonumber\\
     & \quad  +  L\eta \mathbb{E}\left[\delta^{(t)}\right] + \eta^2\mathbb{E}\|\Bar{\bm{u}}^{(t)} - \bm{u}^{(t)}\|^2  + \eta^2G_{w}^2.\nonumber
\end{align}
\end{lemma}
\begin{proof}
From  the updating rule we have:
\begin{align}
   \mathbb{E}\|\bm{w}^{(t+1)} - \bm{w} \|^2 &= \mathbb{E}\left\|\prod_{\mathcal{W}}(\bm{w}^{(t)}- \eta \bm{u}^{(t)}) - \bm{w} \right\|^2 \leq \mathbb{E}\|\bm{w}^{(t)} - \eta \Bar{\bm{u}}^{(t)} - \bm{w} \|^2 + \eta^2 \mathbb{E}\|\Bar{\bm{u}}^{(t)} - \bm{u}^{(t)}\|^2 \nonumber\\
   &   \leq \mathbb{E}\|\bm{w} ^{(t)}- \bm{w} ^*\|^2 +  \underbrace{\mathbb{E}[- 2\eta \langle \Bar{\bm{u}}^{(t)},\bm{w} ^{(t)}- \bm{w} ^* \rangle]}_{T_1} + \underbrace{\eta^2  \mathbb{E}\|\Bar{\bm{u}}^{(t)}\|^2}_{T_2} + \mathbb{E}\|\Bar{\bm{u}}^{(t)} - \bm{u}^{(t)}\|^2  \label{l20}
\end{align}
We are going to bound $T_1$ first:
\begin{align}
   T_1 &=  \mathbb{E}_{\mathcal{D}^{(\floor{\frac{t}{\tau}})}}\left[ \frac{1}{m} \sum_{i\in \mathcal{D}^{(\floor{\frac{t}{\tau}})}}\left[- 2\eta\left \langle\nabla f_i(\bm{w}^{(t)}_i), \bm{w} ^{(t)}- \bm{w}^{(t)}_i  \right\rangle - 2\eta \left\langle \nabla f_i(\bm{w} ^{(t)}_i), \bm{w}^{(t)}_i - \bm{w} ^* \right\rangle\right]\right]  \label{l21} \\
   &\leq \mathbb{E}_{\mathcal{D}^{(\floor{\frac{t}{\tau}})}}\left[2\eta \frac{1}{m}\sum_{i\in \mathcal{D}^{(\floor{\frac{t}{\tau}})}}  \left[f_i(\bm{w} ^{(t)}_i) - f_i(\bm{w} ^{(t)}) + \frac{L}{2}\|\bm{w} ^{(t)} - \bm{w}^{(t)}_i \|^2 +  f_i(\bm{w}) - f_i(\bm{w} ^{(t)}_i) \right]\right] \label{l22} \\
   & = -2\eta \mathbb{E}\left[\sum_{i=1}^N  \lambda^{(\floor{\frac{t}{\tau}})}_i f_i(\bm{w} ^{(t)})  -  \lambda^{(\floor{\frac{t}{\tau}})}_i f_i(\bm{w})\right]  + L\eta \mathbb{E}\left[\delta^{(t)}\right]\nonumber\\
   & = -2\eta\mathbb{E}\left[F(\bm{w} ^{(t)},\boldsymbol{\lambda}^{(\floor{\frac{t}{\tau}})})  - F(\bm{w} ,\boldsymbol{\lambda}^{(\floor{\frac{t}{\tau}})})\right] + L\eta \mathbb{E}\left[\delta^{(t)}\right],\nonumber
\end{align}
where from~(\ref{l21})  to~(\ref{l22}) we use the smoothness and convexity properties. 

We then turn to bounding  $T_2$ as follows:
\begin{align}
   T_2 = \eta^2  \mathbb{E}\left\|\frac{1}{m}\sum_{i\in\mathcal{D}^{(\floor{\frac{t}{\tau}})}}   \nabla  f_i(\bm{w} ^{(t)}_i)\right\|^2 \leq \eta^2 \frac{1}{m}\sum_{i\in\mathcal{D}^{(\floor{\frac{t}{\tau}})}} \mathbb{E}\left\|   \nabla  f_i(\bm{w} ^{(t)}_i)\right\|^2 \leq \eta^2 G_{w}^2.\nonumber
\end{align}
Plugging $T_1$ and $T_2$ back to~(\ref{l20}) gives:
\begin{align}
     \mathbb{E}\|\bm{w} ^{(t+1)} - \bm{w} \|^2 & \leq \mathbb{E}\|\bm{w} ^{(t)}- \bm{w} \|^2 -2\eta\mathbb{E}\left[F(\bm{w} ^{(t)},\boldsymbol{\lambda}^{(\floor{\frac{t}{\tau}})})  - F(\bm{w} ,\boldsymbol{\lambda}^{(\floor{\frac{t}{\tau}})})\right] \nonumber\\
     & \quad + L\eta \mathbb{E}\left[\delta^{(t)}\right] + \eta^2 \mathbb{E}\|\Bar{\bm{u}}^{(t)} - \bm{u}^{(t)}\|^2 + \eta^2G_{w}^2,\nonumber
\end{align}
thus concluding the proof.
\end{proof}
 
The following  lemma bounds the deviation between local models and (virtual) global average model over sampled devices over $T$ iterations.   We note that the following result is general and will be used in all variants. 
\begin{lemma}[Bounded Squared Deviation]
\label{lemma: deviation}
For {\sffamily{DRFA}}, {\sffamily{DRFA-Prox}} and {\sffamily{DRFA-GA}} algorithms,  the expected average squared norm distance of local models  $\bm{w} ^{(t)}_i, i \in \mathcal{D}^{(\floor{\frac{t}{\tau}})}$ and $\bm{w} ^{(t)}$ is bounded as follows:
\begin{align}
     \frac{1}{T}\sum_{t=0}^{T}\mathbb{E}\left[\delta^{(t)}\right]  \leq  10\eta^2\tau^2  \left(\sigma_w^2+\frac{\sigma_w^2}{m} + \Gamma \right).\nonumber
\end{align}
where expectation is taken over sampling of devices at each iteration.
\end{lemma}
\begin{proof}
Consider  $s\tau \leq t\leq (s+1)\tau$. Recall that, we only perform the averaging based on a uniformly sampled subset of workers $\mathcal{D}^{(\floor{\frac{t}{\tau}})}$ of $[N]$. Following the updating rule we have:
{\begin{align}
      \mathbb{E}[\delta^{(t)}] &= \mathbb{E}\left[\frac{1}{m}\sum_{i\in\mathcal{D}^{(\floor{\frac{t}{\tau}})}}\|\bm{w} ^{(t)}_i- \bm{w}^{(t)}\|^2\right] \nonumber\\
     &  \leq  \mathbb{E}\left[\frac{1}{m}\sum_{i\in\mathcal{D}^{(\floor{\frac{t}{\tau}})}}\mathbb{E}\left\|\bm{w}^{(s\tau)}-   \sum_{r=s\tau}^{t -1} \eta   \nabla f_i(\bm{w} ^{(r)}_i;\xi_i^{(r)})  -\left(\bm{w}^{(s\tau)}- \frac{1}{m}\sum_{i'\in\mathcal{D}}\sum_{r=s\tau}^{t -1}\eta   \nabla f_{i'}(\bm{w} ^{(r)}_{i'};\xi_{i'}^{(r)}) \right) \right\|^2\right]\nonumber\\
    & =  \mathbb{E}\left[\frac{1}{m}\sum_{i\in\mathcal{D}^{(\floor{\frac{t}{\tau}})}} \left\| \sum_{r=s\tau}^{t-1}\eta  \nabla f_i(\bm{w} ^{(r)}_i;\xi_i^{(r)})  - \frac{1}{m}\sum_{i'\in\mathcal{D}^{(\floor{\frac{t}{\tau}})}}\sum_{r=s\tau}^{t-1}\eta   \nabla f_{i'}(\bm{w} ^{(r)}_{i'};\xi_{i'}^{(r)})\right\|^2 \right]\nonumber\\
    & \leq \mathbb{E}\left[\frac{1}{m}\sum_{i\in\mathcal{D}^{(\floor{\frac{t}{\tau}})}}\eta^2\tau\sum_{r=s\tau}^{(s+1)\tau} \left\|  \nabla f_i(\bm{w} ^{(r)}_i;\xi_i^{(r)})  - \frac{1}{m}\sum_{i'\in\mathcal{D}^{(\floor{\frac{t}{\tau}})}} \nabla f_{i'}(\bm{w} ^{(r)}_{i'};\xi_{i'}^{(r)})\right\|^2\right]\nonumber \\
     & =  \eta^2\tau\mathbb{E}\left[\frac{1}{m}\sum_{i\in\mathcal{D}^{(\floor{\frac{t}{\tau}})}}\sum_{r=s\tau}^{(s+1)\tau} \left\|  \nabla f_i(\bm{w} ^{(r)}_i;\xi_i^{(r)}) -\nabla f_i(\bm{w}^{(r)}_i )+ \nabla f_i(\bm{w}^{(r)}_i )-\nabla f_i(\bm{w}^{(r)} ) \vphantom{-\frac{1}{m}\sum_{i'\in\mathcal{D}^{(\floor{\frac{t}{\tau}})}} \nabla f_{i'}(\bm{w}^{(r)}_{i'})} \right.\right. \nonumber\\
     & \qquad \qquad\left.  \left. +\nabla f_i(\bm{w}^{(r)})-\frac{1}{m}\sum_{i'\in\mathcal{D}^{(\floor{\frac{t}{\tau}})}} \nabla f_{i'}(\bm{w}^{(r)}) 
      + \frac{1}{m}\sum_{i'\in\mathcal{D}^{(\floor{\frac{t}{\tau}})}} \nabla f_{i'}(\bm{w}^{(r)}) \right.\right. \nonumber\\
      &\qquad\qquad \left.\left.-\frac{1}{m}\sum_{i'\in\mathcal{D}^{(\floor{\frac{t}{\tau}})}} \nabla f_{i'}(\bm{w}^{(r)}_{i'})+\frac{1}{m}\sum_{i'\in\mathcal{D}^{(\floor{\frac{t}{\tau}})}} \nabla f_{i'}(\bm{w}^{(r)}_{i'})- \frac{1}{m}\sum_{i'\in\mathcal{D}^{(\floor{\frac{t}{\tau}})}} \nabla f_{i'}(\bm{w} ^{(r)}_{i'};\xi_{i'}^{(r)})\right\|^2\right] \label{eq: deviation 1}
      \end{align}}
Applying Jensen's inequality to split the norm yields:
 {\begin{align}
      \mathbb{E}[\delta^{(t)}]  &\leq  5\eta^2\tau\sum_{r=s\tau}^{(s+1)\tau}  \left(\sigma_w^2 + L^2\mathbb{E}\left[\frac{1}{m}\sum_{i\in\mathcal{D}^{(\floor{\frac{t}{\tau}})}}\left\|\bm{w}^{(r)}_i - \bm{w}^{(r)}  \right\|^2\right]+ L^2\mathbb{E}\left[\frac{1}{m}\sum_{i'\in\mathcal{D}^{(\floor{\frac{t}{\tau}})}}\left\|\bm{w}^{(r)}_{i'} - \bm{w}^{(r)}  \right\|^2\right] \right. \nonumber\\ 
      & \qquad \qquad \left. +\mathbb{E}\left[\frac{1}{m}\sum_{i'\in\mathcal{D}^{(\floor{\frac{t}{\tau}})}}\left\|\nabla f_{i}(\bm{w}^{(r)}) - \nabla f_{i'}(\bm{w} ^{(r)}) \right\|^2\right]+\frac{\sigma_w^2}{m} \right) \label{eq: deviation 2}\\ 
      & \leq  5\eta^2\tau  \sum_{r=s\tau}^{(s+1)\tau}\left(\sigma^2_w + 2L^2\mathbb{E}[\delta^{(r)}]  + \Gamma +\frac{\sigma^2_w}{m} \right) \label{eq: deviation 3},
\end{align}}
where from (\ref{eq: deviation 1}) to (\ref{eq: deviation 2}) we use the Jensen's inequality.

Now we sum (\ref{eq: deviation 3}) over $t = s\tau$ to $(s+1)\tau$ to get:
\begin{align}
      \sum_{t=s\tau}^{(s+1)\tau}\mathbb{E}[\delta^{(t)}] & \leq  5\eta^2\tau   \sum_{t=s\tau}^{(s+1)\tau}\sum_{r=s\tau}^{(s+1)\tau}\left(\sigma^2_w + 2L^2\mathbb{E}[\delta^{(r)}]  + \Gamma +\frac{\sigma^2_w}{m} \right) \nonumber\\
     & = 5\eta^2\tau^2 \sum_{r=s\tau}^{(s+1)\tau}\left(\sigma_w^2 + 2\mathbb{E}[\delta^{(r)}]  + \Gamma +\frac{\sigma_w^2}{m} \right).\nonumber
\end{align}
Re-arranging the terms and using the fact $1-10\eta^2\tau^2 L^2 \geq \frac{1}{2}$ yields:
\begin{align}
       \sum_{t=s\tau}^{(s+1)\tau}\mathbb{E}[\delta^{(t)}] \leq 10\eta^2\tau^2 \sum_{r=s\tau}^{(s+1)\tau}\left(\sigma_w^2    + \Gamma +\frac{\sigma_w^2}{m} \right).\nonumber
\end{align}
Summing over communication steps $s=0$ to $S-1$, and dividing both sides by $T=S\tau$ yields:
\begin{align}
   \frac{1}{T}\sum_{t=0}^{T} \mathbb{E}[\delta^{(t)}]  \leq  10\eta^2\tau^2  \left(\sigma_w^2+\frac{\sigma_w^2}{m} + \Gamma \right),\nonumber
\end{align}
as desired.
\end{proof}

\begin{lemma}[Bounded Norm Deviation]
\label{lemma: deviation2}
For {\sffamily{DRFA}}, {\sffamily{DRFA-Prox}} and {\sffamily{DRFA-GA}}, $\forall i \in \mathcal{D}^{(\floor{\frac{t}{\tau}})}$, the norm distance between $\bm{w} ^{(t)}$ and $\bm{w} ^{(t)}_i$ is bounded as follows:
\begin{align}
     \frac{1}{T}\sum_{t=0}^{T}\mathbb{E}\left[\frac{1}{m}\sum_{i\in\mathcal{D}^{(\floor{\frac{t}{\tau}})}}\left\|\bm{w} ^{(t)}_i- \bm{w}^{(t)}\right\|\right]  \leq  2\eta \tau \left(\sigma_w +\frac{\sigma_w }{m} + \sqrt{\Gamma} \right).\nonumber
\end{align}
\end{lemma}
\begin{proof}
Similar to what we did in Lemma~\ref{lemma: deviation}, we assume $s\tau \leq t\leq (s+1)\tau$. Again,  we only apply the averaging based on a uniformly sampled subset of workers $\mathcal{D}^{(\floor{\frac{t}{\tau}})}$ of $[N]$. From  the updating rule we have:
{\begin{align}
     &  \mathbb{E}\left[\frac{1}{m}\sum_{i\in\mathcal{D}^{(\floor{\frac{t}{\tau}})}}\|\bm{w} ^{(t)}_i- \bm{w}^{(t)}\|\right] \nonumber\\
     & =  \mathbb{E}\left[\frac{1}{m}\sum_{i\in\mathcal{D}^{(\floor{\frac{t}{\tau}})}} \left\|\bm{w}^{(s\tau)}-   \sum_{r=s\tau}^{t-1} \eta   \nabla f_i(\bm{w} ^{(r)}_i;\xi_i^{(r)})-\left(\bm{w}^{(s\tau)}- \frac{1}{m}\sum_{i'\in\mathcal{D}}\sum_{r=s\tau}^{t-1}\eta   \nabla f_{i'}(\bm{w} ^{(r)}_{i'};\xi_{i'}^{(r)}) \right) \right\|\right]\nonumber \\
     & = \mathbb{E}\left[\frac{1}{m}\sum_{i\in\mathcal{D}^{(\floor{\frac{t}{\tau}})}}\mathbb{E}\left\| \sum_{r=s\tau}^{t-1}\eta  \nabla f_i(\bm{w} ^{(r)}_i;\xi_i^{(r)})  - \frac{1}{m}\sum_{i'\in\mathcal{D}^{(\floor{\frac{t}{\tau}})}}\sum_{r=s\tau}^{t-1}\eta   \nabla f_{i'}(\bm{w} ^{(r)}_{i'};\xi_{i'}^{(r)})\right\|\right] \nonumber \\
     & \leq \mathbb{E}\left[ \frac{1}{m}\sum_{i\in\mathcal{D}^{(\floor{\frac{t}{\tau}})}}\eta \sum_{r=s\tau}^{(s+1)\tau}\mathbb{E}\left\|  \nabla f_i(\bm{w} ^{(r)}_i;\xi_i^{(r)})  - \frac{1}{m}\sum_{i'\in\mathcal{D}^{(\floor{\frac{t}{\tau}})}} \nabla f_{i'}(\bm{w} ^{(r)}_{i'};\xi_{i'}^{(r)})\right\|\right]\nonumber \\
     & =  \eta \mathbb{E}\left[ \frac{1}{m}\sum_{i\in\mathcal{D}^{(\floor{\frac{t}{\tau}})}}\sum_{r=s\tau}^{(s+1)\tau} \left\|  \nabla f_i(\bm{w} ^{(r)}_i;\xi_i^{(r)}) -\nabla f_i(\bm{w}^{(r)}_i )+ \nabla f_i(\bm{w}^{(r)}_i )-\nabla f_i(\bm{w}^{(r)} ) +\nabla f_i(\bm{w}^{(r)}) \vphantom{\sum_{i'\in\mathcal{D}^{(\floor{\frac{t}{\tau}})}}} \right.\right.\nonumber\\
     &  \quad \left. \left. -\frac{1}{m}\sum_{i'\in\mathcal{D}^{(\floor{\frac{t}{\tau}})}} \nabla f_{i'}(\bm{w}^{(r)}) + \frac{1}{m}\sum_{i'\in\mathcal{D}^{(\floor{\frac{t}{\tau}})}} \nabla f_{i'}(\bm{w}^{(r)}) -\frac{1}{m}\sum_{i'\in\mathcal{D}^{(\floor{\frac{t}{\tau}})}} \nabla f_{i'}(\bm{w}^{(r)}_{i'}) \right.\right. \nonumber\\
     & \quad \left.\left. +\frac{1}{m}\sum_{i'\in\mathcal{D}^{(\floor{\frac{t}{\tau}})}} \nabla f_{i'}(\bm{w}^{(r)}_{i'})- \frac{1}{m}\sum_{i'\in\mathcal{D}^{(\floor{\frac{t}{\tau}})}} \nabla f_{i'}(\bm{w} ^{(r)}_{i'};\xi_{i'}^{(r)})\right\|\right] \nonumber
     \end{align}} 
Applying the triangular inequality to split the norm yields:
{ \begin{align} 
       \mathbb{E}\left[\frac{1}{m} \vphantom{\sum_{i\in\mathcal{D}^{(\floor{\frac{t}{\tau}})}}}\right.&\left.\sum_{i\in\mathcal{D}^{(\floor{\frac{t}{\tau}})}}\|\bm{w} ^{(t)}_i- \bm{w}^{(t)}\|\right] \nonumber\\
      & \leq   \eta \mathbb{E}\left[ \frac{1}{m}\sum_{i\in\mathcal{D}^{(\floor{\frac{t}{\tau}})}} \sum_{r=s\tau}^{(s+1)\tau}\left(\sigma_w + L \left\|\bm{w}^{(r)}_i - \bm{w}^{(r)}  \right\|\right.\right.\nonumber\\
      &\quad \left. \left. + \frac{1}{m}\sum_{i'\in\mathcal{D}^{(\floor{\frac{t}{\tau}})}}L \left\|\bm{w}^{(r)}_{i'} - \bm{w}^{(r)}  \right\| +\frac{1}{m}\sum_{i'\in\mathcal{D}^{(\floor{\frac{t}{\tau}})}} \left\|\nabla f_{i}(\bm{w} ^{(r)}) - \nabla f_{i'}(\bm{w}^{(r)}) \right\|+\frac{\sigma_w}{m} \right)\right] \nonumber\\ 
      & =  \eta  \sum_{r=s\tau}^{(s+1)\tau}\left(\sigma_w + 2L\mathbb{E}\left[\frac{1}{m}\sum_{i'\in\mathcal{D}^{(r)}}\mathbb{E}\|\bm{w} ^{(r)}_{i'}- \bm{w}^{(r)}\|\right]  + \sqrt{\Gamma} +\frac{\sigma_w}{m} \right). \label{eq: deviation2 3}
\end{align}} 

Now summing (\ref{eq: deviation2 3}) over $t = s\tau$ to $(s+1)\tau$ gives:
\begin{align}
      &\sum_{t=s\tau}^{(s+1)\tau}\mathbb{E}\left[\frac{1}{m}\right. \left.\sum_{i\in\mathcal{D}^{(\floor{\frac{t}{\tau}})}}\|\bm{w} ^{(t)}_{i}- \bm{w}^{(t)}\|\right]  \nonumber\\
      & \leq \eta   \sum_{t=s\tau}^{(s+1)\tau}\sum_{r=s\tau}^{(s+1)\tau}\left(\sigma_w + 2L\mathbb{E}\left[\frac{1}{m}\sum_{i'\in\mathcal{D}^{(r)}}\|\bm{w} ^{(r)}_{i'}- \bm{w}^{(r)}\|\right]  + \sqrt{\Gamma} +\frac{\sigma_w}{m} \right) \nonumber\\
     & =  \eta\tau \sum_{r=s\tau}^{(s+1)\tau}\left(\sigma_w + 2L\mathbb{E}\left[\frac{1}{m}\sum_{i'\in\mathcal{D}^{(r)}}\|\bm{w} ^{(r)}_{i'}- \bm{w}^{(r)}\|\right] + \sqrt{\Gamma} +\frac{\sigma_w}{m} \right).\nonumber
\end{align}
Re-arranging the terms and using the fact $1-2\eta \tau  L \geq \frac{1}{2}$ yields:
\begin{align}
       \sum_{t=s\tau}^{(s+1)\tau}\mathbb{E}\left[\frac{1}{m}\sum_{i\in\mathcal{D}^{(\floor{\frac{t}{\tau}})}}\|\bm{w} ^{(t)}_{i}- \bm{w}^{(t)}\|\right]  \leq 2 \eta \tau  \sum_{r=s\tau}^{(s+1)\tau}\left(\sigma_w     + \sqrt{\Gamma} +\frac{\sigma_w }{m} \right).\nonumber
\end{align}
Summing over $s=0$ to $S-1$, and dividing both sides by $T=S\tau$ yields:
\begin{align}
   \frac{1}{T}\sum_{t=0}^{T}\mathbb{E}\left[\frac{1}{m}\sum_{i\in\mathcal{D}^{(\floor{\frac{t}{\tau}})}}\|\bm{w} ^{(t)}_{i}- \bm{w}^{(t)}\|\right]  \leq   2\eta \tau   \left(\sigma_w +\frac{\sigma_w }{m} + \sqrt{\Gamma} \right),\nonumber
\end{align}
which concludes the proof.
\end{proof}

\begin{lemma}[One Iteration Dual Analysis]
\label{lemma: one iteration lambda}
For {\sffamily{DRFA}}, under the assumption of Theorem~\ref{theorem1}, the following holds true for any $\boldsymbol{\lambda} \in \Lambda$:
\begin{align}
 \mathbb{E} \|\boldsymbol{\lambda}^{(s+1)} - \boldsymbol{\lambda} \|^2 &\leq \mathbb{E}\|\boldsymbol{\lambda}^{(s)}- \boldsymbol{\lambda} \|^2\nonumber\\ &\quad -\sum_{t=s\tau+1}^{(s+1)\tau} \mathbb{E}[2\gamma(F(\bm{w} ^{(t)},\boldsymbol{\lambda}^{(\floor{\frac{t}{\tau}})})-F(\bm{w} ^{(t)},\boldsymbol{\lambda}))] + \mathbb{E}\|\bar{\Delta}_{t}\|^2 +  \mathbb{E}\|\Delta_{t} - \bar{\Delta}_{t}\|^2. \nonumber
\end{align}
\end{lemma}
\begin{proof}
According to the updating rule for $\boldsymbol{\lambda}$ and the fact $F$ is linear in $\boldsymbol{\lambda}$ we have:
\begin{equation}
    \begin{aligned}
     \mathbb{E}\left\|\boldsymbol{\lambda}^{(s+1)} - \boldsymbol{\lambda}\right\|^2&= \mathbb{E}\left\|\prod_{\Lambda}(\boldsymbol{\lambda}^{(s)}+ \Delta_{s})- \boldsymbol{\lambda} \right\|^2\nonumber\\
     &\leq \mathbb{E}\left\|\boldsymbol{\lambda}^{(s)}- \boldsymbol{\lambda} + \Delta_{s}\right\|^2\nonumber\\
     & = \mathbb{E}\left\|\boldsymbol{\lambda}^{(s)}- \boldsymbol{\lambda} + \bar{\Delta}_{s}\right\|^2 + \mathbb{E}\left\|\Delta_{s} - \bar{\Delta}_{s}\right\|^2\nonumber\\
      &  = \mathbb{E}\|\boldsymbol{\lambda}^{(s)}- \boldsymbol{\lambda} \|^2 + \mathbb{E}\left[2\left\langle \bar{\Delta}_{s},  \boldsymbol{\lambda}^{(s)}- \boldsymbol{\lambda}\right\rangle\right] + \mathbb{E}\|\bar{\Delta}_{s}\|^2 +\mathbb{E} \|\Delta_{s} - \bar{\Delta}_{s}\|^2\nonumber\\
      &   = \mathbb{E}\|\boldsymbol{\lambda}^{(s)}- \boldsymbol{\lambda} \|^2 \nonumber\\
      &   \quad + 2\gamma \sum_{t= s\tau+1}^{(s+1)\tau}\mathbb{E}\left[\left\langle \nabla_{\boldsymbol{\lambda}} F(\bm{w}^{(t)},\boldsymbol{\lambda}^{(s)}),  \boldsymbol{\lambda}^{(s)}- \boldsymbol{\lambda}\right\rangle\right] + \mathbb{E}\|\bar{\Delta}_{s}\|^2 + \mathbb{E}\|\Delta_{s} - \bar{\Delta}_{s}\|^2\nonumber\\
      &=\|\boldsymbol{\lambda}^{(s)}- \boldsymbol{\lambda}\|^2 \nonumber \\ 
      &   \quad -2\gamma \sum_{t= s\tau+1}^{(s+1)\tau} \mathbb{E}\left[F(\bm{w} ^{(t)},\boldsymbol{\lambda})-F(\bm{w} ^{(t)},\boldsymbol{\lambda}^{(s)}))\right] + \mathbb{E}\|\bar{\Delta}_{s}\|^2 +  \mathbb{E}\|\Delta_{s} - \bar{\Delta}_{s}\|^2 \label{l41},
\end{aligned}
\end{equation}
 as desired.

\end{proof}

\subsection{Proof for Theorem~\ref{theorem1}}
\begin{proof} 
Equipped with above results, we are now turn to proving the Theorem~\ref{theorem1}. We start by noting that $\forall \bm{w} \in \mathcal{W}$, $\forall \boldsymbol{\lambda} \in \Lambda$,  according the convexity of global objective w.r.t. $\bm{w} $ and its linearity in terms of $\boldsymbol{\lambda}$ we have:
\begin{align}
     &\mathbb{E}[F(\hat{\bm{w} },\boldsymbol{\lambda} ) - \mathbb{E}[F(\bm{w}  ,\hat{\boldsymbol{\lambda}})]\nonumber\\
     & \leq \frac{1}{T}\sum_{t=1}^T \left\{   \mathbb{E}\left[F( \bm{w}^{(t)},\boldsymbol{\lambda} )\right] - \mathbb{E}\left[F(\bm{w}  ,\boldsymbol{\lambda}^{(\floor{\frac{t}{\tau}})})\right] \right\}\nonumber\\
     &  \leq \frac{1}{T}\sum_{t=1}^T\left \{   \mathbb{E}\left[F( \bm{w}^{(t)},\boldsymbol{\lambda} )\right] -\mathbb{E}\left[F( \bm{w}^{(t)},\boldsymbol{\lambda}^{(\floor{\frac{t}{\tau}})})\right] +\mathbb{E}\left[ F( \bm{w}^{(t)},\boldsymbol{\lambda}^{(\floor{\frac{t}{\tau}})})\right] - \mathbb{E}\left[F(\bm{w},\boldsymbol{\lambda}^{(\floor{\frac{t}{\tau}})})\right] \right\}\nonumber\\
     &  \leq \frac{1}{T}\sum_{s=0}^{S-1} \sum_{t=s\tau+1}^{(s+1)\tau}\mathbb{E}\{F( \bm{w}^{(t)},\boldsymbol{\lambda} ) -F( \bm{w}^{(t)},\boldsymbol{\lambda}^{(s)})\} \label{eq: thm1 1}\\ 
     & \quad +\frac{1}{T}\sum_{t=1}^T\mathbb{E}\{ F( \bm{w}^{(t)},\boldsymbol{\lambda}^{(\floor{\frac{t}{\tau}})}) - F(\bm{w}  ,\boldsymbol{\lambda}^{(\floor{\frac{t}{\tau}})}) \}  \label{eq: thm1 2},
\end{align}
To bound the term in (\ref{eq: thm1 1}), pluggin Lemma~\ref{lemma: bounded variance of lambda}  into Lemma~\ref{lemma: one iteration lambda}, we have:
\begin{align}
    \frac{1}{T}\sum_{s=0}^{S-1} \sum_{t=s\tau+1}^{(s+1)\tau} \mathbb{E}(F(\bm{w} ^{(t)},\boldsymbol{\lambda})-F(\bm{w} ^{(t)},\boldsymbol{\lambda}^{(\floor{\frac{t}{\tau}})}))  &\leq  \frac{1}{2\gamma T}\|\boldsymbol{\lambda}^{(0)}- \boldsymbol{\lambda}\|^2 + \frac{\gamma\tau }{2}G_{\lambda}^2+ \frac{\gamma\tau\sigma_{ \lambda}^2 }{2m} \nonumber\\
    & \leq \frac{D_{\Lambda}^2}{2\gamma T}  + \frac{\gamma \tau G_{\boldsymbol{\lambda}}^2}{2}  + \frac{\gamma\tau\sigma_{ \lambda}^2 }{2m}.\nonumber
\end{align}
To bound the term in (\ref{eq: thm1 2}), we plug Lemma~\ref{lemma: bounded variance of w} into Lemma~\ref{lemma: one iteration w} and apply the telescoping sum from $t = 1$ to $T$ to get:
\begin{align}
    &\frac{1}{T}\sum_{t=1}^T \mathbb{E}(F(\bm{w} ^{(t)},\boldsymbol{\lambda}^{(\floor{\frac{t}{\tau}})})-F(\bm{w}  ,\boldsymbol{\lambda}^{(\floor{\frac{t}{\tau}})}))\nonumber\\
    & \leq \frac{1}{2T\eta}\mathbb{E}\|\ \bm{w}^{(0)}- \bm{w} \|^2 +   5L\eta^2\tau^2  \left(\sigma_w^2+\frac{\sigma_w^2}{m} + \Gamma \right) + \frac{\eta G_w^2}{2} + \frac{\eta \sigma_{w }^2 }{2m}\nonumber\\
    & \leq \frac{D_{\mathcal{W}}^2}{2T\eta}  +   5L\eta^2\tau^2  \left(\sigma_w^2+\frac{\sigma_w^2}{m} + \Gamma \right) + \frac{\eta G_w^2}{2} + \frac{\eta \sigma_{w }^2 }{2m}.\nonumber
\end{align}

Putting pieces together, and taking max over dual $\boldsymbol{\lambda}$, min over primal $\bm{w}$ yields:
\begin{align}
     &\min_{\bm{w}\in \mathcal{W}}\max_{\boldsymbol{\lambda}\in \Lambda} \mathbb{E}[F(\hat{\bm{w} },\boldsymbol{\lambda} ) - \mathbb{E}[F(\bm{w}  ,\hat{\boldsymbol{\lambda}})]\nonumber\\
     & \leq \frac{D_{\mathcal{W}}^2}{2T\eta}  +   5L\eta^2\tau^2  \left(\sigma_{w}^2+\frac{\sigma_{w}^2}{m} + \Gamma \right) + \frac{\eta G_w^2}{2} + \frac{\eta \sigma_{w }^2 }{2m} + \frac{D_{\Lambda}^2}{2\gamma T}  + \frac{\gamma \tau G_{\boldsymbol{\lambda}}^2}{2}  + \frac{\gamma\tau\sigma_{ \lambda}^2 }{2m}.\nonumber
\end{align} 

Plugging in $\tau =  \frac{T^{1/4}}{\sqrt{m}}$, $\eta = \frac{1}{4L \sqrt{T}}$, and $\gamma = \frac{1}{T^{5/8}}$, we  conclude the proof by getting:
\begin{equation}
\begin{aligned}
    \max_{\boldsymbol{\lambda}\in \Lambda}\mathbb{E}[F(\hat{\boldsymbol{w}},\boldsymbol{\lambda} )] -\min_{\bm{w}\in\mathcal{W}} \mathbb{E}[F(\boldsymbol{w} ,\hat{\boldsymbol{\lambda}} )] \leq O\Big{(}&\frac{D_{\mathcal{W}}^2+G_{w}^2}{\sqrt{T}} +\frac{D_{\Lambda}^2}{T^{3/8}} \nonumber. \\&+\frac{G_{\lambda}^2}{m^{1/2}T^{3/8}} +\frac{\sigma_{\lambda}^2}{m^{3/2}T^{3/8}}+ \frac{\sigma_w^2+\Gamma}{m\sqrt{T} }\Big{)},
\end{aligned}
\end{equation}
as desired.
\end{proof}

\section{Proof of Convergence  of {\sffamily{DRFA}} for Nonconvex Losses (Theorem~\ref{thm:nonconvex-linear})}\label{sec: proof DRFA nonconvex}

This section is devoted to the proof of Theorem~\ref{thm:nonconvex-linear}).
\subsection{Overview of Proofs}
Inspired by the techniques in~\cite{lin2019gradient} for analyzing the behavior of stochastic gradient descent ascent (SGDA) algorithm on nonconvex-concave objectives, we consider the Moreau Envelope of $\Phi$:
\begin{align}
    \Phi_{p} (\bm{x}) := \min_{\bm{w}\in \mathcal{W}} \left\{ \Phi  (\bm{w}) + \frac{1}{2p}\|\bm{w}-\bm{x}\|\right\}.\nonumber
\end{align}

We first examine the one iteration dynamic of {\sffamily{DRFA}}:
\begin{align}
     \mathbb{E}[\Phi_{1/2L} (\bm{w}^{(t)})] 
     & \leq \mathbb{E}[\Phi_{1/2L} (\bm{w}^{(t-1)})] + 2\eta D_{\mathcal{W}}L^2 \mathbb{E}\left[ \frac{1}{m}\sum_{i\in\mathcal{D}^{(\floor{\frac{t-1}{\tau}})}} \left\|   \bm{w}^{(t-1)}_i  -  \bm{w}^{(t-1)}  \right\| \right] \nonumber \\
    & 2\eta L \left( \mathbb{E}[\Phi(\bm{w}^{(t-1)})]- \mathbb{E}[F(\bm{w}^{(t-1)}, \boldsymbol{\lambda}^{\floor{\frac{t-1}{\tau}}})]\right) - \frac{\eta}{4}\mathbb{E}\left[\left\|\nabla \Phi_{1/2L} (\bm{w}^{(t-1)})\right\|^2\right].\nonumber
\end{align}
We already know how to bound $\mathbb{E}\left[ \frac{1}{m}\sum_{i\in\mathcal{D}^{(\floor{\frac{t-1}{\tau}})}} \left\|   \bm{w}^{(t-1)}_i  -  \bm{w}^{(t-1)}  \right\| \right]$ in Lemma~\ref{lemma: deviation2}. Then the key  is to bound $ \mathbb{E}[\Phi(\bm{w}^{(t-1)})]- \mathbb{E}[F(\bm{w}^{(t-1)}, \boldsymbol{\lambda}^{(\floor{\frac{t-1}{\tau}})})]$. Indeed this term characterizes how far the current dual variable drifts from the optimal dual variable $\boldsymbol{\lambda}^*(\bm{w}^{(t-1)})$. Then by examining the dynamic of dual variable we have $\forall \boldsymbol{\lambda} \in \Lambda$:

{\begin{align}
      &\sum_{t=(s-1)\tau+1}^{s\tau}\left( \mathbb{E}\left[\Phi(\bm{w}^{(t)} )\right]- \right.\left. \mathbb{E}\left[F(\bm{w}^{(t)}, \boldsymbol{\lambda}^{(s-1)} )\right]\right) \nonumber\\
    &\leq  \sum_{t=(s-1)\tau+1}^{s\tau}\left(\mathbb{E}\left[ F(\bm{w}^{(t)}, \boldsymbol{\lambda}^*(\bm{w}^{t}) )\right]- \mathbb{E}\left[ F(\bm{w}^{(t)}, \boldsymbol{\lambda} )\right] \right) \nonumber \\  &\quad+    \gamma \tau^2\frac{\sigma_{\lambda}^2}{m}+   \gamma \tau^2 G_\lambda^2+  \frac{1}{2\gamma}\left(\mathbb{E}\left[\left\|\boldsymbol{\lambda}-\boldsymbol{\lambda}^{(s-1)}\right\|^2 \right] - \mathbb{E}\left[\left\|\boldsymbol{\lambda}-\boldsymbol{\lambda}^{(s)}\right\|^2 \right]\right).\nonumber
\end{align}}
The above inequality  makes it possible to replace $\boldsymbol{\lambda}$ with $\boldsymbol{\lambda}^*$, and doing the telescoping sum so that the last term cancels up. However, in the minimax problem, the optimal dual variable changes every time when we update primal variable. Thus, we divide $S$ global stages into $\sqrt{S}$ groups, and applying the telescoping sum within one group, by setting $\boldsymbol{\lambda}=\boldsymbol{\lambda}^*(\bm{w}^{c\sqrt{S}\tau})$ at $c$th stage.

\subsection{Proof of Useful Lemmas}
Before presenting the proof of Theorem~\ref{thm:nonconvex-linear}, let us introduce the following useful lemmas.
\begin{lemma}[One iteration analysis]\label{lm: nonconvex lm1}
For {\sffamily{DRFA}}, under the assumptions of Theorem~\ref{thm:nonconvex-linear}, the following statement holds:
\begin{align}
     \mathbb{E}[\Phi_{1/2L} & (\bm{w}^{(t)})]   \leq \mathbb{E}[\Phi_{1/2L} (\bm{w}^{(t-1)})] + 2\eta D_{\mathcal{W}}L^2 \mathbb{E}\left[ \frac{1}{m}\sum_{i\in\mathcal{D}^{(\floor{\frac{t-1}{\tau}})}} \left\|   \bm{w}^{(t-1)}_i  -  \bm{w}^{(t-1)}  \right\| \right] \nonumber \\
    &+ 2\eta L \left( \mathbb{E}[\Phi(\bm{w}^{(t-1)})]- \mathbb{E}[F(\bm{w}^{(t-1)}, \boldsymbol{\lambda}^{(\floor{\frac{t-1}{\tau})}})]\right) - \frac{\eta}{4}\mathbb{E}\left[\left\|\nabla \Phi_{1/2L} (\bm{w}^{(t-1)})\right\|^2\right].\nonumber
\end{align}
\begin{proof}
Define $\tilde{\bm{w}}^{(t)} = \min_{\bm{w}\in\mathcal{W}} \Phi(\bm{w}) + L\|\bm{w}-\bm{w}^{(t)}\|^2$, the by the definition of $\Phi_{1/2L}$ we have:
\begin{align}
    \Phi_{1/2L} (\bm{w}^{(t)}) \leq \Phi (\tilde{\bm{w}}^{(t-1)}) + L\|\tilde{\bm{w}}^{(t-1)}-\bm{w}^{(t)}\|^2.\label{eq: nonconvex lm1 1}
\end{align}
Meanwhile according to updating rule we have:
\begin{align}
    &\mathbb{E}\left[\left\| \tilde{\bm{w}}^{(t-1)}  -\bm{w}^{(t)}\right\|^2\right] \nonumber\\ 
    &=  \mathbb{E}\left[\left\|\tilde{\bm{w}}^{(t-1)}-\prod_{\mathcal{W}}\left(\bm{w}^{(t-1)} - \frac{1}{m}\sum_{i\in\mathcal{D}^{(\floor{\frac{t-1}{\tau}})}} \nabla_x f_i(\bm{w}^{(t-1)}_i;\xi_i^{(t-1)})\right)\right\|^2\right]\nonumber\\
    &\leq \mathbb{E}\left[\left\|\tilde{\bm{w}}^{(t-1)}-\bm{w}^{(t-1)} \right\|^2\right] + \mathbb{E}\left[\left\|\frac{1}{m}\sum_{i\in\mathcal{D}^{(\floor{\frac{t-1}{\tau}})}} \nabla_x f_i(\bm{w}^{(t-1)}_i;\xi_i^{(t-1)}) \right\|^2\right]  \nonumber\\
    & \quad+ 2\eta\mathbb{E}\left[\left\langle  \tilde{\bm{w}}^{(t-1)}-\bm{w}^{(t-1)}, \frac{1}{m}\sum_{i\in\mathcal{D}^{(\floor{\frac{t-1}{\tau}})}} \nabla_x f_i(\bm{w}^{(t-1)}_i )\right\rangle\right]\nonumber.
    \end{align}
    Applying Cauchy inequality to the last inner product term yields:
    \begin{align}
    &\mathbb{E}\left[\left\| \tilde{\bm{w}}^{(t-1)}  -\bm{w}^{(t)}\right\|^2\right] \nonumber\\
    & \leq \mathbb{E}\left[\left\|\tilde{\bm{w}}^{(t-1)}-\bm{w}^{(t-1)} \right\|^2\right] + \eta^2(G_w^2+\sigma_w^2) + 2\eta\left\langle  \tilde{\bm{w}}^{(t-1)}-\bm{w}^{(t-1)}, \frac{1}{m}\sum_{i\in\mathcal{D}^{(\floor{\frac{t-1}{\tau}})}} \nabla_x f_i(\bm{w}^{(t-1)}  )\right\rangle \nonumber\\
    &\quad+ \eta \mathbb{E}\left[\left\|\tilde{\bm{w}}^{(t-1)}-\bm{w}^{(t-1)} \right\| \right]\mathbb{E}\left[ \frac{1}{m}\sum_{i\in\mathcal{D}^{(\floor{\frac{t-1}{\tau}})}} \left\|\nabla_x f_i(\bm{w}^{(t-1)}_i) - \nabla_x f_i(\bm{w}^{(t-1)} )\right\| \right] \nonumber\\
    & \leq \mathbb{E}\left[\left\|\tilde{\bm{w}}^{(t-1)}-\bm{w}^{(t-1)} \right\|^2\right] + \eta^2(G_w^2+\sigma_w^2) +  \eta D_{\mathcal{W}}L\mathbb{E}\left[ \frac{1}{m}\sum_{i\in\mathcal{D}^{(\floor{\frac{t-1}{\tau}})}} \left\|  \bm{w}^{(t-1)}_i  -  \bm{w}^{(t-1)}  \right\| \right] \nonumber\\
    &\quad + 2\eta\mathbb{E}\left[\left\langle  \tilde{\bm{w}}^{(t-1)}-\bm{w}^{(t-1)},  \nabla_x F(\bm{w}^{(t-1)},\boldsymbol{\lambda}^{\floor{\frac{t-1}{\tau}}}  )\right\rangle\right].\label{eq: nonconvex lm1 2}
\end{align}

According to smoothness of $F$ we obtain:
\begin{align}
     &\mathbb{E}\left[\left\langle  \tilde{\bm{w}}^{(t-1)}-\bm{w}^{(t-1)},  \nabla_x F(\bm{w}^{(t-1)},\boldsymbol{\lambda}^{\floor{\frac{t-1}{\tau}}}  )\right\rangle\right] \nonumber\\
     &\leq \mathbb{E}\left[F(\tilde{\bm{w}}^{(t-1)},\boldsymbol{\lambda}^{\floor{\frac{t-1}{\tau}}})\right] - \mathbb{E}\left[F( \bm{w}^{(t-1)},\boldsymbol{\lambda}^{\floor{\frac{t-1}{\tau}}})\right] + \frac{L}{2}\mathbb{E}\left[\left\|\tilde{\bm{w}}^{(t-1)}-\bm{w}^{(t-1)} \right\|^2\right] \nonumber\\
     & \leq \mathbb{E}\left[\Phi(\tilde{\bm{w}}^{(t-1)} )\right]- \mathbb{E}\left[F( \bm{w}^{(t-1)},\boldsymbol{\lambda}^{\floor{\frac{t-1}{\tau}}})\right]+ \frac{L}{2}\mathbb{E}\left[\left\|\tilde{\bm{w}}^{(t-1)}-\bm{w}^{(t-1)} \right\|^2\right]  \nonumber\\
       & \leq \underbrace{ \mathbb{E}\left[\Phi(\tilde{\bm{w}}^{(t-1)} )\right]+ L\mathbb{E}\left[\left\|\tilde{\bm{w}}^{(t-1)}-\bm{w}^{(t-1)} \right\|^2\right]}_{\leq \mathbb{E}\left[\Phi( \bm{w}^{(t-1)} )\right]+ L\mathbb{E}\left[\left\|\bm{w}^{(t-1)}-\bm{w}^{(t-1)} \right\|^2\right]} - \mathbb{E}\left[F( \bm{w}^{(t-1)},\boldsymbol{\lambda}^{\floor{\frac{t-1}{\tau}}})\right] - \frac{L}{2}\mathbb{E}\left[\left\|\tilde{\bm{w}}^{(t-1)}-\bm{w}^{(t-1)} \right\|^2\right] \nonumber\\
     & \leq \mathbb{E}\left[\Phi( \bm{w}^{(t-1)} )\right] - \mathbb{E}\left[F( \bm{w}^{(t-1)},\boldsymbol{\lambda}^{\floor{\frac{t-1}{\tau}}})\right] - \frac{L}{2}\mathbb{E}\left[\left\|\tilde{\bm{w}}^{(t-1)}-\bm{w}^{(t-1)} \right\|^2\right].\label{eq: nonconvex lm1 3}
\end{align}
Plugging (\ref{eq: nonconvex lm1 2}) and (\ref{eq: nonconvex lm1 3}) into (\ref{eq: nonconvex lm1 1}) yields:
\begin{align}
    \Phi_{1/2L} (\bm{w}^{(t)}) &\leq \Phi (\tilde{\bm{w}}^{(t-1)}) + L  \mathbb{E}\left[\left\|\tilde{\bm{w}}^{(t-1)}-\bm{w}^{(t-1)} \right\|^2\right]  \nonumber\\
    & \quad +L \eta^2(G_w^2+\sigma_w^2) +  \eta  D_{\mathcal{W}}L^2 \mathbb{E}\left[ \frac{1}{m}\sum_{i\in\mathcal{D}^{(\floor{\frac{t-1}{\tau}})}} \left\|  \bm{w}^{(t-1)}_i  -  \bm{w}^{(t-1)}  \right\| \right] \nonumber \\
    & \quad + 2L\eta\left( \mathbb{E}\left[\Phi( \bm{w}^{(t-1)} )\right] - \mathbb{E}\left[F( \bm{w}^{(t-1)},\boldsymbol{\lambda}^{\floor{\frac{t-1}{\tau}}})\right] - \frac{L}{2}\mathbb{E}\left[\left\|\tilde{\bm{w}}^{(t-1)}-\bm{w}^{(t-1)} \right\|^2\right]\right)\nonumber\\
    & \leq \Phi_{1/2L} (\bm{w}^{(t-1)}) + L  \mathbb{E}\left[\left\|\tilde{\bm{w}}^{(t-1)}-\bm{w}^{(t-1)} \right\|^2\right] \nonumber\\
    & \quad +L \eta^2(G_w^2+\sigma_w^2) +  \eta  D_{\mathcal{W}}L^2 \mathbb{E}\left[ \frac{1}{m}\sum_{i\in\mathcal{D}^{(\floor{\frac{t-1}{\tau}})}} \left\|   \bm{w}^{(t-1)}_i  -  \bm{w}^{(t-1)}  \right\| \right] \nonumber\\
    & \quad+ 2L\eta\left( \mathbb{E}\left[\Phi( \bm{w}^{(t-1)} )\right] - \mathbb{E}\left[F( \bm{w}^{(t-1)},\boldsymbol{\lambda}^{\floor{\frac{t-1}{\tau}}})\right]  \right) - \frac{\eta}{4} \mathbb{E}\left[\left\|\nabla \Phi_{1/2L} (\bm{w}^{(t-1)}) \right\|^2\right],\nonumber
\end{align}
where we use the result from Lemma 2.2 in~\cite{davis2019stochastic}, i.e,  $\nabla \Phi_{1/2L}(\bm{w}) = 2L (\bm{w}-\Tilde{\bm{w}})$.
\end{proof}
\end{lemma}

\begin{lemma}\label{lm: nonconvex lm2}
For {\sffamily{DRFA}}, $\forall \bm{\lambda} \in \Lambda$, under the same conditions as in Theorem~\ref{thm:nonconvex-linear}, the following statement holds true:
{\begin{align}
       \sum_{t=(s-1)\tau+1}^{s\tau}\left( \mathbb{E}\left[\Phi(\bm{w}^{(t)} )\right] \right. & \left. -  \mathbb{E}\left[F(\bm{w}^{(t)}, \boldsymbol{\lambda}^{(s-1)} )\right]\right) \nonumber\\
    &\leq  \sum_{t=(s-1)\tau+1}^{s\tau}\left(\mathbb{E}\left[ F(\bm{w}^{(t)}, \boldsymbol{\lambda}^*(\bm{w}^{t}) )\right]- \mathbb{E}\left[ F(\bm{w}^{(t)}, \boldsymbol{\lambda} )\right] \right) \nonumber \\
    &\quad +    \gamma \tau^2\frac{\sigma_{\lambda}^2}{m}  +\gamma \tau^2{G_{\lambda}^2}+  \frac{1}{2\gamma}\left(\mathbb{E}\left[\left\|\boldsymbol{\lambda}-\boldsymbol{\lambda}^{(s-1)}\right\|^2 \right] - \mathbb{E}\left[\left\|\boldsymbol{\lambda}-\boldsymbol{\lambda}^{(s)}\right\|^2 \right]\right)\nonumber.
\end{align}}
\begin{proof}
$\forall \boldsymbol{\lambda} \in \Lambda$, according to updating rule for $\boldsymbol{\lambda}^{(s-1)}$, we have:
\begin{align}
    \left\langle \boldsymbol{\lambda}-\boldsymbol{\lambda}^{(s)},   \boldsymbol{\lambda}^{(s)}-\boldsymbol{\lambda}^{(s-1)}- \Delta_{s-1} \right \rangle \geq 0.\nonumber
\end{align}
Taking expectation on both sides, and doing some algebraic manipulation yields:
\begin{align}
    \mathbb{E}\left[\left\| \vphantom{\lambda^{(s)}}\right.\right. & \left.\left. \boldsymbol{\lambda}-\boldsymbol{\lambda}^{(s)}\right\|^2 \right]  \nonumber \\ &\leq 2 \mathbb{E}\left[\left\langle  \boldsymbol{\lambda}^{(s-1)}-\boldsymbol{\lambda},  \Delta_{s-1}  \right \rangle\right] + 2  \mathbb{E}\left[\left\langle  \boldsymbol{\lambda}^{(s)}-\boldsymbol{\lambda}^{(s-1)},  \Delta_{s-1}  \right \rangle\right] \nonumber\\
    & \quad +\mathbb{E}\left[\left\|\boldsymbol{\lambda}-\boldsymbol{\lambda}^{(s-1)}\right\|^2 \right] -  \mathbb{E}\left[\left\|\boldsymbol{\lambda}^{(s)}-\boldsymbol{\lambda}^{(s-1)}\right\|^2 \right]  \nonumber\\
    & \leq 2  \mathbb{E}\left[\left\langle  \boldsymbol{\lambda}^{(s-1)}-\boldsymbol{\lambda}, \bar{\Delta}_{s-1}  \right \rangle\right] + 2  \mathbb{E}\left[\left\langle  \boldsymbol{\lambda}^{(s)}-\boldsymbol{\lambda}^{(s-1)}, \bar{\Delta}_{s-1}  \right \rangle\right]\nonumber\\
    & \quad+ 2  \mathbb{E}\left[\left\langle  \boldsymbol{\lambda}^{(s)}-\boldsymbol{\lambda}^{(s-1)},  \Delta_{s-1} - \bar{\Delta}_{s-1}  \right \rangle\right]+  \mathbb{E}\left[\left\|\boldsymbol{\lambda}-\boldsymbol{\lambda}^{(s-1)}\right\|^2 \right] -  \mathbb{E}\left[\left\|\boldsymbol{\lambda}^{(s)}-\boldsymbol{\lambda}^{(s-1)}\right\|^2 \right]\nonumber. 
\end{align}
Applying the Cauchy-Schwartz and aritmetic
mean-geometric mean inequality: $2\langle \bm{p},\bm{q} \rangle \leq 2\|\bm{p}\| \|\bm{q}\|\leq \frac{1}{2}\|\bm{p}\|^2+2 \|\bm{q}\|^2$, we have:
\begin{align}
      \mathbb{E}\left[\left\| \vphantom{\lambda^{(s)}}\right.\right. & \left.\left. \boldsymbol{\lambda}-\boldsymbol{\lambda}^{(s)}\right\|^2 \right]  \nonumber \\ 
      & \leq 2 \gamma \mathbb{E}\left[ \sum_{t=(s-1)\tau+1}^{s\tau}F(\bm{w}^{(t)}, \boldsymbol{\lambda}^{(s-1)} )- F(\bm{w}^{(t)}, \boldsymbol{\lambda} )\right] +  \mathbb{E}\left[\left\|\boldsymbol{\lambda}-\boldsymbol{\lambda}^{(s-1)}\right\|^2 \right] \nonumber\\
    & \quad+   \mathbb{E}\left[\frac{1}{2}\left\| \boldsymbol{\lambda}^{(s)}-\boldsymbol{\lambda}^{(s-1)}\right \|^2 + 2\left\|\Delta_{s-1} - \bar{\Delta}_{s-1}  \right \|^2\right]+   \mathbb{E}\left[\frac{1}{2}\left\| \boldsymbol{\lambda}^{(s)}-\boldsymbol{\lambda}^{(s-1)}\right \|^2 + 2\left\| \bar{\Delta}_{s-1}  \right \|^2\right] \nonumber \\ & \quad -  \mathbb{E}\left[\left\|\boldsymbol{\lambda}^{(s)}-\boldsymbol{\lambda}^{(s-1)}\right\|^2 \right] \nonumber\\
    & \leq 2 \gamma \mathbb{E}\left[ \sum_{t=(s-1)\tau+1}^{s\tau}F(\bm{w}^{(t)}, \boldsymbol{\lambda}^{(s-1)} )- F(\bm{w}^{(t)}, \boldsymbol{\lambda} )\right]  +    \gamma^2\tau^2\frac{\sigma_{\lambda}^2}{m}+  \gamma^2\tau^2{G_{\lambda}^2} + \mathbb{E}\left[\left\|\boldsymbol{\lambda}-\boldsymbol{\lambda}^{(s-1)}\right\|^2 \right].\nonumber
\end{align}
By adding $\sum_{t=(s-1)\tau+1}^{s\tau}F(\bm{w}^{(t)}, \boldsymbol{\lambda}^*(\bm{w}^{(t)}))$ on both sides and re-arranging the terms we have:
{\begin{align}
   &\sum_{t=(s-1)\tau+1}^{s\tau}\left( \mathbb{E}\left[\Phi(\bm{w}^{(t)} )\right]  -  \mathbb{E}\left[F(\bm{w}^{(t)}, \boldsymbol{\lambda}^{(s-1)} )\right]\right) \nonumber\\
    &\leq  \sum_{t=(s-1)\tau+1}^{s\tau}\left(\mathbb{E}\left[ F(\bm{w}^{(t)}, \boldsymbol{\lambda}^*(\bm{w}^{t}) )\right]- \mathbb{E}\left[ F(\bm{w}^{(t)}, \boldsymbol{\lambda} )\right] \right) +    \gamma \tau^2\frac{\sigma_{\lambda}^2}{m}+    \gamma \tau^2{G_{\lambda}^2} \nonumber\\
    & \quad +  \frac{1}{2\gamma}\left(\mathbb{E}\left[\left\|\boldsymbol{\lambda}-\boldsymbol{\lambda}^{(s-1)}\right\|^2 \right] - \mathbb{E}\left[\left\|\boldsymbol{\lambda}-\boldsymbol{\lambda}^{(s)}\right\|^2 \right]\right). \nonumber
\end{align}}
\end{proof}
\end{lemma}

\begin{lemma}\label{lm: nonconvex lm3}
For {\sffamily{DRFA}}, under the assumptions in Theorem~\ref{thm:nonconvex-linear}, the following statement holds true:
 \begin{align}
   \frac{1}{T} \sum_{t=1}^{T}\left( \mathbb{E}\left[\Phi(\bm{w}^{(t)} )\right]-  \mathbb{E}\left[F(\bm{w}^{(t)}, \boldsymbol{\lambda}^{(\floor{\frac{t}{\tau}})} )\right]\right)   \leq   2\sqrt{S} \tau  \eta G_w \sqrt{G_w^2+\sigma_w^2}+ \gamma \tau\frac{\sigma_{\lambda}^2}{m}+ \gamma \tau {G_{\lambda}^2}  +  \frac{ D_{\Lambda}^2}{2\sqrt{S}\tau\gamma} \nonumber
\end{align} 
\begin{proof}
Without loss of generality we assume $\sqrt{S}$ is an integer, so we can equally divide index $0$ to $S-1$ into $\sqrt{S}$ groups. Then we have:
{\begin{align}
      \frac{1}{T} \sum_{t=1}^{T}&\left( \mathbb{E}\left[\Phi(\bm{w}^{(t)} )\right]-  \mathbb{E}\left[F(\bm{w}^{(t)}, \boldsymbol{\lambda}^{(\floor{\frac{t}{\tau}})} )\right]\right) \nonumber\\
      & = \frac{1}{T} \sum_{c=0}^{\sqrt{S}-1}\left[ \sum_{s=c\sqrt{S}+1}^{(c+1)\sqrt{S}}\sum_{t=(s-1)\tau+1}^{s\tau}  \left( \mathbb{E}\left[\Phi(\bm{w}^{(t)} )\right]-  \mathbb{E}\left[F(\bm{w}^{(t)}, \boldsymbol{\lambda}^{(s-1)} )\right]\right)\right]. \label{eq: nonconvex lm3 0}
\end{align}}
Now we and examine one group. Plugging in Lemma~\ref{lm: nonconvex lm2} and letting $\boldsymbol{\lambda} = \boldsymbol{\lambda}^*(\bm{w}^{(c+1)\sqrt{S}\tau})$ yields: 
\begin{align}
     &\sum_{s=c\sqrt{S}+1}^{(c+1)\sqrt{S}}\sum_{t=(s-1)\tau+1}^{s\tau}  \left( \mathbb{E}\left[\Phi(\bm{w}^{(t)} )\right]-  \mathbb{E}\left[F(\bm{w}^{(t)}, \boldsymbol{\lambda}^{(s-1)} )\right]\right) \nonumber\\
     &\quad\leq \sum_{s=c\sqrt{S}+1}^{(c+1)\sqrt{S}} \sum_{t=(s-1)\tau+1}^{s\tau}\left(\mathbb{E}\left[ F(\bm{w}^{(t)}, \boldsymbol{\lambda}^*(\bm{w}^{t}) )\right]- \mathbb{E}\left[ F(\bm{w}^{(t)}, \boldsymbol{\lambda}^*(\bm{w}^{(c+1)\sqrt{S}\tau}) )\right] \right) \nonumber\\
     &\qquad +    \gamma \tau^2\frac{\sqrt{S}\sigma_{\lambda}^2}{m}+  \gamma \tau^2{\sqrt{S}G_{\lambda}^2}+ \frac{1}{2\gamma}\sum_{s=c\sqrt{S}+1}^{(c+1)\sqrt{S}}\left(\mathbb{E}\left[\left\|\boldsymbol{\lambda}^*(\bm{w}^{(c+1)\sqrt{S}\tau})-\boldsymbol{\lambda}^{(s-1)}\right\|^2 \right] - \mathbb{E}\left[\left\|\boldsymbol{\lambda}^*(\bm{w}^{(c+1)\sqrt{S}\tau})-\boldsymbol{\lambda}^{(s)}\right\|^2 \right]\right)  \nonumber\\
     & \quad \leq \sum_{s=c\sqrt{S}+1}^{(c+1)\sqrt{S}} \sum_{t=(s-1)\tau+1}^{s\tau} \left(\mathbb{E}\left[ F(\bm{w}^{(t)}, \boldsymbol{\lambda}^*(\bm{w}^{t}) )\right]-\mathbb{E}\left[ F(\bm{w}^{((c+1)\sqrt{S}\tau)}, \boldsymbol{\lambda}^*(\bm{w}^{t}) )\right] \right. \nonumber\\
     &\qquad \left. +\mathbb{E}\left[ F(\bm{w}^{((c+1)\sqrt{S}\tau)}, \boldsymbol{\lambda}^*(\bm{w}^{(c+1)\sqrt{S}\tau}) )\right]- \mathbb{E}\left[ F(\bm{w}^{(t)},\boldsymbol{\lambda}^*(\bm{w}^{(c+1)\sqrt{S}\tau}))\right] \right) \nonumber\\
     &\qquad  +    \gamma \tau^2\frac{\sqrt{S}\sigma_{\lambda}^2}{m}+  \gamma \tau^2{\sqrt{S}G_{\lambda}^2}+  \frac{1}{2\gamma}\sum_{s=c\sqrt{S}+1}^{(c+1)\sqrt{S}}\left(\mathbb{E}\left[\left\|\boldsymbol{\lambda}^*(\bm{w}^{(c+1)\sqrt{S}\tau})-\boldsymbol{\lambda}^{(s-1)}\right\|^2 \right] - \mathbb{E}\left[\left\|\boldsymbol{\lambda}^*(\bm{w}^{(c+1)\sqrt{S}\tau})-\boldsymbol{\lambda}^{(s)}\right\|^2 \right]\right) \label{eq: nonconvex lm3 1}\\
     &\quad\leq \sum_{s=c\sqrt{S}+1}^{(c+1)\sqrt{S}} \sum_{t=(s-1)\tau+1}^{s\tau}(2\sqrt{S}\tau \eta G_w \sqrt{G_w^2+\sigma_w^2})+    \gamma \tau\frac{\sqrt{S}\sigma_{\lambda}^2}{m}+  \gamma \tau{\sqrt{S}G_{\lambda}^2} +  \frac{D_{\Lambda}^2}{2\gamma}\label{eq: nonconvex lm3 2}\\
     &\quad \leq  2S\tau^2 \eta G_w \sqrt{G_w^2+\sigma_w^2}+    \gamma \tau^2\frac{\sqrt{S}\sigma_{\lambda}^2}{m}+  \gamma \tau^2{\sqrt{S}G_{\lambda}^2} +  \frac{D_{\Lambda}^2}{2\gamma},\label{eq: nonconvex lm3 3}
\end{align}
where from (\ref{eq: nonconvex lm3 1}) to (\ref{eq: nonconvex lm3 2}) we use the $G_w$-Lipschitz property of $F(\cdot,\boldsymbol{\lambda})$ so that $F(\bm{w}^{t_1}, \boldsymbol{\lambda}) - F(\bm{w}^{t_2}, \boldsymbol{\lambda}) \leq G_w \|\bm{w}^{t_1} - \bm{w}^{t_2}\|$.

Now plugging (\ref{eq: nonconvex lm3 3}) back to (\ref{eq: nonconvex lm3 0}) yields:
 \begin{align}
   \frac{1}{T} \sum_{t=1}^{T}\left( \mathbb{E}\left[\Phi(\bm{w}^{(t)} )\right]-  \mathbb{E}\left[F(\bm{w}^{(t)}, \boldsymbol{\lambda}^{(s)} )\right]\right)  & \leq \frac{1}{T}  2\sqrt{S}S\tau^2 \eta G_w \sqrt{G_w^2+\sigma_w^2}+    \gamma \tau\frac{\sigma_{\lambda}^2}{m} +  \gamma \tau{ G_{\lambda}^2}+  \frac{\sqrt{S}D_{\Lambda}^2}{2T\gamma}\nonumber\\
    & \leq   2\sqrt{S} \tau  \eta G_w \sqrt{G_w^2+\sigma_w^2}+ \gamma \tau\frac{\sigma_{\lambda}^2}{m}+  \gamma \tau{ G_{\lambda}^2} +  \frac{ D_{\Lambda}^2}{2\sqrt{S}\tau\gamma}.\nonumber
\end{align} 
\end{proof}
\end{lemma}

\subsection{Proof of Theorem~\ref{thm:nonconvex-linear}}
Now we proceed to the formal proof of Theorem~\ref{thm:nonconvex-linear}. Re-arranging terms in Lemma~\ref{lm: nonconvex lm1},  summing over $t = 1$ to $T$, and dividing by $T$ yields:
\begin{align}
      \frac{1}{T}\sum_{t=1}^T \mathbb{E}\left[\left\|\nabla \Phi_{1/2L}\right.\right. & \left.\left. (\bm{w}^{(t)})\right\|^2\right] \nonumber\\
&\leq \frac{4}{\eta T}\mathbb{E}[\Phi_{1/2L} (\bm{w}^{(0)})]   +    \frac{1}{2T}\sum_{t=1}^T D_{\mathcal{W}}L^2 \mathbb{E}\left[ \frac{1}{m}\sum_{i\in\mathcal{D}^{(\floor{\frac{t}{\tau}})}} \left\|   \bm{w}^{(t)}_i  -  \bm{w}^{(t)}  \right\| \right] \nonumber \\
    & \quad +   L \frac{1}{2T}\sum_{t=1}^T\left( \mathbb{E}[\Phi(\bm{w}^{(t)})]- \mathbb{E}[F(\bm{w}^{(t)}, \boldsymbol{\lambda}^{\floor{\frac{t}{\tau}}})]\right).\nonumber
\end{align}
Plugging in Lemma~\ref{lemma: deviation2} and \ref{lm: nonconvex lm3} 
yields:
\begin{align}
      \frac{1}{T}\sum_{t=1}^T \mathbb{E}\left[\left\|\nabla \Phi_{1/2L}(\bm{w}^{(t)})\right\|^2\right] 
&\leq \frac{4}{\eta T}\mathbb{E}[\Phi_{1/2L} (\bm{w}^{(0)})]   +      \eta \tau D_{\mathcal{W}}L^2 \left(\sigma_w +\frac{\sigma_w }{m} + \sqrt{\Gamma} \right). \nonumber \\
    & \quad + \frac{L}{2 } \left(  2\sqrt{S} \tau  \eta G_w \sqrt{G_w^2+\sigma_w^2}+ \gamma \tau\frac{\sigma_{\lambda}^2}{m}+  \gamma \tau{ G_{\lambda}^2} +  \frac{ D_{\Lambda}^2}{2\sqrt{S}\tau\gamma}\right)\nonumber\\
    & \leq \frac{4}{\eta T}\mathbb{E}[\Phi_{1/2L} (\bm{w}^{(0)})]   +    \eta \tau D_{\mathcal{W}}L^2 \left(\sigma_w +\frac{\sigma_w }{m} + \sqrt{\Gamma} \right) \nonumber \\
    & \quad +    \sqrt{S} \tau  \eta G_w L\sqrt{G_w^2+\sigma_w^2}+ \gamma \tau\frac{\sigma_{\lambda}^2L}{2m}+  \gamma \tau\frac{ G_{\lambda}^2 L}{2} +  \frac{ D_{\Lambda}^2L}{4\sqrt{S}\tau\gamma}.\nonumber 
\end{align}
Plugging in $\eta = \frac{1}{4LT^{3/4}}$ , $\gamma = \frac{1}{T^{1/2}}$ and $\tau = T^{1/4}$ we recover the convergence rate as cliamed:
\begin{align}
      \frac{1}{T}\sum_{t=1}^T \mathbb{E}\left[\left\|\nabla \Phi_{1/2L} (\bm{w}^{(t)})\right\|^2\right]  
    & \leq \frac{4}{ T^{1/4}}\mathbb{E}[\Phi_{1/2L} (\bm{w}^{(0)})]   +    \frac{L^2}{T^{1/2}}\left(\sigma_w +\frac{\sigma_w }{m} + \sqrt{\Gamma} \right) \nonumber \\
    & \quad +    \frac{1}{T^{1/8}}G_w L\sqrt{G_w^2+\sigma_w^2}+  \frac{\sigma_{\lambda}^2L}{2m T^{1/4}}+  \frac{G_{\lambda}^2L}{2 T^{1/4}} +  \frac{ D_{\Lambda}^2L}{4 T^{1/8}}, \nonumber
\end{align}
which concludes the proof.
\qed

\section{Proof of Convergence  of {\sffamily{DRFA-Prox}}}
This section is devoted to the proof of convergence of {\sffamily{DRFA-Prox}} algorithm in both convex and nonconvex settings.

\subsection{Convex Setting}\label{sec: proof DRFA-Prox convex}
In this section we are going to provide the proof of Theorem~\ref{thm: regularized convex-linear}, the convergence of {\sffamily{DRFA-Prox}} on convex losses,  i.e., global objective $F$ is convex in $\bm{w}$. Let us first introduce a key lemma:
\begin{lemma}\label{lemma: DRFA-prox convex-concave one iteration}
For {\sffamily{DRFA-Prox}}, $\forall \boldsymbol{\lambda} \in \Lambda$, and for any $s$ such that $0 \leq s\leq \frac{T}{\tau}-1$ we have:
\begin{align}
    &\sum_{t=s\tau+1}^{(s+1)\tau}\left( \mathbb{E}\left[ F(\bm{w}^{(t)},\boldsymbol{\lambda} )\right] - \mathbb{E}\left[  F(\bm{w}^{(t)}, \boldsymbol{\lambda}^{(s)}  )\right] \right) \nonumber\\
    &\leq - \frac{1}{2\gamma}\mathbb{E}\|\boldsymbol{\lambda}^{(s+1)} -\boldsymbol{\lambda}\|^2 + \frac{1}{2\gamma} \mathbb{E}[\| \boldsymbol{\lambda}^{(s)}-\boldsymbol{\lambda}\|^2]+ \frac{1}{2\gamma} \mathbb{E}[\|\bar{\Delta}_s- \Delta_s\|^2]\nonumber\\
    & \quad +\tau^2 \gamma G_w(G_w + \sqrt{G_w^2 +G_\lambda^2+ \sigma_\lambda^2})+\tau^2 \gamma G_\lambda^2\nonumber
\end{align}

\end{lemma}
\begin{proof}
Recall that to update $\boldsymbol{\lambda}^{(s)}$, we sampled a index $t'$ from $s\tau+1$ to $(s+1)\tau$, and obtain the averaged model $\bm{w}^{(t')}$. Now, consider iterations from $s\tau+1$ to $(s+1)\tau$.
Define following function:
\begin{align}
    \Psi(\bm{u}) &= \tau f(\bm{w}^{(t')},\bm{y}) + \tau g(\bm{u}) - \frac{1}{2\gamma}\|\bm{y}+ \Delta_s - \bm{u}\|^2 \nonumber \\
     &=  \tau f(\bm{w}^{(t')},\bm{y})+ \tau g(\bm{u}) - \frac{1}{2\gamma}\|\bm{y}+\bar{\Delta}_s - \bm{u}\|^2 - \frac{1}{2\gamma}\| \bar{\Delta}_s - \Delta_s\|^2 \\
     & \quad + \frac{1}{\gamma}\langle \bar{\Delta}_s - \Delta_s,\bm{y} +\bar{\Delta}_s - \bm{u}  \rangle\nonumber. 
\end{align}
By taking the expectation on both side, we get:
\begin{align}
     &\mathbb{E}[\Psi(\bm{u})]\nonumber\\
     &=  \mathbb{E}[\tau f(\bm{w}^{(t')},\bm{y})] +  \frac{1}{\gamma}\mathbb{E}[\left \langle  \bar{\Delta}_s , \bm{u}-\bm{y} \right\rangle] + \mathbb{E}[\tau g(\bm{u})] - \frac{1}{2\gamma}\mathbb{E}\|\bm{u}-\bm{y}\|^2 - \frac{1}{2\gamma}\mathbb{E}\|\bar{\Delta}_s - \Delta_s\|^2 - \frac{1}{2\gamma}\mathbb{E}\|\bar{\Delta}_s  \|^2 \nonumber \\
     &=  \mathbb{E}\left[\sum_{t=s\tau+1}^{(s+1)\tau} F(\bm{w}^{(t)},\bm{u})\right]    - \frac{1}{2\gamma}\mathbb{E}\|\bm{u}-\bm{y}\|^2 - \frac{1}{2\gamma}\mathbb{E}\|\bar{\Delta}_s - \Delta_s\|^2-\frac{1}{2\gamma}\mathbb{E}\|\bar{\Delta}_s  \|^2 \nonumber 
\end{align}
where we used the fact that $\mathbb{E}[\tau f(\bm{w}^{(t')},\bm{y})] = \mathbb{E}\left[\sum_{t=s\tau+1}^{(s+1)\tau} f(\bm{w}^{(t)},\bm{y})\right]$ and  $\frac{1}{\gamma}\mathbb{E}[\left \langle  \Delta_s , \bm{u}-\bm{y} \right\rangle] =\sum_{t=s\tau+1}^{(s+1)\tau}\mathbb{E}\left[ f(\bm{w}^{(t)},\bm{u}) - f(\bm{w}^{(t)},\bm{y})\right] $.

Define the operator:
\begin{align}
     T_{g}(\boldsymbol{y} ) := \arg \max_{\bm{u}\in \Lambda} \left\{\tau g(\bm{u}) - \frac{1}{2\gamma}\|\boldsymbol{y} +\Delta_s - \bm{u} \|^2\right\} 
\end{align}
Since $\Psi(\bm{u})$ is $\frac{1}{2\gamma}$-strongly concave, and $T_g(\bm{y})$ is the maximizer of $\Psi(\bm{u})$, we have:
\begin{align}
    \mathbb{E}[\Psi(T_g(\bm{y}) )] - \mathbb{E}[\Psi(\bm{u})]  \geq \frac{1}{2\gamma}\mathbb{E}\|T_g(\bm{y})-\bm{u}\|^2 \nonumber
\end{align}
Notice that:
\begin{align}
    \mathbb{E}[\Psi(T_g(\bm{y}) )]  =  \mathbb{E}\left[\sum_{t=s\tau+1}^{(s+1)\tau} F(\bm{w}^{(t)},T_g(\bm{y}) )\right]   - \frac{1}{2\gamma} \mathbb{E}[\|T_g(\bm{y}) -\bm{y}\|^2] - \frac{1}{2\gamma} \mathbb{E}[\|\bar{\Delta}_s- \Delta_s\|^2]-\frac{1}{2\gamma}\mathbb{E}\|\bar{\Delta}_s  \|^2 \nonumber  
\end{align} 
So we know that $ \mathbb{E}\left[\sum_{t=s\tau+1}^{(s+1)\tau} F(\bm{w}^{(t)},T_g(\bm{y}) )\right]\geq  \mathbb{E}[\Psi(T_g(\bm{y}) )] $, and hence:
\begin{align}
   \mathbb{E}\left[\sum_{t=s\tau+1}^{(s+1)\tau} F(\bm{w}^{(t)},T_g(\bm{y}) )\right] - \mathbb{E}[\Psi(\bm{u})]\geq \mathbb{E}[\Psi(T_g(\bm{y}) )] - \mathbb{E}[\Psi(\bm{u})]   \geq \frac{1}{2\gamma}\mathbb{E}\|T_g(\bm{y})-\bm{u}\|^2 \nonumber
\end{align}
Plugging in $\mathbb{E}[\Psi(\bm{u})]$ results in:
{\begin{align}
   \mathbb{E}&\left[\sum_{t=s\tau+1}^{(s+1)\tau} F(\bm{w}^{(t)},T_g(\bm{y}) )\right] - \mathbb{E}[\Psi(\bm{u})]\nonumber \\
   &=  \mathbb{E}\left[\sum_{t=s\tau+1}^{(s+1)\tau} F(\bm{w}^{(t)},T_g(\bm{y}) )\right] -\left(\mathbb{E}\left[\sum_{t=s\tau+1}^{(s+1)\tau} F(\bm{w}^{(t)},\bm{u} )\right]  - \frac{1}{2\gamma} \mathbb{E}[\|\bm{u} -\bm{y}\|^2] - \frac{1}{2\gamma} \mathbb{E}[\|\bar{\Delta}_s- \Delta_s\|^2] - \frac{1}{2\gamma}\mathbb{E}\|\bar{\Delta}_s  \|^2 \right)\nonumber\\
   &\geq \frac{1}{2\gamma}\mathbb{E}\|T_g(\bm{y})-\bm{u}\|^2.\nonumber
\end{align}} 
Re-arranging the terms yields:
\begin{align}
   \mathbb{E}\left[\sum_{t=s\tau+1}^{(s+1)\tau} F(\bm{w}^{(t)},\bm{u} )\right] & - \mathbb{E}\left[\sum_{t=s\tau+1}^{(s+1)\tau} F(\bm{w}^{(t)},T_g(\bm{y}) )\right]  \nonumber\\
   &\leq - \frac{1}{2\gamma}\mathbb{E}\|T_g(\bm{y})-\bm{u}\|^2 + \frac{1}{2\gamma} \mathbb{E}[\| \bm{y}-\bm{u}\|^2]+ \frac{1}{2\gamma} \mathbb{E}[\|\bar{\Delta}_s- \Delta_s\|^2]+\frac{1}{2\gamma}\mathbb{E}\|\bar{\Delta}_s  \|^2 \label{eq: proof DRFA-Prox convex eq2} .
\end{align}
Let $\bm{u} = \boldsymbol{\lambda}$, $\bm{y} = \boldsymbol{\lambda}^{(s)}$, then we have:
\begin{align}
   &\sum_{t=s\tau+1}^{(s+1)\tau}\left( \mathbb{E}\left[ F(\bm{w}^{(t)},\boldsymbol{\lambda} )\right] - \mathbb{E}\left[  F(\bm{w}^{(t)},T_g(\boldsymbol{\lambda}^{(s)}) )\right] \right) \nonumber\\
   &\leq - \frac{1}{2\gamma}\mathbb{E}\|T_g( \boldsymbol{\lambda}^{(s)})-\boldsymbol{\lambda}\|^2 + \frac{1}{2\gamma} \mathbb{E}[\| \boldsymbol{\lambda}^{(s)}-\boldsymbol{\lambda}\|^2]+ \frac{1}{2\gamma} \mathbb{E}[\|\bar{\Delta}_s- \Delta_s\|^2]+\frac{1}{2\gamma}\mathbb{E}\|\bar{\Delta}_s  \|^2  .\nonumber
\end{align} 
Since $T_g(\boldsymbol{\lambda}^{(s)}) = \boldsymbol{\lambda}^{(s+1)}$, we have:

\begin{align}
   &\sum_{t=s\tau+1}^{(s+1)\tau}\left( \mathbb{E}\left[ F(\bm{w}^{(t)},\boldsymbol{\lambda} )\right] - \mathbb{E}\left[  F(\bm{w}^{(t)}, \boldsymbol{\lambda}^{(s)} )\right] \right) \nonumber\\
   &\leq - \frac{1}{2\gamma}\mathbb{E}\|T_g( \boldsymbol{\lambda}^{(s)})-\boldsymbol{\lambda}\|^2 + \frac{1}{2\gamma} \mathbb{E}[\| \boldsymbol{\lambda}^{(s)}-\boldsymbol{\lambda}\|^2]+ \frac{1}{2\gamma} \mathbb{E}[\|\bar{\Delta}_s- \Delta_s\|^2]+\frac{1}{2\gamma}\mathbb{E}\|\bar{\Delta}_s  \|^2 \nonumber\\
   & \quad +\underbrace{\sum_{t=s\tau+1}^{(s+1)\tau}\left( \mathbb{E}\left[ F(\bm{w}^{(t)},\boldsymbol{\lambda}^{(s+1)} )\right] - \mathbb{E}\left[  F(\bm{w}^{(t)}, \boldsymbol{\lambda}^{(s)} )\right] \right)}_{T_1}.\nonumber
\end{align} 
Now our remaining task is to bound $T_1$. By the Lipschitz property of $F$, we have the following upper bound for $T_1$:
\begin{align}
   T_1 \leq \tau G_w\mathbb{E}\|\boldsymbol{\lambda}^{(s+1)}- \boldsymbol{\lambda}^{(s)}\|\label{eq: proof DRFA-Prox convex eq3}.
\end{align} 
Then, by plugging $\bm{u} = \boldsymbol{\lambda}^{(s)}$, $\bm{y} = \boldsymbol{\lambda}^{(s)}$ into (\ref{eq: proof DRFA-Prox convex eq2}), we have the following lower bound:
\begin{align}
    \frac{1}{2\gamma}\mathbb{E}\|\boldsymbol{\lambda}^{(s+1)} - \boldsymbol{\lambda}^{(s)}\|^2 -\frac{1}{2\gamma} \mathbb{E}[\|\bar{\Delta}_s- \Delta_s\|^2] -\frac{1}{2\gamma}\mathbb{E}\|\bar{\Delta}_s  \|^2  \leq T_1 \label{eq: proof DRFA-Prox convex eq4}.
\end{align} 
Combining (\ref{eq: proof DRFA-Prox convex eq3}) and (\ref{eq: proof DRFA-Prox convex eq4}) we have:
\begin{align}
    \frac{1}{2\gamma}\mathbb{E}\|\boldsymbol{\lambda}^{(s+1)} - \boldsymbol{\lambda}^{(s)}\|^2 &-\frac{1}{2\gamma} \mathbb{E}[\|\bar{\Delta}_s- \Delta_s\|^2]-\frac{1}{2\gamma}\mathbb{E}\|\bar{\Delta}_s  \|^2  \nonumber\\
    & \leq  \tau G_w\mathbb{E}\|\boldsymbol{\lambda}^{(s+1)} - \boldsymbol{\lambda}^{(s)}\|\leq  \tau G_w\sqrt{\mathbb{E}\|\boldsymbol{\lambda}^{(s+1)} - \boldsymbol{\lambda}^{(s)}\|^2}\label{eq: proof DRFA-Prox convex eq5}.
\end{align} 
Let $X = \sqrt{\mathbb{E}\|\boldsymbol{\lambda}^{(s+1)} - \boldsymbol{\lambda}^{(s)}\|^2}$,  $A = \frac{1}{2\gamma}$, $B = -\tau G_w$ and $C =- \frac{1}{2\gamma}\mathbb{E}[\|\bar{\Delta}_s- \Delta_s\|^2]-\frac{1}{2\gamma}\mathbb{E}\|\bar{\Delta}_s  \|^2  $, then we can re-formulate (\ref{eq: proof DRFA-Prox convex eq5}) as:
\begin{align}
    AX^2 + BX + C \leq 0.
\end{align}
Obviously $A\geq 0$.  According to the root of quadratic equation, we know that:
\begin{align}
    X \leq \frac{-B+\sqrt{B^2-4AC}}{2A} &= \frac{\tau G_w + \sqrt{G_w^2\tau^2 + \frac{1}{\gamma^2}(\mathbb{E}[\|\bar{\Delta}_s- \Delta_s\|^2]+ \mathbb{E}\|\bar{\Delta}_s  \|^2) }}{1/\gamma} \nonumber \\
    &\leq \tau \gamma \left(G_w + \sqrt{G_w^2+G_\lambda^2 + \sigma_\lambda^2}\right).\nonumber
\end{align}
Hence, we have 
\begin{align}
    T_1 \leq \tau G_w \mathbb{E}\|\boldsymbol{\lambda}^{(s+1)}- \boldsymbol{\lambda}^{(s)}\| \leq \tau^2 \gamma G_w\left(G_w + \sqrt{G_w^2+G_\lambda^2 + \sigma_\lambda^2}\right),\nonumber
\end{align}
which concludes the proof.

\end{proof}

\noindent\textbf{Proof of Theorem~\ref{thm: regularized convex-linear}.}~We start the proof by noting that $\forall \bm{w} \in \mathcal{W}$, $\forall \boldsymbol{\lambda} \in \Lambda$,  according the convexity in $\bm{w} $ and concavity in $\boldsymbol{\lambda}$, we have:
\begin{align}
     &\mathbb{E}[F(\hat{\bm{w} },\boldsymbol{\lambda} ) - \mathbb{E}[F(\bm{w}  ,\hat{\boldsymbol{\lambda}})]\nonumber\\
     & \leq \frac{1}{T}\sum_{t=1}^T \left\{   \mathbb{E}\left[F( \bm{w}^{(t)},\boldsymbol{\lambda} )\right] - \mathbb{E}\left[F(\bm{w}  ,\boldsymbol{\lambda}^{(\floor{\frac{t}{\tau}})})\right] \right\}\nonumber\\
     &  \leq \frac{1}{T}\sum_{t=1}^T\left \{   \mathbb{E}\left[F( \bm{w}^{(t)},\boldsymbol{\lambda} )\right] -\mathbb{E}\left[F( \bm{w}^{(t)},\boldsymbol{\lambda}^{(\floor{\frac{t}{\tau}})})\right] +\mathbb{E}\left[ F( \bm{w}^{(t)},\boldsymbol{\lambda}^{(\floor{\frac{t}{\tau}})})\right] - \mathbb{E}\left[F(\bm{w},\boldsymbol{\lambda}^{(\floor{\frac{t}{\tau}})})\right] \right\}\nonumber\\
     &  \leq \frac{1}{T}\sum_{s=0}^{S-1} \sum_{t=s\tau+1}^{(s+1)\tau}\mathbb{E}[F( \bm{w}^{(t)},\boldsymbol{\lambda} ) -F( \bm{w}^{(t)},\boldsymbol{\lambda}^{(s)})]  +\frac{1}{T}\sum_{t=1}^T\mathbb{E}[ F( \bm{w}^{(t)},\boldsymbol{\lambda}^{(\floor{\frac{t}{\tau}})}) - F(\bm{w}  ,\boldsymbol{\lambda}^{(\floor{\frac{t}{\tau}})}) ]. \label{eq: proof DRFA-Prox convex eq1}
\end{align}
To bound the first term in (\ref{eq: proof DRFA-Prox convex eq1}), plugging Lemma~\ref{lemma: bounded variance of lambda}  into Lemma~\ref{lemma: DRFA-prox convex-concave one iteration}, and summing over $s=0$ to $S-1$ where $S=T/\tau$, and dividing both sides with $T$ yields:
\begin{align}
    \frac{1}{T}\sum_{s=0}^{S-1}\sum_{t=s\tau+1}^{(s+1)\tau}\left\{ \right. & \left.\mathbb{E}\left[ F(\bm{w}^{(t)},\boldsymbol{\lambda} )\right]- \mathbb{E}\left[  F(\bm{w}^{(t)}, \boldsymbol{\lambda}^{(s)}  )\right]\right\} \nonumber\\
    &  \leq   \frac{1}{2\gamma T} D_{\Lambda}^2+ \frac{1}{2\gamma\tau} \mathbb{E}[\|\bar{\Delta}_s- \Delta_s\|^2]+\tau  \gamma G_w(G_w + \sqrt{G_w^2+G_\lambda^2 + \sigma_\lambda^2})+\gamma \tau G_\lambda^2 \nonumber\\
     &\leq \frac{1}{2\gamma T} D_{\Lambda}^2+ \frac{1}{2\gamma} \mathbb{E}[\|\bar{\Delta}_s- \Delta_s\|^2]+\tau  \gamma (G_w + \sqrt{G_w^2+G_\lambda^2 + \sigma_\lambda^2})+\gamma \tau G_\lambda^2\nonumber\\
    & \leq \frac{D_{\Lambda}^2}{2\gamma T}    + \frac{\gamma \tau\sigma_{ \lambda}^2 }{2m}+\tau  \gamma G_w(G_w + \sqrt{G_w^2 +G_\lambda^2+ \sigma_\lambda^2})+\gamma \tau G_\lambda^2.\nonumber
\end{align}
To bound the second term in (\ref{eq: proof DRFA-Prox convex eq1}), we plug Lemma~\ref{lemma: bounded variance of w} and Lemma~\ref{lemma: deviation} into Lemma~\ref{lemma: one iteration w} and apply the telescoping sum from $t = 1$ to $T$ to get:
\begin{align}
    \frac{1}{T}\sum_{t=1}^T \mathbb{E}[F(\bm{w} ^{(t)},\boldsymbol{\lambda}^{(\floor{\frac{t}{\tau}})}) & -F(\bm{w}  ,\boldsymbol{\lambda}^{(\floor{\frac{t}{\tau}})})]\nonumber\\
    & \leq \frac{1}{2T\eta}\mathbb{E}\|\ \bm{w}^{(0)}- \bm{w} \|^2 +   5L\eta^2\tau^2  \left(\sigma_w^2+\frac{\sigma_w^2}{m} + \Gamma \right) + \frac{\eta G_w^2}{2} + \frac{\eta \sigma_{w }^2 }{2m}\nonumber\\
    & \leq \frac{D_{\mathcal{W}}^2}{2T\eta}  +   5L\eta^2\tau^2  \left(\sigma_w^2+\frac{\sigma_w^2}{m} + \Gamma \right) + \frac{\eta G_w^2}{2} + \frac{\eta \sigma_{w }^2 }{2m},\nonumber
\end{align}
So that we can conclude:
\begin{align}
     \mathbb{E}[F(\hat{\bm{w} },\boldsymbol{\lambda} ) - \mathbb{E}[F(\bm{w}  ,\hat{\boldsymbol{\lambda}})]
     & \leq \frac{D_{\mathcal{W}}^2}{2T\eta}  +   5L\eta^2\tau^2  \left(\sigma_w^2+\frac{\sigma_w^2}{m} + \Gamma \right) + \frac{\eta G_w^2}{2} + \frac{\eta \sigma_{w }^2 }{2m} +  \frac{D_{\Lambda}^2}{2\gamma T} \nonumber\\
     & \quad +\gamma\tau G_\lambda^2+ \frac{\gamma \tau\sigma_{ \lambda}^2 }{2m}+\tau \gamma G_w(G_w + \sqrt{G_w^2 + G_\lambda^2+\sigma_\lambda^2}).\nonumber
\end{align}
Since the RHS does not depend on $\bm{w}$ and $\boldsymbol{\lambda}$, we can maximize over $\boldsymbol{\lambda}$ and minimize over $\bm{w}$ on both sides:
\begin{align}
     &\min_{\bm{w}\in \mathcal{W}}\max_{\boldsymbol{\lambda}\in \Lambda} \mathbb{E}[F(\hat{\bm{w} },\boldsymbol{\lambda} ) - \mathbb{E}[F(\bm{w}  ,\hat{\boldsymbol{\lambda}})]\nonumber\\
     & \leq \frac{D_{\mathcal{W}}^2}{2T\eta}  +   5L\eta^2\tau^2  \left(\sigma_w^2+\frac{\sigma_w^2}{m} + \Gamma \right) + \frac{\eta G_w^2}{2} + \frac{\eta \sigma_{w }^2 }{2m} +  \frac{D_{\Lambda}^2}{2\gamma T} \nonumber\\
     & \quad +\gamma\tau G_\lambda^2+ \frac{\gamma \tau\sigma_{ \lambda}^2 }{2m}+\tau \gamma G_w\left(G_w + \sqrt{G_w^2 + G_\lambda^2+\sigma_\lambda^2}\right).\nonumber
\end{align}
Plugging in $\tau =  \frac{T^{1/4}}{\sqrt{m}}$, $\eta = \frac{1}{4L \sqrt{T}}$, and $\gamma = \frac{1}{T^{5/8}}$, we get:
\begin{equation}
\begin{aligned}
    \max_{\boldsymbol{\lambda}\in \Lambda}\mathbb{E}[F(\hat{\boldsymbol{w}},\boldsymbol{\lambda} )] -\min_{\bm{w}\in\mathcal{W}} \mathbb{E}[F(\boldsymbol{w} ,\hat{\boldsymbol{\lambda}} )] \leq O\Big{(}\frac{D_{\mathcal{W}}^2+G_{w}^2}{\sqrt{T}} +\frac{D_{\Lambda}^2+G_w^2}{T^{3/8}}+\frac{G_\lambda^2  }{m^{1/2}T^{3/8}} +\frac{\sigma_{\lambda}^2}{m^{3/2}T^{3/8}}+ \frac{\sigma_w^2+\Gamma}{m \sqrt{T} }\Big{)},\nonumber
\end{aligned}
\end{equation}
thus concluding the proof.

\subsection{Nonconvex Setting} \label{sec: proof DRFA-Prox nonconvex}
In this section we are going to prove Theorem~\ref{thm:regularized nonconvex-linear}. The whole framework is similar to the proof of Theorem~\ref{thm: regularized convex-linear}, but to bound $\mathbb{E}\left[\Phi(\bm{w}^{(t)} )\right]-  \mathbb{E}\left[F(\bm{w}^{(t)}, \boldsymbol{\lambda}^{(\floor{\frac{t}{\tau}})}  )\right]$ term, we employ different technique for proximal method. The following lemma characterize the bound of $\mathbb{E}\left[\Phi(\bm{w}^{(t)} )\right]-  \mathbb{E}\left[F(\bm{w}^{(t)}, \boldsymbol{\lambda}^{(\floor{\frac{t}{\tau}})}  )\right]$:
\begin{lemma}\label{lemma: DRFA-prox nonconvex dual iteration}
For {\sffamily{DRFA-Prox}}, under Theorem~\ref{thm:regularized nonconvex-linear}'s assumption, the following statement holds true:
\begin{align}
   &\frac{1}{T} \sum_{t=1}^{T} \mathbb{E}\left[\Phi(\bm{w}^{(t)} )-  F(\bm{w}^{(t)}, \boldsymbol{\lambda}^{(\floor{\frac{t}{\tau}})} )\right] \nonumber\\
   &\leq   2\sqrt{S}\tau \eta G_w \sqrt{G_w^2+\sigma_w^2}+    \gamma  \tau \frac{ \sigma_{\lambda}^2}{2m}+    \gamma  \tau \frac{ G_{\lambda}^2}{2} +  \frac{D_{\Lambda}^2}{2\sqrt{S}\tau }+\tau   \gamma G_w\left(G_w + \sqrt{G_w^2 +G_\lambda^2+ \sigma_\lambda^2}\right). \nonumber
\end{align}
\begin{proof}
We recall that in Lemma~\ref{lemma: DRFA-prox convex-concave one iteration}, we have:
{\begin{align}
   &\sum_{t=s\tau+1}^{(s+1)\tau}  \left(\mathbb{E}\left[ F(\bm{w}^{(t)},\boldsymbol{\lambda} )\right] - \mathbb{E}\left[  F(\bm{w}^{(t)},\boldsymbol{\lambda}^{(s)} )\right]\right)\nonumber\\
   &\leq - \frac{1}{2\gamma}\mathbb{E}\| \boldsymbol{\lambda}^{(s+1)}-\boldsymbol{\lambda}\|^2 + \frac{1}{2\gamma} \mathbb{E}[\| \boldsymbol{\lambda}^{(s)}-\boldsymbol{\lambda}\|^2]+ \frac{1}{2\gamma} \mathbb{E}[\|\bar{\Delta}_s- \Delta_s\|^2]+\frac{1}{2\gamma} \mathbb{E}[\|\bar{\Delta}_s \|^2]  \nonumber \\ &\quad +\tau^2  \gamma \left(G_w + \sqrt{G_w^2+G_\lambda^2 + \sigma_\lambda^2}\right).\nonumber
\end{align}}
 Adding $\sum_{t=s\tau+1}^{(s+1)\tau} \mathbb{E}\left[ \Phi(\bm{w}^{(t)})\right] $ to both sides, and re-arranging the terms give:
 {\begin{align}
     &\sum_{t=s\tau+1}^{(s+1)\tau} \left(\mathbb{E}\left[ \Phi(\bm{w}^{(t)})\right]  \right.  \left. - \mathbb{E}\left[  F( \bm{w} ^{(t)},\boldsymbol{\lambda}^{(s)} )\right] \right) \nonumber \\ &\leq \sum_{t=s\tau+1}^{(s)\tau} \left(\mathbb{E}\left[ \Phi(\bm{w}^{(t)})\right] -\mathbb{E}\left[ F( \bm{w}^{(t)},\boldsymbol{\lambda} )\right]\right) - \frac{1}{2\gamma}\mathbb{E}\| \boldsymbol{\lambda}^{(s+1)}-\boldsymbol{\lambda}\|^2\nonumber\\
     & \quad + \frac{1}{2\gamma} \mathbb{E}[\| \boldsymbol{\lambda}^{(s)}-\boldsymbol{\lambda}\|^2]+ \frac{1}{2\gamma} \mathbb{E}[\|\bar{\Delta}_s- \Delta_s\|^2]+\frac{1}{2\gamma} \mathbb{E}[\|\bar{\Delta}_s \|^2]  +\tau^2  \gamma G_w \left(G_w + \sqrt{G_w^2+G_\lambda^2 + \sigma_\lambda^2}\right).\nonumber
\end{align}}
 
 Then, we follow the same procedure as in Lemma~\ref{lm: nonconvex lm3}. Without loss of generality we assume $\sqrt{S}$ is an integer, so we can equally divide index $0$ to $S-1$ into $\sqrt{S}$ groups. Then we examine one block by summing $s$ from $s = c\sqrt{S}$ to $(c+1)\sqrt{S}-1$, and set $\boldsymbol{\lambda} = \boldsymbol{\lambda}^*(\bm{w}^{(c+1)\sqrt{S}\tau})$:
\begin{align}
     \sum_{s=c\sqrt{S}}^{(c+1)\sqrt{S}-1}&\sum_{t=s\tau+1}^{(s+1)\tau}  \left( \mathbb{E}\left[\Phi(\bm{w}^{(t)} )\right]-  \mathbb{E}\left[F(\bm{w}^{(t)}, \boldsymbol{\lambda}^{(s)} )\right]\right) \nonumber\\
     &\leq \sum_{s=c\sqrt{S}}^{(c+1)\sqrt{S}-1} \sum_{t=s\tau+1}^{(s+1)\tau} \left(\mathbb{E}\left[ F(\bm{w}^{(t)}, \boldsymbol{\lambda}^*(\bm{w}^{t}) )\right]- \mathbb{E}\left[ F(\bm{w}^{(t)}, \boldsymbol{\lambda}^*(\bm{w}^{(c+1)\sqrt{S}\tau}) )\right] \right) \nonumber\\
     & \quad +\sqrt{S}\tau^2  \gamma G_w (G_w + \sqrt{G_w^2+G_\lambda^2 + \sigma_\lambda^2}) +    \gamma  \tau^2 \frac{\sqrt{S}\sigma_{\lambda}^2}{2m} +    \gamma  \tau^2 \frac{\sqrt{S}G_{\lambda}^2}{2} \nonumber\\
     & \quad +  \frac{1}{2\gamma }\sum_{s=c\sqrt{S}}^{(c+1)\sqrt{S}-1}\left(\mathbb{E}\left[\left\|\boldsymbol{\lambda}^*(\bm{w}^{(c+1)\sqrt{S}\tau})-\boldsymbol{\lambda}^{(s)}\right\|^2 \right] - \mathbb{E}\left[\left\|\boldsymbol{\lambda}^*(\bm{w}^{(c+1)\sqrt{S}\tau})-\boldsymbol{\lambda}^{(s+1)}\right\|^2 \right]\right)\nonumber
  \end{align}   
Adding and subtracting  $\mathbb{E}\left[ F(\bm{w}^{(t)}, \boldsymbol{\lambda}^*(\bm{w}^{t}) )\right]- \mathbb{E}\left[ F(\bm{w}^{(t)}, \boldsymbol{\lambda}^*(\bm{w}^{(c+1)\sqrt{S}\tau}) )\right] $ yields:
\begin{align}
      \sum_{s=c\sqrt{S}}^{(c+1)\sqrt{S}-1}&\sum_{t=s\tau+1}^{(s+1)\tau}  \left( \mathbb{E}\left[\Phi(\bm{w}^{(t)} )\right]-  \mathbb{E}\left[F(\bm{w}^{(t)}, \boldsymbol{\lambda}^{(s)} )\right]\right) \nonumber\\
     &\leq \sum_{s=c\sqrt{S}}^{(c+1)\sqrt{S}-1} \sum_{t=s\tau+1}^{(s+1)\tau} \left(\mathbb{E}\left[ F(\bm{w}^{(t)}, \boldsymbol{\lambda}^*(\bm{w}^{t}) )\right]-\mathbb{E}\left[ F(\bm{w}^{((c+1)\sqrt{S}\tau)}, \boldsymbol{\lambda}^*(\bm{w}^{t}) )\right] \right.\nonumber\\
     &\quad \left.+\mathbb{E}\left[ F(\bm{w}^{((c+1)\sqrt{S}\tau)}, \boldsymbol{\lambda}^*(\bm{w}^{(c+1)\sqrt{S}\tau}) )\right]- \mathbb{E}\left[ F(\bm{w}^{(t)},\boldsymbol{\lambda}^*(\bm{w}^{(c+1)\sqrt{S}\tau}))\right] \right) \nonumber\\
     & \quad+    \gamma  \tau^2\frac{\sqrt{S}\sigma_{\lambda}^2}{2m}+    \gamma  \tau^2 \frac{\sqrt{S}G_{\lambda}^2}{2} +\sqrt{S}\tau^2  \gamma G_w(G_w + \sqrt{G_w^2 +G_\lambda^2+ \sigma_\lambda^2})\nonumber  \\
     &\quad +  \frac{1}{2\gamma }\sum_{s=c\sqrt{S}+1}^{(c+1)\sqrt{S}}\left(\mathbb{E}\left[\left\|\boldsymbol{\lambda}^*(\bm{w}^{(c+1)\sqrt{S}\tau})-\boldsymbol{\lambda}^{(s)}\right\|^2 \right] - \mathbb{E}\left[\left\|\boldsymbol{\lambda}^*(\bm{w}^{(c+1)\sqrt{S}\tau})-\boldsymbol{\lambda}^{(s+1)}\right\|^2 \right]\right)\nonumber\\
     &\leq \sum_{s=c\sqrt{S}}^{(c+1)\sqrt{S}-1} \sum_{t=s\tau+1}^{(s+1)\tau} (2\sqrt{S}\tau\eta G_w \sqrt{G_w^2+\sigma_w^2})+    \gamma  \tau^2\frac{\sqrt{S}\sigma_{\lambda}^2}{2m}+    \gamma  \tau^2 \frac{\sqrt{S}G_{\lambda}^2}{2}  +  \frac{D_{\Lambda}^2}{2\gamma }\nonumber\\
     &\quad +\sqrt{S}\tau^2  \gamma G_w \left(G_w + \sqrt{G_w^2 +G_\lambda^2+ \sigma_\lambda^2}\right)\nonumber\\
     & \leq  2S\tau^2 \eta G_w \sqrt{G_w^2+\sigma_w^2}+    \gamma \tau^2\frac{\sqrt{S} \sigma_{\lambda}^2}{2m}+    \gamma  \tau^2 \frac{\sqrt{S}G_{\lambda}^2}{2}  +  \frac{ D_{\Lambda}^2}{2\gamma } \nonumber \\
     & \quad +\sqrt{S}\tau^2  \gamma G_w\left(G_w + \sqrt{G_w^2+G_\lambda^2 + \sigma_\lambda^2}\right). \nonumber
\end{align}
 So we can conclude that:
 \begin{align}
     &\sum_{s=c\sqrt{S}}^{(c+1)\sqrt{S}-1}\sum_{t=s\tau+1}^{(s+1)\tau}  \left( \mathbb{E}\left[\Phi(\bm{w}^{(t)} )\right]-  \mathbb{E}\left[F(\bm{w}^{(t)}, \boldsymbol{\lambda}^{(s)} )\right]\right) \nonumber\\
     &\leq  2S\tau^2 \eta G_w \sqrt{G_w^2+\sigma_w^2}+  \gamma \tau^2\frac{\sqrt{S}\sigma_{\lambda}^2}{2m} +   \gamma \tau^2\frac{\sqrt{S}G_{\lambda}^2}{2} \frac{D_{\Lambda}^2}{2\gamma } +\sqrt{S}\tau^2  \gamma G_w\left(G_w + \sqrt{G_w^2 +G_\lambda^2+ \sigma_\lambda^2}\right)\nonumber
\end{align}
 Summing above inequality over $c$ from $0$ to $\sqrt{S}-1$, and dividing both sides by $T$ gives
  \begin{align}
      \frac{1}{T}\sum_{s=0}^{S-1}&\sum_{t=s\tau+1}^{(s+1)\tau}  \left( \mathbb{E}\left[\Phi(\bm{w}^{(t)} )\right]-  \mathbb{E}\left[F(\bm{w}^{(t)}, \boldsymbol{\lambda}^{(s)} )\right]\right) \nonumber\\
      &\leq  2\sqrt{S}\tau \eta G_w \sqrt{G_w^2+\sigma_w^2}+  \gamma  \tau \frac{ \sigma_{\lambda}^2}{2m}+  \gamma  \tau \frac{ G_{\lambda}^2}{2 } +  \frac{D_{\Lambda}^2}{2\sqrt{S}\tau \gamma }+\tau   \gamma G_w\left(G_w + \sqrt{G_w^2 +G_\lambda^2+ \sigma_\lambda^2}\right), \nonumber
\end{align}
 which  concludes the proof.
\end{proof}

\end{lemma}

\noindent\textbf{Proof of Theorem~\ref{thm:regularized nonconvex-linear}.}~Now we proceed to the formal proof of Theorem~\ref{thm:regularized nonconvex-linear}. Re-arranging terms in Lemma~\ref{lm: nonconvex lm1},  summing over $t = 1$ to $T$, and dividing by $T$ yields:
{
\begin{align}
      &\frac{1}{T}\sum_{t=1}^T \mathbb{E}\left[\left\|\nabla \Phi_{1/2L} \right.\right.  \left.\left. (\bm{w}^{(t)})\right\|^2\right] \nonumber\\
&\leq \frac{4}{\eta T}\mathbb{E}[\Phi_{1/2L} (\bm{w}^{(0)})]   +    \frac{1}{2T}\sum_{t=1}^T D_{\mathcal{W}}L^2 \mathbb{E}\left[ \frac{1}{m}\sum_{i\in\mathcal{D}^{(\floor{\frac{t}{\tau}})}} \left\|   \bm{w}^{(t)}_i  -  \bm{w}^{(t)}  \right\| \right] \nonumber \\
    & \quad +   L \frac{1}{2T}\sum_{t=1}^T\left( \mathbb{E}[\Phi(\bm{w}^{(t)})]- \mathbb{E}[F(\bm{w}^{(t)}, \boldsymbol{\lambda}^{\floor{\frac{t}{\tau}}})]\right).\nonumber
\end{align}}
Plugging in Lemmas~\ref{lemma: deviation2} and \ref{lemma: DRFA-prox nonconvex dual iteration} 
yields:
\begin{align}
      \frac{1}{T}\sum_{t=1}^T &\mathbb{E}\left[\left\|\nabla \Phi_{1/2L} (\bm{w}^{(t)})\right\|^2\right] \nonumber\\
&\leq \frac{4}{\eta T}\mathbb{E}[\Phi_{1/2L} (\bm{w}^{(0)})]   +      \eta \tau D_{\mathcal{W}}L^2 \left(\sigma_w +\frac{\sigma_w }{m} + \sqrt{\Gamma} \right). \nonumber \\
    &\quad + \frac{L}{2 } \left(  2\sqrt{S}\tau \eta G_w \sqrt{G_w^2+\sigma_w^2}+ \gamma  \tau \frac{ \sigma_{\lambda}^2}{2m}+ \gamma  \tau \frac{ G_{\lambda}^2}{2} +  \frac{D_{\Lambda}^2}{2\sqrt{S}\tau }+\tau\gamma G_w(G_w + \sqrt{G_w^2+G_\lambda^2 + \sigma_\lambda^2})\right)\nonumber\\
    & \leq \frac{4}{\eta T}\mathbb{E}[\Phi_{1/2L} (\bm{w}^{(0)})]   +    \eta \tau D_{\mathcal{W}}L^2 \left(\sigma_w +\frac{\sigma_w }{m} + \sqrt{\Gamma} \right) \nonumber \\
    &\quad +    \sqrt{S} \tau  \eta G_w L\sqrt{G_w^2+\sigma_w^2}+ \gamma  \tau\frac{\sigma_{\lambda}^2L}{4m}+ \gamma  \tau \frac{ G_{\lambda}^2 L}{4} +  \frac{ D_{\Lambda}^2L}{4\sqrt{S}\gamma\tau}+\frac{\tau \gamma L G_w(G_w + \sqrt{G_w^2 +G_\lambda^2+ \sigma_\lambda^2})}{2}.\nonumber 
\end{align}
Plugging in $\eta = \frac{1}{4LT^{3/4}}$ , $\gamma = \frac{1}{T^{1/2}}$ and $\tau = T^{1/4}$ we recover the stated convergence rate as:
\begin{align}
      \frac{1}{T}\sum_{t=1}^T &\mathbb{E}\left[\left\|\nabla \Phi_{1/2L} \vphantom{\left\|\bm{w}^{(t)}\right\|^2} (\bm{w}^{(t)})\right\|^2\right]  \nonumber\\
    & \leq \frac{4}{ T^{1/4}}\mathbb{E}[\Phi_{1/2L} (\bm{w}^{(0)})]   +    \frac{L^2}{T^{1/2}}\left(\sigma_w +\frac{\sigma_w }{m} + \sqrt{\Gamma} \right) \nonumber \\
    &+    \frac{1}{T^{1/8}}G_w L\sqrt{G_w^2+\sigma_w^2}+  \frac{\sigma_{\lambda}^2L}{4m T^{1/4}}+  \frac{G_{\lambda}^2L}{4 T^{1/4}} +  \frac{ D_{\Lambda}^2L}{4 T^{1/8}}+\frac{ L G_w(G_w + \sqrt{G_w^2 +G_\lambda^2+ \sigma_\lambda^2})}{2T^{1/4}}. \nonumber
\end{align}
\qed

\section{Proof of Convergence of {\sffamily{DRFA-GA}} in Strongly-Convex-Strongly-Concave Setting}\label{sec:scsc}

In this section we proceed to the proof of the  convergence in strongly-convex-strongly-concave setting (Theorem~\ref{theorem2}). In this section we  abuse the notation and use the following  definition for $\bar{\bm{u}}_t$:
\begin{equation}
    \bar{\bm{u}}_t = \sum_{i=1}^N \lambda_i^{(\floor{\frac{t}{\tau}})} \nabla f_i(\bm{w}_i^{(t)}).\nonumber
\end{equation}

\subsection{Overview of the Proof}
We again start with the dynamic of one iteration:
\begin{align}
     \mathbb{E}\left\|\bm{w} ^{(t+1)} - \bm{w} ^*\right\|^2 &\leq \left (1-\frac{\mu\eta}{2}\right) \mathbb{E}\left\|\bm{w} ^{(t)}- \bm{w} ^*\right\|^2 -\eta\mathbb{E}\left[\Phi( \bm{w}^{(t)})-\Phi(\bm{w} ^*)\right]\nonumber\\
     & \quad +  \eta^2\frac{2\sigma_w^2+4G_w^2}{m}  4L^2\left(\eta^2+\frac{\eta}{\mu}\right)  \mathbb{E}\left[\delta^{(t)} \right] \nonumber\\
     & \quad +4\left(\frac{\eta}{\mu}+\eta^2\right)\mathbb{E}\|\nabla_{\bm{w} }F( \bm{w}^{(t)},\boldsymbol{\lambda}^{(\floor{\frac{t}{\tau}})})- \nabla\Phi( \bm{w}^{(t)})\|^2.\nonumber
\end{align}
In addition to the local-global deviation, in this case  we also have a new term $\|\nabla_{\bm{w} }F( \bm{w}^{(t)},\boldsymbol{\lambda}^{(\floor{\frac{t}{\tau}})})- \nabla\Phi( \bm{w}^{(t)})\|^2$. Recall that $\nabla\Phi( \bm{w}^{(t)})$ is the gradient evaluated at $\boldsymbol{\lambda}^*(\bm{w}^{(t)})$. A straightforward approach is to use the smoothness of $\Phi$, to convert the difference between gradient to the difference between $\boldsymbol{\lambda}^{(\floor{\frac{t}{\tau}})}$ and $\boldsymbol{\lambda}^*(\bm{w}^{(t)})$. By examining the dynamic of $\boldsymbol{\lambda}$, we can prove that:
\begin{align}
    \mathbb{E}\left\|\boldsymbol{\lambda}^*( \bm{w}^{(t)}) - \boldsymbol{\lambda}^{(\floor{\frac{t}{\tau}})} \right\|^2\leq   2\left(1-\frac{1}{2\kappa}\right)^{(\floor{\frac{t}{\tau}})}\mathbb{E}\left\|\boldsymbol{\lambda}^{(0)}- \boldsymbol{\lambda}^*(\bm{w} ^{(0)})\right\|^2   + 2(4\kappa^2+1)\kappa^2  \tau^2 \eta^2 G_{w}^2.\nonumber
\end{align}
Putting these pieces together, and unrolling the recursion will conclude the proof.

\subsection{Proof of Technical Lemmas}
\begin{lemma} [~\citet{lin2019gradient}. Properties of $\Phi (\cdot)$ and $\boldsymbol{\lambda}^*(\cdot)$] \label{danskin}
If $F(\cdot, \boldsymbol{\lambda})$ is $L$-smooth function and $F(\boldsymbol{w},\cdot)$ is $\mu$-strongly-concave, $L$-smooth function, let $\kappa = \frac{L}{\mu}$, then $\Phi (\boldsymbol{w})$ is $\alpha$-smooth function where $\alpha = L + \kappa L$ and $\boldsymbol{\lambda}^*(\boldsymbol{w})$ is $\kappa$-Lipschitz. Also $\nabla \Phi(\boldsymbol{w}) = \nabla_{\boldsymbol{w}}F(\boldsymbol{w},\boldsymbol{\lambda}^*(\boldsymbol{w}
))$.
\end{lemma}

\begin{lemma}
\label{lemma: one iteration SCSC}
For {\sffamily{DRFA-GA}}, under Theorem~\ref{theorem2}'s assumptions,  the following holds true: 
\begin{align}
     \mathbb{E}\left\|\bm{w} ^{(t+1)} - \bm{w} ^*\right\|^2 &\leq \left (1-\frac{\mu\eta}{2}\right) \mathbb{E}\left\|\bm{w} ^{(t)}- \bm{w} ^*\right\|^2 -\eta\mathbb{E}\left[\Phi( \bm{w}^{(t)})-\Phi(\bm{w} ^*)\right]\nonumber\\
     & \quad +  \eta^2\frac{2\sigma_w^2+4G_w^2}{m}  4L^2\left(\eta^2+\frac{\eta}{\mu}\right)  \mathbb{E}\left[\delta^{(t)} \right] \\ & \quad +4\left(\frac{\eta}{\mu}+\eta^2\right)\mathbb{E}\|\nabla_{\bm{w} }F( \bm{w}^{(t)},\boldsymbol{\lambda}^{(\floor{\frac{t}{\tau}})})- \nabla\Phi( \bm{w}^{(t)})\|^2. \nonumber
\end{align}
\end{lemma}
\begin{proof}
According to Lemma B2 in~\cite{lin2020near}, if $F(\cdot,\boldsymbol{\lambda})$ is $\mu$-strongly-convex, then $\Phi(\cdot)$ is also $\mu$-strongly-convex. Noting this, from the strong convexity and the updating rule we have:
\begin{align}
   \mathbb{E} \|\bm{w} ^{(t+1)}& - \bm{w} ^*\|^2 \nonumber\\
   &= \mathbb{E}\left\|\prod_{\mathcal{W}}\left(\bm{w} ^{(t)}- \eta \bm{u}^{(t)}\right)- \bm{w} ^*\right\|^2 \leq \mathbb{E}\|\bm{w} ^{(t)}- \eta \Bar{\bm{u}}^{(t)} - \bm{w} ^*\|^2 + \eta^2 \mathbb{E}\|\Bar{\bm{u}}^{(t)} - \bm{u}^{(t)}\|^2 \nonumber\\
   &   = \mathbb{E}\|\bm{w} ^{(t)}- \bm{w} ^*\|^2 +  \underbrace{\mathbb{E}[- 2\eta \langle \Bar{\bm{u}}^{(t)},\bm{w} ^{(t)}- \bm{w} ^* \rangle]}_{T_1} + \underbrace{\eta^2  \mathbb{E}\|\Bar{\bm{u}}^{(t)}\|^2}_{T_2} + \eta^2 \mathbb{E}\|\Bar{\bm{u}}^{(t)} - \bm{u}^{(t)}\|^2 \label{l71} 
\end{align}
First we are to bound the variance $\mathbb{E}\|\Bar{\bm{u}}^{(t)} - \bm{u}^{(t)}\|^2$:
\begin{align}
    \mathbb{E}\|\Bar{\bm{u}}^{(t)} - \bm{u}^{(t)}\|^2 & = \mathbb{E}\left\|\frac{1}{m}\sum_{i\in\mathcal{D}^{(\floor{\frac{t}{\tau}})}} \nabla f_i(\bm{w}_i^{(t)}) - \bar{\bm{u}}^{(t)}\right\|^2\nonumber\\
      &= \mathbb{E}\left\|\frac{1}{m}\sum_{i\in\mathcal{D}^{(\floor{\frac{t}{\tau}})}} \nabla f_i(\bm{w}_i^{(t)};\xi_i^{(t)}) - \frac{1}{m}\sum_{i\in\mathcal{D}^{(\floor{\frac{t}{\tau}})}} \bar{\bm{u}}^{(t)}\right\|^2  \leq \frac{2\sigma_w^2 + 4G_w^2}{m},\nonumber
\end{align}
where we use the fact $Var(\sum_{i=1}^m \bm{X}_i) = \sum_{i=1}^m  Var(\bm{X}_i)$ for independent variables $\bm{X}_i, i=1, \ldots, m$, and $Var(\nabla f_i(\bm{w}_i^{(t)};\xi_i^{(t)}) ) = \mathbb{E}\left\| \nabla f_i(\bm{w}_i^{(t)};\xi_i^{(t)}) -  \bar{\bm{u}}^{(t)}\right\|^2 \leq   2\left\| \nabla f_i(\bm{w}_i^{(t)};\xi_i^{(t)})- \nabla f_i(\bm{w}_i^{(t)} )\right\|^2+2\left\|\nabla f_i(\bm{w}_i^{(t)} )- \bar{\bm{u}}^{(t)}\right\|^2\leq 2\sigma_w^2 + 4G_w^2$.

Then we switch to bound $T_1$:
\begin{align}
   T_1 &=  2\eta\mathbb{E}\left[-\left\langle \nabla \Phi( \bm{w}^{(t)}),  \bm{w}^{(t)} - \bm{w} ^* \right\rangle + \left\langle \nabla \Phi( \bm{w}^{(t)}) - \Bar{\bm{u}}^{(t)},  \bm{w}^{(t)} - \bm{w} ^* \right\rangle\right] \nonumber\\
   & \leq 2\eta\mathbb{E}\left[-(\Phi( \bm{w}^{(t)})- \Phi(\bm{w} ^*))-\frac{\mu}{2}\| \bm{w}^{(t)} - \bm{w} ^*\|^2+ \frac{1}{\mu}\| \nabla
   \Phi( \bm{w}^{(t)})-\Bar{\bm{u}}^{(t)}\|^2+ \frac{\mu}{4}\|   \bm{w}^{(t)} - \bm{w} ^* \|^2\right] \nonumber\\
   & \leq \mathbb{E}\left[-2\eta(\Phi( \bm{w}^{(t)})- \Phi(\bm{w} ^*))-\frac{\mu\eta}{2}\| \bm{w}^{(t)} - \bm{w} ^*\|^2+ \frac{2\eta}{\mu}\| \nabla
   \Phi( \bm{w}^{(t)})-\Bar{\bm{u}}^{(t)}\|^2\right] \nonumber\\
   & \leq \mathbb{E}\left[-2\eta(\Phi( \bm{w}^{(t)})- \Phi(\bm{w} ^*))-\frac{\mu\eta}{2}\| \bm{w}^{(t)} - \bm{w} ^*\|^2+ \frac{4\eta}{\mu}\left\| \nabla
   \Phi( \bm{w}^{(t)})-\nabla_{\bm{w} }F( \bm{w}^{(t)},\boldsymbol{\lambda}^{(\floor{\frac{t}{\tau}})})\right\|^2 \right. \nonumber\\
   & \left. \quad + \frac{4\eta}{\mu}\| \nabla_{\bm{w} }F( \bm{w}^{(t)},\boldsymbol{\lambda}^{(\floor{\frac{t}{\tau}})})-\Bar{\bm{u}}^{(t)}\|^2\right]\nonumber\\
   &  \leq \mathbb{E}\left[-2\eta(\Phi( \bm{w}^{(t)})- \Phi(\bm{w} ^*))-\frac{\mu\eta}{2}\| \bm{w}^{(t)} - \bm{w} ^*\|^2+ \frac{4\eta}{\mu}\left\| \nabla
   \Phi( \bm{w}^{(t)})-\nabla_{\bm{w} }F( \bm{w}^{(t)},\boldsymbol{\lambda}^{(\floor{\frac{t}{\tau}})})\right\|^2 \right. \nonumber\\
   &  \left.\quad + \frac{4L^2\eta}{\mu}\sum_{k=1}^K \lambda^{(\floor{\frac{t}{\tau}})}_i\| \bm{w}^{(t)}-\bm{w} ^{(t)}_i\|^2\right],\nonumber
\end{align}
where in the second step we use the arithmetic and geometric inequality and the strong convexity of $\Phi(\cdot)$; and at the last step we use the smoothness, the convexity of $\|\cdot\|^2$ and Jensen's inequality. 

Then, we can bound $T_2$ as:
\begin{align}
   T_2 &\leq \eta^2\mathbb{E}\left[4\left\|\Bar{\bm{u}}^{(t)} - \nabla_{\bm{w} }F( \bm{w}^{(t)},\boldsymbol{\lambda}^{(\floor{\frac{t}{\tau}})})\right\|^2 + 4\|\nabla_{\bm{w} }F( \bm{w}^{(t)},\boldsymbol{\lambda}^{(\floor{\frac{t}{\tau}})})- \nabla\Phi( \bm{w}^{(t)})\|^2 \right.\nonumber\\
   & \qquad \qquad \qquad\left. +2 \left\|\nabla\Phi( \bm{w}^{(t)})\right\|^2\right] \nonumber\\
   & \leq \eta^2\mathbb{E}\left[4\left\|\Bar{\bm{u}}^{(t)} - \nabla_{\bm{w} }F( \bm{w}^{(t)},\boldsymbol{\lambda}^{(\floor{\frac{t}{\tau}})})\right\|^2 + 4\left\|\nabla_{\bm{w} }F( \bm{w}^{(t)},\boldsymbol{\lambda}^{(\floor{\frac{t}{\tau}})})- \nabla\Phi( \bm{w}^{(t)})\right\|^2 \right. \nonumber\\
   & \qquad\qquad\qquad \left. + 4\alpha (\Phi( \bm{w}^{(t)})-\Phi(\bm{w} ^*)) \vphantom{\left\|\nabla\Phi( \bm{w}^{(t)})\right\|^2} \right]\nonumber\\
   & \leq \eta^2\mathbb{E}\left[4L^2\sum_{i=1}^N \lambda^{(\floor{\frac{t}{\tau}})}_i\| \bm{w}^{(t)}-\bm{w} ^{(t)}_i\|^2 + 4\|\nabla_{\bm{w} }F( \bm{w}^{(t)},\boldsymbol{\lambda}^{(\floor{\frac{t}{\tau}})})- \nabla\Phi( \bm{w}^{(t)})\|^2  \right.\nonumber\\
   & \qquad\qquad\qquad \left. + 4\alpha(\Phi(\bm{w} _t)-\Phi(\bm{w} ^*)) \vphantom{\left\|\nabla\Phi( \bm{w}^{(t)})\right\|^2} \right ]\nonumber\\
   & \leq \eta^2\mathbb{E}\left[4L^2 \frac{1}{m}\sum_{i\in\mathcal{D}^{(\floor{\frac{t}{\tau}})}} \| \bm{w}^{(t)}-\bm{w} ^{(t)}_i\|^2 + 4\|\nabla_{\bm{w} }F( \bm{w}^{(t)},\boldsymbol{\lambda}^{(\floor{\frac{t}{\tau}})})- \nabla\Phi( \bm{w}^{(t)})\|^2 \right.\nonumber\\
   & \qquad\qquad\qquad \left. + 4\alpha(\Phi(\bm{w} _t)-\Phi(\bm{w} ^*)) \vphantom{\left\|\nabla\Phi( \bm{w}^{(t)})\right\|^2} \right].\nonumber
\end{align}
Plugging $T_1$ and $T_2$ back to (\ref{l71}) results in:
\begin{align}
     \mathbb{E}\left\|\bm{w} ^{(t+1)} - \bm{w} ^*\right\|^2 &\leq \left (1-\frac{\mu\eta}{2}\right) \mathbb{E}\left\|\bm{w} ^{(t)}- \bm{w} ^*\right\|^2 +(4\alpha \eta^2-2\eta)\mathbb{E}\left[\Phi( \bm{w}^{(t)})-\Phi(\bm{w} ^*)\right] \nonumber \\ 
     & \quad +  \eta^2\frac{2\sigma_w^2+4G_w^2}{m} + 4L^2\left(\eta^2+\frac{\eta}{\mu}\right)  \mathbb{E}\left[\delta^{(t)} \right] \nonumber\\
     & \quad +4\left(\frac{\eta}{\mu}+\eta^2\right)\mathbb{E}\|\nabla_{\bm{w} }F( \bm{w}^{(t)},\boldsymbol{\lambda}^{(\floor{\frac{t}{\tau}})})- \nabla\Phi( \bm{w}^{(t)})\|^2 \label{l70}.
\end{align}
By choosing $\eta < \frac{1}{4\alpha}$, it holds that $(4\alpha \eta^2 - 2\eta) < -\eta$, therefore we conclude the proof.
\end{proof}

\begin{lemma}[Decreasing Optimal Gap of $\boldsymbol{\lambda}$]
\label{lemma: optimal gap}
For {\sffamily{DRFA-GA}}, if $F(\bm{w} ,\cdot)$ is $\mu$-strongly-concave, choosing $\gamma = \frac{1}{L}$, the optimality gap of $\boldsymbol{\lambda}$ is decreasing by the following recursive relation:
\begin{align}
    \mathbb{E}\left\|\boldsymbol{\lambda}^*( \bm{w}^{(t)}) - \boldsymbol{\lambda}^{(\floor{\frac{t}{\tau}})} \right\|^2\leq   2\left(1-\frac{1}{2\kappa}\right)^{\floor{\frac{t}{\tau}}}\mathbb{E}\left\|\boldsymbol{\lambda}^{(0)}- \boldsymbol{\lambda}^*(\bm{w} ^{(0)})\right\|^2   + 2(4\kappa^2+1)\kappa^2  \tau^2 \eta^2 G_{w}^2.\nonumber
\end{align} 
\end{lemma}
\emph{Proof:}
Assume $s\tau+1 \leq t\leq (s+1)\tau$. By the Jensen's inequality:
\begin{align}
     \mathbb{E}\|\boldsymbol{\lambda}^*( \bm{w}^{(t)}) - \boldsymbol{\lambda}^{(\floor{\frac{t}{\tau}})} \|^2 &\leq  2 \mathbb{E}\|\boldsymbol{\lambda}^*(\bm{w}^{(t)})- \boldsymbol{\lambda}^*( \bm{w}  ^{(s\tau)}) \|^2+2 \mathbb{E}\| \boldsymbol{\lambda}^*( \bm{w}  ^{(s\tau)})- \boldsymbol{\lambda}^{(s)} \|^2. \nonumber 
\end{align}
Firstly we are going to bound $\mathbb{E}\|\boldsymbol{\lambda}^*(\bm{w}^{(t)})- \boldsymbol{\lambda}^*( \bm{w}^{(s\tau)}) \|^2$. We use the $\kappa$-Lipschitz property of $\boldsymbol{\lambda}^*(\cdot)$:
\begin{align}
    \mathbb{E}\left\|\boldsymbol{\lambda}^*(\bm{w}^{(t)})- \boldsymbol{\lambda}^*\left( \bm{w}  ^{(s\tau)}\right) \right\|^2 \leq \kappa^2 \mathbb{E}\|\bm{w} ^{(t)} - \bm{w}^{(s\tau)}\|^2 \leq \kappa^2\tau^2\eta^2 G_{w}^2.\nonumber
\end{align}

Then we switch to bound $\mathbb{E}\|\boldsymbol{\lambda}^{(s)}-\boldsymbol{\lambda}^*( \bm{w}^{(s\tau)}) \|^2$. We apply the Jensen's inequality first to get:
\begin{align}
    \mathbb{E}\left\|\boldsymbol{\lambda}^{(s)} \right.  \left. -\boldsymbol{\lambda}^*( \bm{w}^{(s\tau)})\right\|^2     & \leq \left(1+\frac{1}{2(\kappa-1)}\right)\mathbb{E}\left\|\boldsymbol{\lambda}^{(s)}-\boldsymbol{\lambda}^*\left( \bm{w}^{((s-1)\tau)}\right)\right\|^2\nonumber\\ &\quad +\left(1+2(\kappa-1)\right)  \mathbb{E}\left\|\boldsymbol{\lambda}^*\left( \bm{w}^{((s-1)\tau)}\right) - \boldsymbol{\lambda}^*\left( \bm{w}^{(s\tau)}\right)\right\|^2 \nonumber\\
    & \leq \left(1+\frac{1}{2(\kappa-1)}\right)\mathbb{E}\left\|\boldsymbol{\lambda}^{(s)}-\boldsymbol{\lambda}^*\left( \bm{w}^{((s-1)\tau)}\right)\right\|^2  + 2 \kappa^3  \tau^2 \eta^2G_{w}^2, \label{eq: lambda gap eq 1}
\end{align}
where we use the fact that $\boldsymbol{\lambda}^*(\cdot)$ is $\kappa$-Lipschitz.\\

To bound $\mathbb{E} \left\|\boldsymbol{\lambda}^{(s)}-\boldsymbol{\lambda}^*\left( \bm{w}^{((s-1)\tau)}\right)\right\|^2$, by the updating rule of $\boldsymbol{\lambda}$ and the $\mu$-strongly-concavity of $F(\bm{w} ,\cdot)$ we have:
\begin{align}
     \mathbb{E} \left\|\boldsymbol{\lambda}^{(s)} \right. & \left. -\boldsymbol{\lambda}^*\left( \bm{w}^{((s-1)\tau)}\right)\right\|^2  \nonumber\\ & \leq \mathbb{E}\left\|\boldsymbol{\lambda}^{(s-1)} - \boldsymbol{\lambda}^*\left( \bm{w}^{((s-1)\tau)}\right)\right\|^2 +\gamma^2 \left\|\nabla_{\boldsymbol{\lambda}} F\left(\bm{w}^{((s-1)\tau)}, \boldsymbol{\lambda}^{(s-1)}\right)\right\|^2  \nonumber\\
     &  \quad +2\gamma \left \langle \nabla_{\boldsymbol{\lambda}} F\left(\bm{w}^{((s-1)\tau)}, \boldsymbol{\lambda}^{(s-1)}\right), \boldsymbol{\lambda}^{(s-1)} - \boldsymbol{\lambda}^*\left( \bm{w}^{((s-1)\tau)}\right)\right \rangle  \nonumber\\
     &  \leq (1-\mu\gamma) \mathbb{E}\left\|\boldsymbol{\lambda}^{(s-1)} - \boldsymbol{\lambda}^*\left( \bm{w}^{((s-1)\tau)}\right)\right\|^2  \nonumber\\
     &  \quad + \underbrace{(2\gamma^2L - 2\gamma )}_{\leq 0} \left[F\left(\bm{w}^{((s-1)\tau)}, \boldsymbol{\lambda}^*\left( \bm{w}^{((s-1)\tau)}\right)\right)-F\left(\bm{w}^{((s-1)\tau)}, \boldsymbol{\lambda}^{(s-1)}\right) \right]\nonumber \\
     &   \leq \left(1-\frac{1}{\kappa}\right) \mathbb{E}\left\|\boldsymbol{\lambda}^{(s-1)} - \boldsymbol{\lambda}^*\left( \bm{w}^{((s-1)\tau)}\right)\right\|^2, \label{eq: lambda gap eq 2} 
\end{align}
where we used the smoothness property of $F(\bm{w} ,\cdot)$:
\begin{align}
     {\small \left\|\nabla_{\boldsymbol{\lambda}} F\left(\bm{w}^{((s-1)\tau)}, \boldsymbol{\lambda}^{(s-1)}\right)\right\|^2  \leq 2L \left(  F\left(\bm{w}^{((s-1)\tau)}, \boldsymbol{\lambda}^*\left( \bm{w}^{((s-1)\tau)}\right)\right) - F\left(\bm{w}^{((s-1)\tau)}, \boldsymbol{\lambda}^{(s-1)}\right) \right).}\nonumber
\end{align}
Plugging (\ref{eq: lambda gap eq 2}) into (\ref{eq: lambda gap eq 1}) yields:
\begin{align}
    \mathbb{E}\left\|\boldsymbol{\lambda}^{(s)} \right. & \left.-\boldsymbol{\lambda}^*( \bm{w}^{(s\tau)})\right\|^2 \nonumber \\
    & \leq \left(1+\frac{1}{2(\kappa-1)}\right)\left(1-\frac{1}{\kappa}\right) \mathbb{E}\left\|\boldsymbol{\lambda}^{(s-1)} - \boldsymbol{\lambda}^*\left( \bm{w}^{((s-1)\tau)}\right)\right\|^2  + 2 \kappa^3  \tau^2 \eta^2G_{w}^2 \nonumber\\
    & \leq \left(1-\frac{1}{2 \kappa }\right)  \mathbb{E}\left\|\boldsymbol{\lambda}^{(s-1)} - \boldsymbol{\lambda}^*\left( \bm{w}^{((s-1)\tau)}\right)\right\|^2  + 2 \kappa^3  \tau^2 \eta^2G_{w}^2.\nonumber
\end{align}
 
Applying the recursion on the above relation gives:
\begin{align}
    \mathbb{E}\|\boldsymbol{\lambda}^{(s)} -\boldsymbol{\lambda}^*( \bm{w}^{(s\tau)})\|^2   \leq \left(1-\frac{1}{2\kappa}\right)^s\mathbb{E}\left\|\boldsymbol{\lambda}^{{0}}- \boldsymbol{\lambda}^*(\bm{w}^{(0)})\right\|^2   + 4 \kappa^4  \tau^2 \eta^2G_{w}^2.\nonumber
\end{align}
Putting these pieces together  concludes the proof:
\begin{align}
     \mathbb{E}\left\|\boldsymbol{\lambda}^*( \bm{w}^{(t)}) - \boldsymbol{\lambda}^{(\floor{\frac{t}{\tau}})} \right\|^2\leq   2\left(1-\frac{1}{2\kappa}\right)^{\floor{\frac{t}{\tau}}}\mathbb{E}\left\|\boldsymbol{\lambda}_{{0}}- \boldsymbol{\lambda}^*(\bm{w} ^{(0)})\right\|^2   + 2(4\kappa^2+1)\kappa^2  \tau^2 \eta^2 G_{w}^2.\nonumber
\end{align}
\qed

\begin{lemma}
\label{lemma: sum of series}
For $\eta \mu \leq 1$, $\kappa >1$,$\tau\geq 1$, the following inequalities holds:
\begin{align}
    \sum_{t=0}^T \left(1-\frac{1}{2}\eta\mu\right)^t\left(1-\frac{1}{2\kappa}\right)^{\floor{\frac{t}{\tau}}} \leq  \frac{2\kappa\tau}{1- \frac{1}{2}\eta\mu}, \nonumber\\
    \sum_{t=0}^T \left(1-\frac{1}{4}\eta\mu\right)^t\left(1-\frac{1}{2\kappa}\right)^{\floor{\frac{t}{\tau}}} \leq  \frac{2\kappa\tau}{1- \frac{1}{4}\eta\mu}. \nonumber
\end{align}
\end{lemma}
\begin{proof}

\begin{align}
     \sum_{t=0}^T (1-\frac{1}{2}\eta\mu)^t(1-\frac{1}{2\kappa})^{\floor{\frac{t}{\tau}}} &=  \sum_{s=0}^{S-1}\sum_{t=1}^{\tau} (1-\frac{1}{2}\eta\mu)^{s\tau+t}(1-\frac{1}{2\kappa})^{s} \nonumber \\
     & \leq  \sum_{s=0}^{S-1}(1-\frac{1}{2\kappa})^{s} \sum_{t=1}^{\tau} \left(1-\frac{1}{2}\eta\mu\right)^{s\tau+t} \nonumber\\
    & \leq 2 \sum_{s=0}^{S-1}(1-\frac{1}{2\kappa})^{s} \frac{\left(1-\frac{1}{2}\eta\mu\right)^{s\tau}(1 - \left(1-\frac{1}{2}\eta\mu\right)^{\tau})}{\eta\mu} \nonumber\\
    & =  \frac{2(1 - \left(1-\frac{1}{2}\eta\mu\right)^{\tau})}{\eta\mu} \sum_{s=0}^{S-1}\left(1-\frac{1}{2\kappa}\right)^{s} \left(1-\frac{1}{2}\eta\mu\right)^{s\tau} \nonumber\\
    & \leq  \frac{2(1 - \left(1-\frac{1}{2}\eta\mu\right)^{\tau})}{\eta\mu} \sum_{s=0}^{S-1}\left(1-\frac{1}{2\kappa}\right)^{s} \left(1-\frac{1}{2}\eta\mu\right)^{s} \label{lemma: sum of series3}\\
    &\leq   \frac{2\tau \ln \frac{1}{\left(1-\frac{1}{2}\eta\mu\right)}}{\eta\mu} \frac{1}{1-\left(1-\frac{1}{2\kappa}\right)\left(1-\frac{1}{2}\eta\mu\right)} \label{lemma: sum of series2}\\
   &\leq   \frac{2\tau \ln \frac{1}{\left(1-\frac{1}{2}\eta\mu\right)}}{\left(\frac{\eta\mu}{2\kappa} +(\frac{1}{2}-\frac{1}{4\kappa})\eta^2\mu^2\right)}  \leq \frac{4\kappa \tau }{\eta\mu } \left(\frac{1}{1- \frac{1}{2}\eta\mu}-1\right) \label{lemma: sum of series0}\\
   & \leq \frac{2\kappa\tau}{\eta\mu } \left( \frac{\eta\mu}{1-\frac{1}{2} \eta\mu} \right) =  \frac{2\kappa\tau}{1- \frac{1}{2}\eta\mu},  \label{lemma: sum of series1}
\end{align}
where from (\ref{lemma: sum of series3}) to (\ref{lemma: sum of series2}) we use the inequality $1-a^x \leq x\ln\frac{1}{a}$, and from (\ref{lemma: sum of series0}) to (\ref{lemma: sum of series1}) we use the inequality $\ln x \leq x-1$. 

Similarly, for the second statement:
\begin{align}
     \sum_{t=0}^T (1-\frac{1}{4}\eta\mu)^t(1-\frac{1}{2\kappa})^{\floor{\frac{t}{\tau}}} &=  \sum_{s=0}^{S-1}\sum_{t=1}^{\tau} \left(1-\frac{1}{4}\eta\mu\right)^{s\tau+t}(1-\frac{1}{2\kappa})^{s} \nonumber\\
     & \leq  \sum_{s=0}^{S-1}(1-\frac{1}{2\kappa})^{s} \sum_{t=1}^{\tau} \left(1-\frac{1}{4}\eta\mu\right)^{s\tau+t} \nonumber\\
    & \leq 2 \sum_{s=0}^{S-1}(1-\frac{1}{2\kappa})^{s} \frac{\left(1-\frac{1}{4}\eta\mu\right)^{s\tau}(1 - \left(1-\frac{1}{4}\eta\mu\right)^{\tau})}{\eta\mu} \nonumber\\
    & =  \frac{2(1 - \left(1-\frac{1}{4}\eta\mu\right)^{\tau})}{\eta\mu} \sum_{s=0}^{S-1}\left(1-\frac{1}{2\kappa}\right)^{s} \left(1-\frac{1}{4}\eta\mu\right)^{s\tau} \nonumber\\
    & \leq  \frac{2(1 - \left(1-\frac{1}{4}\eta\mu\right)^{\tau})}{\eta\mu} \sum_{s=0}^{S-1}\left(1-\frac{1}{2\kappa}\right)^{s} \left(1-\frac{1}{4}\eta\mu\right)^{s}  \nonumber\\
    &\leq   \frac{2\tau \ln \frac{1}{\left(1-\frac{1}{4}\eta\mu\right)}}{\eta\mu} \frac{1}{1-\left(1-\frac{1}{2\kappa}\right)\left(1-\frac{1}{4}\eta\mu\right)}  \nonumber\\
   &\leq   \frac{2\tau \ln \frac{1}{\left(1-\frac{1}{4}\eta\mu\right)}}{\left(\frac{\eta\mu}{2\kappa} +(\frac{1}{4}-\frac{1}{8\kappa})\eta^2\mu^2\right)}  \leq \frac{4\kappa \tau }{\eta\mu } \left(\frac{1}{1- \frac{1}{4}\eta\mu}-1\right)\nonumber  \\
   & \leq \frac{2\kappa\tau}{\eta\mu } \left( \frac{\eta\mu}{1-\frac{1}{4} \eta\mu} \right) =  \frac{2\kappa\tau}{1- \frac{1}{4}\eta\mu}   \nonumber.
\end{align} 

\end{proof}

\subsection{Proof of Theorem~\ref{theorem2}}
Now we proceed to the proof of Theorem~\ref{theorem2}. According to Lemma~\ref{lemma: one iteration SCSC} we have:
\begin{align}
     \mathbb{E}\left\|\bm{w} ^{(t+1)}  - \bm{w} ^*\right\|^2  &\leq \left (1-\frac{\mu\eta}{2}\right) \mathbb{E}\left\|\bm{w} ^{(t)}- \bm{w} ^*\right\|^2 -\eta\mathbb{E}\left[\Phi( \bm{w}^{(t)})-\Phi(\bm{w} ^*)\right]+  \eta^2\frac{2\sigma_w^2+4G_w^2}{m}\nonumber\\
     & \quad  +  4L^2\left(\eta^2+\frac{\eta}{\mu}\right)  \mathbb{E}\left[\delta^{(t)} \right] +4\left(\frac{\eta}{\mu}+\eta^2\right)\mathbb{E}\left\|\nabla_{\bm{w} }F( \bm{w}^{(t)},\boldsymbol{\lambda}^{(\floor{\frac{t}{\tau}})})- \nabla\Phi( \bm{w}^{(t)})\right\|^2   \nonumber\\
     &\leq \left (1-\frac{\mu\eta}{2}\right) \mathbb{E}\left\|\bm{w} ^{(t)}- \bm{w} ^*\right\|^2 -\eta\mathbb{E}\left[\Phi( \bm{w}^{(t)})-\Phi(\bm{w} ^*)\right]+  \eta^2\frac{2\sigma_w^2+4G_w^2}{m}\nonumber\\
     & \quad + 4L^2\left(\eta^2+\frac{\eta}{\mu}\right)  \mathbb{E}\left[\delta^{(t)} \right] +4\left(\frac{\eta}{\mu}+\eta^2\right) L^2 \mathbb{E}\left\|  \boldsymbol{\lambda}^*(\bm{w} ^{(t)})-  \boldsymbol{\lambda}^{(\floor{\frac{t}{\tau}})}\right\|^2,\nonumber
\end{align}
where we use the smoothness of $F$ at the last step to substitute $\|\nabla_{\bm{w} }F(\bm{w} ^{(t)}, \boldsymbol{\lambda}^*(\bm{w} ^{(t)})) - \nabla_{\bm{w} }F(\bm{w} ^{(t)}, \boldsymbol{\lambda}^{(\floor{\frac{t}{\tau}})})\|^2$:
\begin{align}
     \left\|\nabla_{\bm{w} }F(\bm{w} ^{(t)}, \boldsymbol{\lambda}^*(\bm{w} ^{(t)})) - \nabla_{\bm{w} }F(\bm{w} ^{(t)}, \boldsymbol{\lambda}^{(\floor{\frac{t}{\tau}})})\right\|^2  \leq  L^2 \left\|  \boldsymbol{\lambda}^*(\bm{w} ^{(t)})-  \boldsymbol{\lambda}^{(\floor{\frac{t}{\tau}})}\right\|^2.\nonumber
\end{align}
Then plugging in Lemma~\ref{lemma: optimal gap} yields:
\begin{align}
    \mathbb{E}\|\bm{w} ^{(t+1)}  - \bm{w} ^*\|^2 
    &\leq \left (1-\frac{\mu\eta}{2}\right) \mathbb{E}\left\|\bm{w} ^{(t)}- \bm{w} ^*\right\|^2 -\eta\mathbb{E}\left[\Phi( \bm{w}^{(t)})-\Phi(\bm{w} ^*)\right]+  \eta^2\frac{2\sigma_w^2+4G_w^2}{m}\nonumber\\  
     &\quad + 4L^2\left(\eta^2+\frac{\eta}{\mu}\right)  \mathbb{E}\left[\delta^{(t)} \right] \nonumber\\
     & \quad +8\left(\frac{\eta}{\mu}+\eta^2\right)L^2 \left(\left(1-\frac{1}{2\kappa}\right)^{\floor{\frac{t}{\tau}}}\mathbb{E}\|\boldsymbol{\lambda}^{(0)}- \boldsymbol{\lambda}^*(\bm{w} ^{(0)})\|^2  +  \kappa^2\tau^2\eta^2 G_{w}^2 \left(4\kappa^2 +1 \right)\right). \label{mainrecursion} 
\end{align} 

Unrolling the recursion yields:
\begin{align}
    &\mathbb{E}\|\bm{w} ^{(T)} - \bm{w} ^*\|^2 \nonumber\\
    & \leq \left(1-\frac{1}{2}\mu \eta\right)^T\mathbb{E}\|\bm{w} ^{(0)}- \bm{w} ^*\|^2  +\sum_{t=1}^{T}\left(1-\frac{1}{2}\mu \eta\right)^t \left[8L^2\kappa^2\tau^2\eta^2 G_{w}^2\left(\frac{\eta}{\mu}+\eta^2\right) \left(4\kappa^2 +1 \right)\right] \nonumber\\
    & \quad  + \sum_{t=1}^{T}\left(1-\frac{1}{2}\mu \eta\right)^t \left[\eta^2\frac{2\sigma_w^2+4G_w^2}{m}+4L^2\left(\eta^2+\frac{\eta}{\mu}\right) \mathbb{E}\left[\delta^{(t)} \right] \right] \nonumber\\
     &\quad + 8\left(\frac{\eta}{\mu}+\eta^2\right)L^2\mathbb{E}\|\boldsymbol{\lambda}^{(0)}- \boldsymbol{\lambda}^*(\bm{w} ^{(0)})\|^2 \sum_{t=1}^{T}\left(1-\frac{1}{2}\mu \eta\right)^t\left(1-\frac{1}{2\kappa}\right)^{\floor{\frac{t}{\tau}}}  \label{t21}\\
       & \leq \exp\left(-\frac{1}{2}\mu \eta T\right)D_{\mathcal{W}}^2+ \eta \frac{ 4\sigma_w^2+8G_w^2   }{\mu m}+8L^2\left(\frac{\eta}{\mu}+\frac{1}{\mu^2}\right)\sum_{t=0}^{T}\mathbb{E}\left[\delta^{(t)} \right] \nonumber\\
    &\quad + 16L^2\kappa^2\tau^2\eta^2 G_{w}^2\left(\frac{\eta}{\mu}+\frac{1}{\mu^2}\right) \left( 4\kappa^2 +1 \right)  + 16L^2   \left( \frac{\kappa\tau }{1- \frac{1}{2}\eta\mu} \right)\left(\frac{\eta}{\mu}+\eta^2\right)D_{\Lambda}^2, \label{t20}
\end{align}
where we used the result from  Lemma~\ref{lemma: sum of series} from (\ref{t21}) to (\ref{t20}). Now, we simplify (\ref{mainrecursion}) by applying the telescoping sum on (\ref{mainrecursion}) for $t=\frac{T}{2}$ to $T$:
\begin{align}
    &\frac{2}{T}\sum_{t=T/2}^T\mathbb{E}\left[\Phi(\bm{w} ^{(t)})-\Phi(\bm{w} ^*)\right] \nonumber\\
    & \leq \frac{2}{\eta T}\mathbb{E}\| \bm{w}^{(T/2)}- \bm{w}^*\|^2  +  \eta \frac{ 2\sigma_w^2+4G_w^2}{m} +  4L^2\left(\eta +\frac{1}{\mu}\right) \frac{2}{T}\sum_{t=T/2}^T \mathbb{E}\left[\delta^{(t)} \right] \nonumber\\
     & \quad  +8\left(\frac{1}{\mu}+\eta\right)L^2D_{\Lambda}^2  \frac{2}{T}\sum_{t=T/2}^T\left(1-\frac{1}{2\kappa}\right)^{\floor{\frac{t}{\tau}}}   +8\left(\frac{1}{\mu}+\eta\right) \kappa^2\tau^2\eta^2 L^2G_{w}^2 \left(4\kappa^2 +1 \right) \nonumber\\
     &  \leq \frac{2}{\eta T}\mathbb{E}\| \bm{w}^{(T/2)}- \bm{w} ^*\|^2  +  \eta \frac{ 2\sigma_w^2+4G_w^2}{m} +  80\eta^2\tau^2L^2\left(\eta +\frac{1}{\mu}\right)  \left(\sigma_w^2+\frac{\sigma_w^2}{m} + \Gamma \right)\nonumber\\
     & \quad +16\left(\frac{1}{\mu}+\eta\right)L^2  O\left(\frac{ \tau \exp(-\mu\eta T/4  \tau) }{T}D_{\Lambda}^2 \right) +8\left(\frac{1}{\mu}+\eta\right) \kappa^2\tau^2\eta^2 L^2G_{w}^2 \left(4\kappa^2 +1 \right) \nonumber\\
     &  \leq \frac{2}{\eta T}\mathbb{E}\| \bm{w}^{(T/2)}- \bm{w} ^*\|^2  +  \eta \frac{ 2\sigma_w^2+4G_w^2}{m} +  80\eta^2\tau^2L^2\left(\eta +\frac{1}{\mu}\right)  \left(\sigma_w^2+\frac{\sigma_w^2}{m} + \Gamma \right)\nonumber\\
     &  \quad  +16\left(\frac{1}{\mu}+\eta\right)L^2  O\left(\frac{ \tau \exp(-\mu\eta T/4  \tau) }{T}D_{\Lambda}^2 \right) +8\left(\frac{1}{\mu}+\eta\right) \kappa^2\tau^2\eta^2 L^2G_{w}^2 \left(4\kappa^2 +1 \right).\nonumber
\end{align}
Plugging in (\ref{t20}) yields:
\begin{align}
    \frac{2}{T} & \sum_{t=T/2}^T \mathbb{E}\left[\Phi(\bm{w} ^{(t)})-\Phi(\bm{w} ^*)\right] \nonumber\\
     &  \leq \frac{2}{\eta T}\left( \exp\left(-\frac{1}{4}\mu \eta T\right)D_{\mathcal{W}}^2+ \eta \frac{ 4\sigma_w^2+8G_w^2   }{\mu m}+8L^2\left(\frac{\eta}{\mu}+\frac{1}{\mu^2}\right)\sum_{t=0}^{T}\mathbb{E}\left[\delta^{(t)} \right]\right)\nonumber \\
     & \quad+ \frac{2}{\eta T}\left( 16L^2\kappa^2\tau^2\eta^2 G_{w}^2\left(\frac{\eta}{\mu}+\frac{1}{\mu^2}\right) \left( 4\kappa^2 +1 \right)  +16L^2  \left( \frac{\kappa\tau }{1- \frac{1}{2}\eta\mu} \right)\left(\frac{\eta}{\mu}+\eta^2\right)D_{\Lambda}^2\right) \nonumber\\
     &\quad+  \eta \frac{ 2\sigma_w^2+4G_w^2}{m} +  80\eta^2\tau^2L^2\left(\eta +\frac{1}{\mu}\right)  \left(\sigma_w^2+\frac{\sigma_w^2}{m} + \Gamma \right)\nonumber\\
     & \quad +16\left(\frac{1}{\mu}+\eta\right)L^2  O\left(\frac{ \tau \exp(-\mu\eta T/4  \tau) }{T}D_{\Lambda}^2 \right) +8\left(\frac{1}{\mu}+\eta\right) \kappa^2\tau^2\eta^2 L^2G_{w}^2 \left(4\kappa^2 +1 \right).\nonumber
\end{align}
Combining the terms yields:
\begin{align}
    &\frac{2}{T}\sum_{t=T/2}^T\mathbb{E}\left[\Phi(\bm{w} ^{(t)})-\Phi(\bm{w} ^*)\right] \nonumber\\
    & \leq \frac{2}{\eta T} \exp\left(-\frac{1}{4}\mu \eta T\right)D_\mathcal{W}^2 + 16\left(\frac{1}{\mu}+\eta\right)L^2  O\left(\frac{ \tau \exp(-\mu\eta T/4  \tau) }{T}D_{\Lambda}^2 \right)\nonumber\\
    & \quad+ \left(\frac{4}{\mu T}+ \eta\right)\frac{ 2\sigma_w^2+4G_w^2 }{m} + \left( 1 +\frac{2}{\mu \eta T}\right)80\eta^2\tau^2L^2\left(\eta +\frac{1}{\mu}\right)\left(\sigma_w^2+\frac{\sigma_w^2}{m} + \Gamma \right)\nonumber\\
    & \quad+ \left( \frac{4}{\mu\eta T} + 1\right)8L^2\kappa^2\tau^2\eta^2 G_{w}^2\left( 4\kappa^2 +1 \right) \left( \eta +\frac{1}{\mu }\right)\nonumber \\
    & \quad+ \frac{32L^2}{T} \left( \frac{\kappa\tau }{1- \frac{1}{2}\eta\mu} \right)\left(\frac{1}{\mu}+\eta \right)D_{\Lambda}^2.\nonumber
\end{align}

And finally, plugging in $\eta = \frac{4 \log T}{\mu T}$ and using the fact that $\Phi(\frac{2}{T}\sum_{t=T/2}^T \bm{w} ^{(t)}) \leq \frac{2}{T}\sum_{t=T/2}^T \Phi(\bm{w} ^{(t)})$ yields:
\begin{align}
    &\mathbb{E}[\Phi(\hat{\bm{w} })-\Phi(\bm{w} ^*)]\nonumber\\  
    &\leq \frac{\mu D_{\mathcal{W}}^2}{2 T\log T}   + 16\left(\frac{1}{\mu}+ \frac{4 \log T}{\mu T}\right)L^2  O\left(\frac{ \tau  }{T^{(1+1/\tau)}}D_{\Lambda}^2 \right)\nonumber\\
    & \quad + \left(\frac{4}{\mu T}+ \frac{4 \log T}{\mu T}\right)\frac{ 2\sigma_w^2+4G_w^2 }{m} + \left( 1 +\frac{2}{\mu \eta T}\right) \frac{1280\kappa^2 \tau^2 \log^2 T }{ T^2} \left(\eta +\frac{1}{\mu}\right)\left(\sigma_w^2+\frac{\sigma_w^2}{m} + \Gamma \right)\nonumber\\
    & \quad + \left( \frac{1}{\log T} + 1\right)  \frac{ 8\kappa^4 \tau^2 \log^2 T }{ T^2} G_{w}^2\left( 4\kappa^2 +1 \right) \left( \frac{4 \log T}{\mu T} +\frac{1}{\mu }\right)\nonumber \\
    & \quad + \frac{32L^2}{T} \left( \frac{\kappa\tau }{1- \frac{2 \log T}{  T}} \right)\left(\frac{1}{\mu}+\frac{4 \log T}{\mu T} \right)D_{\Lambda}^2  \nonumber\\
    &\leq \Tilde{O}\left( \frac{\mu D_{\mathcal{W}}^2 }{  T } \right) + O\left(\frac{\kappa L \tau D_{\Lambda}^2 }{T^{(1+1/\tau)}} \right)+ \Tilde{O}\left(\frac{  \sigma_w^2+ G_w^2 }{\mu mT}\right) + O\left(   \frac{\kappa^2\tau^2(\sigma_w^2 + \Gamma)}{\mu T^2}\right)  \nonumber\\
    & \quad + \Tilde{O}\left(\frac{\kappa^2L\tau D_{\Lambda}^2 }{T}\right)+ \Tilde{O}\left(\frac{\kappa^6\tau^2G_w^2}{\mu T^2}\right).\nonumber
\end{align}
\qed

\section{Proof of Convergence of {\sffamily{DRFA-GA}} in Nonconvex (PL Condition)-Strongly-Concave Setting}\label{sec:ncsc}
\subsection{Overview of Proofs}
In this section we will present formal proofs in nonconvex (PL condition)-strongly-concave setting (Theorem~\ref{thm3}). The main idea is similar to strongly-convex-strongly-concave case: we start from one iteration analysis, and plug in the upper bound of $\delta^{(t)}$ and $\|\nabla_{\bm{w} }F( \bm{w}^{(t)},\boldsymbol{\lambda}^{(\floor{\frac{t}{\tau}})})- \nabla\Phi( \bm{w}^{(t)})\|^2$.

However, a careful analysis need to be employed in order to deal with projected SGD in constrained nonconvex optimization problem. We employ the technique used in~\cite{ghadimi2016mini}, where they advocate to study the following quantity:
\begin{align}
   P_{\mathcal{W}}(\bm{w}, \bm{g},\eta) = \frac{1}{\eta}\left[\bm{w}-\prod_{\mathcal{W}} \left (   \bm{w}- \eta \bm{g}  \right)\right]\nonumber. 
\end{align}

If we plug in $\bm{w} = \bm{w}^{(t)}$, $\bm{g} = \bm{u}^{(t)} = \frac{1}{m}\sum_{i\in \mathcal{D}^{(\floor{\frac{t}{\tau}})}} \nabla  f_i ( \bm{w}_i^{(t)};\xi_i^t)$, then  
\begin{align}
   P_{\mathcal{W}}(\bm{w}^{(t)}, \bm{u}^{(t)},\eta) = \frac{1}{\eta} \left[\bm{w}^{(t)} - \prod_{\mathcal{W}}\left (\bm{w}^{(t)} - \eta \bm{u}^{(t)}  \right)\right]\nonumber.
\end{align}
characterize the difference between iterates $\bm{w}^{(t+1)}$ and $\bm{w}^{(t)}$. A trivial property of operator $ P_{\mathcal{W}}$ is contraction mapping, which follows the property of projection:
\begin{align}
     \left\|P_{\mathcal{W}}(\bm{w},\bm{g}_1,\eta) - P_{\mathcal{W}}(\bm{w}, \bm{g}_2,\eta)\right\|^2 \leq \left\|\bm{g}_1 - \bm{g}_2\right\|^2.\nonumber
\end{align}

The significant property of operator $P_{\mathcal{W}}$ is given by the following lemma:
\begin{lemma}[Property of Projection, \cite{ghadimi2016mini} Lemma 1]\label{lemma: projection1}
For all $\bm{w} \in \mathcal{W} \subset \mathbb{R}^d$, $\bm{g} \in \mathbb{R}^d$ and $\eta > 0$, we have:
\begin{align}
     \left \langle \bm{g}, P_{\mathcal{W}}(\bm{w}, \bm{g},\eta) \right\rangle \geq \left\|P_{\mathcal{W}}(\bm{w}, \bm{g},\eta)\right\|^2.\nonumber
\end{align}
\end{lemma}

The above lemma establishes a lower bound for the inner product $ \left \langle \bm{g}, P_{\mathcal{W}}(\bm{y},\bm{g},\eta) \right\rangle$, and  will play a significant role in our analysis.

\subsection{Proof of Technical Lemmas}
\begin{lemma}
If $F(\cdot,\boldsymbol{\lambda})$ satisfies $\mu$-generalized PL condition, then $\Phi(\cdot)$ also satisfies $\mu$-generalized PL condition.
\end{lemma}
\begin{proof}
 Let $\bm{w}^* \in \arg \min_{\bm{w}\in \mathcal{W}} \Phi(\bm{w})$. Since $F(\cdot,\boldsymbol{\lambda})$ satisfies $\mu$-generalized PL condition, we have for any $\bm{w}\in\mathcal{W}$:
\begin{align}
    \frac{1}{2\eta^2}\left\|\bm{w} -\prod_{\mathcal{W}}\left( \bm{w}-\eta\nabla_{\bm{w} } F(\bm{w} ,\boldsymbol{\lambda}^*(\bm{w} ))\right) \right \|^2 & \geq \mu(F(\bm{w} ,\boldsymbol{\lambda}^*(\bm{w} ) -\min_{\bm{w}'\in\mathcal{W}} F(\bm{w}',\boldsymbol{\lambda}^*(\bm{w} )) \nonumber \\
    &\geq \mu(F(\bm{w} ,\boldsymbol{\lambda}^*(\bm{w} ) - F(\bm{w} ^*,\boldsymbol{\lambda}^*(\bm{w} )) \nonumber\\ &\geq \mu(F(\bm{w} ,\boldsymbol{\lambda}^*(\bm{w} ) - F(\bm{w}^*,\boldsymbol{\lambda}^*(\bm{w}^*)) .\nonumber
\end{align}
which immediately implies $\frac{1}{2\eta^2}\|\bm{w} -\prod_{\mathcal{W}}\left( \bm{w}-\eta \nabla \Phi(\bm{w} )\right)\|^2 \geq \mu( \Phi(\bm{w} ) -  \Phi(\bm{w} ^*))$ as desired.
\end{proof}

\begin{lemma}
\label{lemma: one iteration PLSC}
For {\sffamily{DRFA-GA}}, under Theorem~\ref{thm3}'s assumptions, we have:
\begin{align}
    \mathbb{E}\left[\Phi(\bm{w} ^{(t+1)}) - \Phi(\bm{w} ^{*}) \right] &\leq \left(1-\frac{\mu \eta}{4}\right)  \mathbb{E}\left[\Phi(\bm{w} ^{(t)}) - \Phi(\bm{w} ^{*}) \right]  \nonumber \\
    & \quad +  \frac{3\eta}{2 }  \mathbb{E}\left\|\sum_{i=1}^N \lambda^{(\floor{\frac{t}{\tau}})}_i \nabla f_i(\bm{w} ^{(t)}_i) - \nabla \Phi ( \bm{w}^{(t)})  \right\|^2 + 3\eta\frac{2\sigma_w^2+4G_w^2}{2m},
\end{align}
where $\alpha=L + \kappa L$
\end{lemma}
\begin{proof}
 
Define the following quantities:
 \begin{align}
     &  \bm{u}_t = \frac{1}{m}\sum_{i\in D^t} \nabla f_i(\bm{w}_i^{(t)};\xi_i^t),   \bar{\bm{u}}_t = \sum_{i=1}^N \lambda_i^{(\floor{\frac{t}{\tau}})} \nabla f_i(\bm{w}_i^{(t)}).\nonumber \\
     &\Tilde{R}^{(t)} = P_{\mathcal{W}}(\bm{w}^t, \bm{u}_t ,\eta) =\bm{w}^{(t)} -\frac{1}{\eta} \prod_{\mathcal{W}} \left (  \bm{w}^{(t)}-\eta  \bm{u}_t \right)  \nonumber\\
   & {R}^{(t)} = P_{\mathcal{W}}(\bm{w}^t, \bar{\bm{u}}_t,\eta) =\bm{w}^{(t)}- \frac{1}{\eta}\prod_{\mathcal{W}} \left ( \bm{w}^{(t)}- \eta \bar{\bm{u}}_t  \right) \nonumber\\ 
     &\hat{R}^{(t)} = P_{\mathcal{W}}(\bm{w}^t,\Phi(\bm{w}^{(t)}) ,\eta) =\bm{w}^{(t)}- \frac{1}{\eta}\prod_{\mathcal{W}} \left ( \bm{w}^{(t)}- \eta \nabla  \Phi(\bm{w}^{(t)})  \right) \nonumber.
 \end{align}

 By the $\alpha$-smoothness of $\Phi$ and the updating rule of $\bm{w}$ we have:
 
 \begin{align}
     \mathbb{E}[\Phi( \bm{w}^{(t+1)} )] - \mathbb{E}[\Phi(\bm{w}^{(t)} )] &\leq \frac{  \alpha}{2}\mathbb{E}\left[\left\|\bm{w}^{(t+1)}-\bm{w}^{(t)}\right\|^2\right] + \left \langle \nabla \Phi(\bm{w}^{(t)} ), \bm{w}^{(t+1)}-\bm{w}^{(t)}  \right\rangle \nonumber\\
      &\leq \frac{\eta^2 \alpha}{2}\mathbb{E}\left[\left\| \Tilde{R}^{(t)}\right\|^2\right]  - \eta \mathbb{E}\left[\left \langle  \nabla \Phi(\bm{w}^{(t)} ), \Tilde{R}^{(t)} \right\rangle\right] \nonumber\\ 
      &\leq \frac{\eta^2 \alpha}{2}\mathbb{E}\left[\left\| \Tilde{R}^{(t)}\right\|^2\right]- \eta \mathbb{E}\left[\left \langle \bm{u}_t, P_{\mathcal{W}}(\bm{y}^t,  \bm{u}_t,\eta) \right\rangle\right] \nonumber\\
      & \quad - \eta\mathbb{E}\left[ \left \langle \nabla \Phi(\bm{w}^{(t)} )-\bm{u}_t, \Tilde{R}^{(t)} \right\rangle\right] \nonumber.
 \end{align}
 According to Lemma~\ref{lemma: projection1}, we can bound the first dot product term in the last inequality by $\|\tilde{R}^{(t)}\|^2$, so then we have:
 
  \begin{align}
    \mathbb{E}[\Phi( \bm{w}^{(t+1)} )] &- \mathbb{E}[\Phi(\bm{w}^{(t)} )] \nonumber\\
      &\leq \frac{\eta^2 \alpha}{2}\mathbb{E}\left[\left\| \Tilde{R}^{(t)}\right\|^2 \right] - \eta \mathbb{E}\left[\|\Tilde{R}^{(t)}\|^2\right] - \eta \mathbb{E}\left[\left \langle \nabla \Phi(\bm{w}^{(t)} )-\bm{u}_t, \Tilde{R}^{(t)} \right\rangle \right] \nonumber\\
      &\leq -\left(\eta-\frac{\eta^2 \alpha}{2}\right)\mathbb{E}\left[\left\| \Tilde{R}^{(t)}\right\|^2 \right] - \eta \mathbb{E}\left[\left \langle \nabla \Phi(\bm{w}^{(t)} )-\bm{u}_t, \Tilde{R}^{(t)} \right\rangle \right] \nonumber\\
      &\leq -\left(\eta-\frac{\eta^2 \alpha}{2}\right)\mathbb{E}\left[\left\| \Tilde{R}^{(t)}\right\|^2 \right] + \frac{\eta}{2} \mathbb{E}\left[\left \| \nabla \Phi(\bm{w}^{(t)} )-\bm{u}_t\right\|^2+ \left\| \Tilde{R}^{(t)} \right\|^2 \right] \nonumber\\ 
       &\leq \underbrace{-\left(\frac{\eta}{2}-\frac{\eta^2 \alpha}{2}\right)}_{\leq -\frac{1}{4}\eta}\mathbb{E}\left[\left\| \Tilde{R}^{(t)}\right\|^2 \right] +  \eta  \mathbb{E}\left[\left \| \nabla \Phi(\bm{w}^{(t)} )-\bar{\bm{u}}_t\right\|^2+ \left \|\bar{\bm{u}}_t- \bm{u}_t\right\|^2 \right] \nonumber\\
       &\leq -\frac{1}{4}\eta\mathbb{E}\left[\left\| \Tilde{R}^{(t)}\right\|^2 \right] +  \eta  \mathbb{E}\left[\left \| \nabla \Phi(\bm{w}^{(t)} )-\bar{\bm{u}}_t\right\|^2\right] + \frac{\eta (2\sigma_w^2+4G_w^2)}{m} 
       \label{eq: ncnc lm2 eq1}.
 \end{align}
 Notice that:
 \begin{align}
     \mathbb{E}\left[\left\| \hat{R}^{(t)}\right\|^2 \right] &\leq 2\mathbb{E}\left[\left\| \Tilde{R}^{(t)}\right\|^2 \right] + 2\mathbb{E}\left[\left\| \hat{R}^{(t)}- \Tilde{R}^{(t)}\right\|^2 \right]\nonumber\\
     &\leq 2\mathbb{E}\left[\left\| \Tilde{R}^{(t)}\right\|^2 \right] + 4\mathbb{E}\left[\left\| \hat{R}^{(t)}-  {R}^{(t)}\right\|^2 \right]+ 4\mathbb{E}\left[\left\|  {R}^{(t)}- \Tilde{R}^{(t)}\right\|^2 \right]\nonumber\\
     &\leq 2\mathbb{E}\left[\left\| \Tilde{R}^{(t)}\right\|^2 \right] + 4\mathbb{E}\left[\left\| \hat{R}^{(t)}-  {R}^{(t)}\right\|^2 \right]+ 4\mathbb{E}\left[\left\|  \bm{u}^{(t)}- \bar{\bm{u}}^{(t)}\right\|^2 \right]\nonumber\\
     &\leq 2\mathbb{E}\left[\left\| \Tilde{R}^{(t)}\right\|^2 \right] + 4   \mathbb{E}\left[\left \| \nabla \Phi(\bm{w}^{(t)} )-\bar{\bm{u}}_t\right\|^2\right]+  \frac{4\eta (2\sigma_w^2+4G_w^2)}{m} \label{eq: ncnc lm2 eq2}.
 \end{align}
 Thus, plugging (\ref{eq: ncnc lm2 eq2}) into (\ref{eq: ncnc lm2 eq1}) to substitute $\mathbb{E}\left[\left\| \Tilde{R}^{(t)}\right\|^2 \right]$ yields:
   \begin{align}
     \mathbb{E}[\Phi( \bm{w}^{(t+1)} )] &- \mathbb{E}[\Phi(\bm{w}^{(t)} )] \nonumber\\
       &\leq -\frac{1}{8}\eta  \mathbb{E}\left[\left\| \hat{R}^{(t)}\right\|^2\right] + \frac{1}{2}\eta  \mathbb{E}\left[\left \| \nabla \Phi(\bm{w}^{(t)} )-\bar{\bm{u}}_t\right\|^2\right] + \frac{\eta (2\sigma_w^2+4G_w^2)}{2m}   \nonumber\\
       & \quad +\eta  \mathbb{E}\left[\left \| \nabla \Phi(\bm{w}^{(t)} )-\bar{\bm{u}}_t\right\|^2\right] + \frac{\eta (2\sigma_w^2+4G_w^2)}{m} \nonumber\\
       &\leq -\frac{1}{8}\eta  \mathbb{E}\left[\left\| \hat{R}^{(t)}\right\|^2\right] + \frac{3}{2}\eta  \mathbb{E}\left[\left \| \nabla \Phi(\bm{w}^{(t)} )-\bar{\bm{u}}_t\right\|^2\right] + \frac{3\eta (2\sigma_w^2+4G_w^2)}{2m}  
       \label{eq: ncnc lm2 eq3}.
 \end{align}
 
 Plugging in the generalized PL-condition:
 \begin{align}
     \frac{1}{\eta^2}\mathbb{E}\left[\left\|\prod_{\mathcal{W}}\left(\bm{w}^{(t)}- \eta\nabla \Phi(\bm{w}^{(t)}) \right) - \bm{w}^{(t)}\right\|^2\right] =  \mathbb{E}\left[\left\| \hat{R}^{(t)}\right\|^2\right] \geq 2\mu \left( \mathbb{E}[\Phi( \bm{w}^{t}  )] - \mathbb{E}[\Phi( \bm{w}^{*}  )] \right)\nonumber
 \end{align}  into (\ref{eq: ncnc lm2 eq3}) yields: 
\begin{align}
     \mathbb{E}\left[\Phi(\bm{w} ^{(t+1)}) - \Phi(\bm{w} ^{*}) \right] &\leq \left(1-\frac{\mu \eta}{4}\right)  \mathbb{E}\left[\Phi(\bm{w} ^{(t)}) - \Phi(\bm{w} ^{*}) \right]  \nonumber \\
    & \quad +  \frac{3\eta}{2 }  \mathbb{E}\left\|\sum_{i=1}^N \lambda^{(\floor{\frac{t}{\tau}})}_i \nabla f_i(\bm{w} ^{(t)}_i) - \nabla \Phi ( \bm{w}^{(t)})  \right\|^2 + 3\eta\frac{2\sigma_w^2+4G_w^2}{2m}.  \nonumber
\end{align}
 
\end{proof}

\subsection{Proof for Theorem~\ref{thm3}}
Now we proceed to the proof of Theorem~\ref{thm3}.  According to Lemma~\ref{lemma: one iteration PLSC} we have:
\begin{align}
    \mathbb{E}\left[\Phi(\bm{w} ^{(t+1)}) - \Phi(\bm{w} ^{*}) \right] &\leq \left(1-\frac{\mu \eta}{4}\right)  \mathbb{E}\left[\Phi(\bm{w} ^{(t)}) - \Phi(\bm{w} ^{*}) \right]  \nonumber \\
    & \quad +  \frac{3\eta}{2 }  \underbrace{\mathbb{E}\left\|\sum_{i=1}^N \lambda^{(\floor{\frac{t}{\tau}})}_i \nabla f_i(\bm{w} ^{(t)}_i) - \nabla \Phi ( \bm{w}^{(t)})  \right\|^2}_{T_1} + 3\eta\frac{2\sigma_w^2+4G_w^2}{2m}.\nonumber
\end{align}
Now, we  bound the term $T_1$ in above as:
\begin{align}
    T_1 &\leq 2\mathbb{E}\left\| \nabla_{\bm{w} }\Phi(\bm{w} ^{(t)}) -\sum_{i=1}^N\lambda^{(\floor{\frac{t}{\tau}})}_i\nabla_{\bm{w} }f_i(\bm{w} ^{(t)}) \right\|^2\nonumber\\
     & \quad + 2\mathbb{E}\left\|\sum_{i=1}^N\lambda^{(\floor{\frac{t}{\tau}})}_i\nabla_{\bm{w} }f_i(\bm{w} ^{(t)}) -\sum_{i=1}^N\lambda^{(\floor{\frac{t}{\tau}})}_i\nabla f_i(\bm{w} ^{(t)}_i) \right\|^2 \nonumber \\
    &\leq 2\mathbb{E}\left\| \nabla_{\bm{w} }F(\bm{w} ^{(t)}, \boldsymbol{\lambda}^*(\bm{w} ^{(t)})) - \nabla_{\bm{w} }F(\bm{w} ^{(t)}, \boldsymbol{\lambda}^{(\floor{\frac{t}{\tau}})}) \right\|^2   \nonumber \\
    & \quad + 2\sum_{i=1}^N\lambda^{(\floor{\frac{t}{\tau}})}_i\mathbb{E}\left\|\nabla_{\bm{w} }f_i(\bm{w} ^{(t)}) -\nabla f_i(\bm{w} ^{(t)}_i) \right\|^2 \nonumber\\
    &\leq 2 L^2 \mathbb{E}\left\|  \boldsymbol{\lambda}^*(\bm{w} ^{(t)})-  \boldsymbol{\lambda}^{(\floor{\frac{t}{\tau}})}) \right\|^2 + 2 L^2\mathbb{E}\left[\delta^{(t)}\right]\nonumber\\
    & \leq 2L^2\left(2\left(1-\frac{1}{2\kappa}\right)^{\floor{\frac{t}{\tau}}}\mathbb{E}\|\boldsymbol{\lambda}^{(0)}- \boldsymbol{\lambda}^*(\bm{w} ^{(0)})\|^2  +  2\kappa^2\tau^2\eta^2 G_{w}^2 \left(4\kappa^2 +1 \right)\right) + 2 L^2\mathbb{E}\left[\delta^{(t)}\right],  \nonumber
\end{align}
where we plug in the Lemma~\ref{lemma: optimal gap}.
Plugging $T_1$ back yields:
\begin{align}
     &\mathbb{E}\left[\Phi(\bm{w}^{(t+1)}) - \Phi( \bm{w}  ^*)\right] \nonumber\\
    & \leq \left( 1- \frac{1}{4}\mu\eta\right)\mathbb{E}\left[\Phi(\bm{w}^{(t)}) - \Phi( \bm{w}^*)\right] + 3 \eta \frac{2\sigma_w^2+4G_w^2}{2m}   \nonumber\\
    & \quad  + \frac{3\eta}{2}\left(4L^2\left(1-\frac{1}{2\kappa}\right)^{\floor{\frac{t}{\tau}}} \mathbb{E}\|\boldsymbol{\lambda}^*(\bm{w}^{(0)}) -  \boldsymbol{\lambda}^{(0)} \|^2+4L^2\kappa^2\tau^2\eta^2 G_{w}^2 \left(4\kappa^2 +1 \right)+2 L^2\mathbb{E}\left[\delta^{(t)}\right]\right)  \nonumber\\
     & \leq \left( 1- \frac{1}{4}\mu\eta\right)\mathbb{E}\left[\Phi(\bm{w}^{(t)}) - \Phi(\bm{w}^*)\right] + 3\eta \frac{2\sigma_w^2+4G_w^2}{2m}   \nonumber\\
    & \quad   + 6\eta L^2 \left(\left(1-\frac{1}{2\kappa}\right)^{\floor{\frac{t}{\tau}}} \mathbb{E}\|\boldsymbol{\lambda}^*( \bm{w} ^{(0)}) -  \boldsymbol{\lambda}^{(0)} \|^2\right)  \nonumber\\
    & \quad +  \frac{3\eta}{2}\left(4L^2\kappa^2\tau^2\eta^2 G_{w}^2 \left(4\kappa^2 +1 \right)+2 L^2\mathbb{E}\left[\delta^{(t)}\right]\right). \nonumber
\end{align}
Unrolling the recursion yields
\begin{align}
     &\mathbb{E}\left[\Phi(\bm{w} ^{(T)}) - \Phi( \bm{w} ^*)\right] \nonumber\\
     & \leq \left( 1-\frac{1}{4} \mu\eta\right)^T\mathbb{E}\left[\Phi(\bm{w} ^{(0)}) - \Phi( \bm{w} ^*) \right] + \sum_{t=0}^{T} \left( 1-\frac{1}{4} \mu\eta\right)^t 3 \eta \frac{2\sigma_w^2+4G_w^2}{2m}   \nonumber\\
    & \quad  + 6\eta L^2 \mathbb{E}\|\boldsymbol{\lambda}^*( \bm{w}^{(0)}) -  \boldsymbol{\lambda}_0 \|^2\sum_{t=0}^{T} \left [\left( 1- \frac{1}{2}\mu\eta\right)^t \left(1-\frac{1}{2\kappa}\right)^{\floor{\frac{t}{\tau}}} \right] \nonumber\\
    & \quad + \frac{3}{2}\eta\left(\sum_{t=0}^{T}\left( 1- \frac{1}{4}\mu\eta\right)^t 4L^2\kappa^2\tau^2\eta^2 G_{w}^2 \left(4\kappa^2 +1 \right)+2 L^2\sum_{t=0}^{T}\left( 1- \frac{1}{4}\mu\eta\right)^t \mathbb{E}\left[\delta^{(t)}\right]  \right) \nonumber \\
    & \leq \exp \left( - \frac{\mu\eta T}{4}\right)\mathbb{E}\left[\Phi(\bm{w}^{(0)}) - \Phi( \bm{w} ^*)\right] + 12 \frac{2\sigma_w^2+4G_w^2}{2\mu m}  \nonumber\\
    & \quad  + 6\eta  L^2 \mathbb{E}\|\boldsymbol{\lambda}^*( \bm{w}^{(0)}) -  \boldsymbol{\lambda}^{(0)} \|^2 \left(\frac{2\kappa\tau}{1- \frac{1}{4}\eta\mu} \right)  \nonumber\\
    & \quad + \frac{6}{\mu}\left(4L^2\kappa^2\tau^2\eta^2 G_{w}^2 \left(4\kappa^2 +1 \right) \right) +  3\eta L^2\left( 10\eta^2\tau^2  \left(\sigma_w^2+\frac{\sigma_w^2}{m} + \Gamma \right)\right)T \nonumber,
\end{align}
where we use the result of Lemmas~\ref{lemma: deviation} and \ref{lemma: sum of series}. Plugging in $\eta = \frac{ 4\log T}{\mu T}$, and $m \geq T$, we have:
\begin{align}
     \Phi(\bm{w} ^{(t)}) - \Phi( \bm{w} ^*) &\leq O\left(\frac{\Phi(\bm{w} ^{(0)}) - \Phi( \bm{w} ^*)}{T}\right)+ 
    \Tilde{O}\left(\frac{\sigma_w^2+G_w^2}{\mu T}\right)  +    \Tilde{O}\left(\frac{\kappa^2 L \tau D_\Lambda^2}{T}\right)   \nonumber\\
    & \quad +  \Tilde{O}\left(\frac{\kappa^6\tau^2  G_w^2}{\mu T^2}\right)+  \Tilde{O}\left(\frac{\kappa^2\tau^2 (\sigma_w^2+\Gamma)}{\mu T^2}\right)\nonumber,
\end{align}
thus concluding the proof.
\qed

\end{document}